\documentclass[poms,nonblindrev]{poms1} 

\usepackage[lmargin=0.82in, rmargin=0.82in, tmargin=0.9in, bmargin=0.95in]{geometry}
\renewcommand{\normalsize}{\fontsize{11.5pt}{19pt}\selectfont}

\usepackage{graphicx}
\usepackage{subfigure, epsfig}
\usepackage{natbib}
\usepackage{algorithm,algorithmic}
\usepackage{booktabs}
\usepackage{multirow}
\usepackage{enumerate}
\usepackage{soul}
\usepackage{enumitem}
\usepackage{float}
\usepackage{wrapfig}

\newcommand{\PP}{\mathbb{P}}
\newcommand{\EE}{\mathbb{E}}

 \bibpunct[, ]{(}{)}{,}{a}{}{,}%
 \def\bibfont{\small}%
\TheoremsNumberedThrough     
\ECRepeatTheorems

\EquationsNumberedThrough    

\begin{document}

\RUNAUTHOR{Chen, Shi and Pu}

\RUNTITLE{Data Pooling for Health Intervention}

\TITLE{Data-pooling Reinforcement Learning for Personalized Healthcare Intervention}

\ARTICLEAUTHORS{%
\AUTHOR{Xinyun Chen}
\AFF{The Chinese University of Hong Kong, Shenzhen, China \EMAIL{chenxinyun@cuhk.edu.cn}} 
\AUTHOR{Pengyi Shi}
\AFF{Purdue University, West Lafayette, IN, USA \EMAIL{shi178@purdeu.edu}}
\AUTHOR{Shanwen Pu}
\AFF{Shanghai University of Finance and Economics, Shanghai, China \EMAIL{2019212802@live.sufe.edu.cn}}
} 

\ABSTRACT{Motivated by the emerging needs of personalized preventative intervention in many healthcare applications, we consider a multi-stage, dynamic decision-making problem in the online setting with unknown model parameters. To deal with the pervasive issue of small sample size in personalized planning, we develop a novel data-pooling reinforcement learning (RL) algorithm based on a general perturbed value iteration framework. Our algorithm adaptively pools historical data, with three main innovations: (i) the weight of pooling ties directly to the performance of decision (measured by regret) as opposed to estimation accuracy in conventional methods; (ii) no parametric assumptions are needed between historical and current data; and (iii) requiring data-sharing only via aggregate statistics, as opposed to patient-level data. Our data-pooling algorithm framework applies to a variety of popular RL algorithms, and we establish a theoretical performance guarantee showing that our pooling version achieves a regret bound strictly smaller than that of the no-pooling counterpart. We substantiate the theoretical development with empirically better performance of our algorithm via a case study in the context of post-discharge intervention to prevent unplanned readmissions, generating practical insights for healthcare management. In particular, our algorithm alleviates privacy concerns about sharing health data, which (i) opens the door for individual organizations to levering public datasets or published studies to better manage their own patients; and (ii) provides the basis for public policy makers to encourage organizations to share aggregate data to improve population health outcomes for the broader community.
}


\maketitle

\vspace{-0.3in}
\section{Introduction}
\label{sec: introduction}

Preventative interventions, such as disease screening, monitoring, and follow-up phone calls or visits to enhance medical adherence, are important strategies to improve delivery of health care and population health. 
Progress has been made in managing major public health priorities, such as diabetes~\citep{newman2018community}, depression~\citep{lowe2004monitoring}, hypertension~\citep{liao2020personalized}, and maternal health~\citep{biswas21}. On the one hand, efficient intervention can improve individual health outcomes and reduce expenditures and burdens on the healthcare systems. On the other hand, interventions are costly, e.g., they require additional resources to staff healthcare workers to provide preventative interventions~\citep{biswas21}. Therefore, it is imperative for decision makers in the healthcare and public health sectors to manage these tradeoffs when allocating intervention resources to individuals in need. Dynamic intervention management has been a long-standing topic in the operations research fields~\citep{tunc2014opportunities}. In recent years, with the advancement of data analytics and integration of health information systems, there is a burgeoning interest in tailoring intervention strategies to different patients to best suit their individual needs~\citep{ayer2018}. 
Personalization requires prescribing tailored policies based on individuals' heterogeneous trajectory of health condition changes, presenting unique challenges. This research aims to address one critical challenge in data-driven personalized intervention: the \emph{small-sample issue}. Next, we elaborate on the small-sample issue and its associated challenges, 
which are pervasive not only in many healthcare problems, but also in domains beyond healthcare -- hence, of broad significance.

\subsection{Challenges}
\label{sec:challenge}

Personalized medicine has received increasing attention in recent years. Examples of preventative interventions, among others, include preventative screening for cancers~\citep{ayer2016heterogeneity, bertsimas2018optimal}, wearable devices for chronic disease management~\citep{gautier2021impact,liao2020personalized}, and post-discharge follow-ups with patients via phone calls or home visits to prevent hospital readmissions (the case study in this research). 
Markov Decision Process (MDP) is a common modeling tool to dynamically optimize the personalized intervention planning. 
The conventional MDP-based personalized medicine research focuses on tailored planning via updated action or belief based on individual-specific observations~\citep{piri2022individualized}. These studies often assume a common model of the condition progression for the population, where MDP parameters such as the transition matrix can be reliably estimated beforehand from existing clinical data or medical studies. Recent research has begun to do risk stratification based on regression or advanced machine learning models to capture the heterogeneity among different individuals in terms of the intervention effect and risk/condition evolution; see~\cite{ayer2018} for a comprehensive survey. 
However, these works still follow the estimate-then-optimize (ETO) framework; that is, they first estimate the MDP model parameters and then solve the optimal policies assuming that the MDP parameters are known exactly. ETO ignores the estimation errors when making decisions, which can be particularly problematic in personalized decisions due to far too few samples for each individual patient class. For example, in the hospital data used in our case study, the total patient sample size is around 20,000, but for each individual class, the size is only on the order of dozens. Using inaccurate parameters may lead ETO to prescribe highly suboptimal policies.

This research addresses challenges in dynamic, personalized intervention from a perspective that is \textbf{fundamentally different} from that of the mainstream literature in personalized medicine: the need to capture heterogeneity in individuals' different responses to intervention and risk-evolving trajectory, but the corresponding model parameters are not known a priori; only a small number of data observations are collected over the time. In other words, we consider the online learning setting in which one needs to optimize the policies while learning the model parameters. Compared to ETO methods, online learning methods explicitly accounts for parameter estimation errors when prescribing policies, and can explore potentially good policy that has not been used in the past, i.e., not restricted by the collected data. Specifically, we consider the \emph{reinforcement learning} (RL) framework for multi-stage, online decision-making. This is more suitable than static frameworks (e.g., bandit learning) for preventative intervention problems because current decisions can generate long-term impacts on patient outcomes, although the dynamic nature makes the problem a lot more complex without simple heuristics (see Section~\ref{subs:MDP}). Importantly, our focused regime is to learn and optimize the decision for each individual patient class -- the \emph{target patients} -- in a data-limited environment: each class has a small number of data observations, particularly at the beginning of the learning phase (e.g., in a new hospital). This ``small-data'' regime is common in personalized decisions in various applications~\citep{gupta2022optimization}; however, major challenges arise when applying existing RL algorithms in this regime.  


The obvious first challenge is the difficulty of learning caused by the limited data from target patients. 
To account for estimation errors, standard RL algorithms require exploration over possibly sub-optimal policies during learning. The small sample size leads to large variance, results in less ``confidence'' in the estimated parameters, and requires more exploration than usual. 
Our case study shows that using only the target patients' data, the decision quality measured via regret is significantly higher ($50\%$) than the best-performing policy when there are only dozens of observations. We note that, although typical preventative interventions are not medically based and explorations are relatively non-harmful, the additional explorations are still undesirable. For example, too-frequent follow-ups can cause patient fatigue and a waste of expensive intervention resources.

Regression-based methods such as contextual learning are sought to address the first challenge when data from other patient populations are available. However, these methods suffer major limitations in healthcare settings. We review these methods in Section~\ref{sec:lit} and illustrate the second challenge associated with them via the following example: A large health system has just opened a new hospital. 
A straightforward way to manage patient interventions in this new hospital is to run a regression combining target data and historical data from other locations in the system to facilitate learning. Indeed, regression-based methods are widely used in leveraging other datasets or previously trained models to improve personalized decisions, e.g., transfer learning in recommendation systems or dynamic pricing~\citep{bastani2021predicting}. However, these approaches often rely on a critical assumption: that the dependence of the outcome, with or without intervention, on the individual features (context) is similar across different populations.
Unfortunately, this is often not the case in healthcare, as different locations can have very different social determinants or other factors, which may not be captured by the contextual information collected by hospitals but can largely affect the individual risk evolution and treatment effects.\footnote{\cite{tomkins2021intelligentpooling} point out: ``While this approach (contextual learning) can have advantages compared to ignoring the user's context, it fails to address that users can respond differently to treatment even when they appear to be in the same context.''} Even within the same hospital, the risk and treatment effect can constantly change while the features remain similar. 
As we will demonstrate both analytically and numerically, contextual learning is prone to misspecification of the underlying models (how outcomes depend on features) across populations and can suffer from significant performance deterioration. 

The third challenge for applying existing RL methods that use multiple datasets is the typical need to share individual-level information when merging the datasets for learning. This need is extremely challenging to achieve across different health organizations due to privacy concerns. It also restricts researchers from using public datasets or published studies, which typically provide only aggregate-level data, to augment their own datasets for better intervention management. 

\vspace{-0.15in}
\subsection{Overview and Contributions}

In this paper, we develop a novel \emph{data-pooling} RL method 
to tackle the small-sample issue, while addressing the three challenges discussed above in a holistic way. Our method effectively leverages historical data (private or public data), via aggregate statistics only, to facilitate personalized learning and decision-making in an environment with limited data, with provable performance guarantee. We provide an overview of our development and highlight the contributions.

\textbf{I. Personalized Intervention RL with Model-free Design. } We formalize the MDP model for intervention planning and our learning framework in Section~\ref{sec:model-formulation}, along with defining the primary performance metric -- regret. 
Compared to standard no-pooling RL methods, our RL framework sets the structure of augmenting target data with historical data. This lays out the basis for our data-pooling algorithm design to address the first challenge associated with limited target data. 

A salient feature that differentiates our RL framework from other frameworks that include data from multiple populations is the \emph{model-free} design. We do not make specific parametric assumptions regarding how the historical and target populations are connected; for example, there are no assumed functional forms for modeling the dependence of outcomes on individual features, as in regression-based methods. Instead, we require only that the responses to interventions across these populations are sufficiently ``close'', measured in a non-parametric way (to be made precise in Section~\ref{sec:model-formulation}). Consequently, our data-pooling algorithm developed within this framework is model-free and (i) naturally deals with settings in which contextual information is non-stationary; and (ii) requires data sharing only via aggregate statistics, instead of patient-level data. 
Feature (i) addresses the second challenge regarding model misspecifications, while feature (ii) addresses the third challenge involving data privacy. 

\textbf{II. Data-pooling Algorithm and Regret Analysis}. 
The idea of using data from others to augment target data and facilitate learning is intuitive and can be traced back to the empirical Bayesian literature~\citep{efron_hastie_2016}. The innovation of our data-pooling algorithm is that its derivation ties directly to improving the decision quality, which contrasts with conventional methods that focus on improving prediction accuracy with more data. The derivation starts by formulating a ``meta'' algorithm and a unified regret analysis framework, covering major RL algorithms (Section~\ref{sec:overview-alg-regret}). Through this framework, we characterize the two key components in improving the performance (regret) of the algorithm -- the value function estimation error and exploration effort -- and we quantify their impact on the regret, which establishes connections between the regret and the design of a better estimator for value functions through a set of confidence radii. We then pin down the key mechanism in designing the data-pooling estimator -- how to adaptively weigh in the historical data -- and derive the optimal weights by minimizing the confidence radii, specified in Section~\ref{sec:data-pool-for-value}. Essentially, our data-pooling design improves the accuracy of the estimator, but this improvement is not solely for prediction accuracy. Rather, it has the more-sophisticated aim of improving value function estimation and reducing the unnecessary exploration of learning with only target data, hence improving the decision quality. This analytical framework of connecting data-pooling design with regret could generate independent theoretical interest for a broad class of RL problems.

\textbf{III. Practical Applicability and Algorithm Explainability. } We demonstrate the practical applicability of our algorithm for healthcare management in a specific context: the personalized post-discharge intervention. Leveraging a large patient-level dataset from a partner hospital, we construct a simulation platform and perform a case study in Section~\ref{sec:numerical}. 
We show the superior empirical performance of our data-pooling algorithm over benchmarks in three categories: ETO methods; learning with own (target) data; and state-of-the-art learning algorithms using both target and historical data (brief review of the methods in Section~\ref{sec:lit}).
For operations and implementation considerations, we also demonstrate how to use this algorithmic tool under various operational constraints in practice, particularly accounting for model misspecification and privacy in data sharing. Moreover, we investigate the drivers behind the better performance of our algorithm over other benchmarks, providing explainability that is critical for adoption by practitioners. 

We conclude this section by discussing the broader impact of our research and summarizing managerial implications of our algorithm for healthcare and public-sector operations.

\noindent\textbf{Broader Impact. }
While our primary context is health care, the growing needs of personalized intervention are pervasive in many applications, as are the challenges associated with the small-sample issue. 
Our analysis framework and the data-pooling estimator are developed for a rather general class of multi-stage decision problems and can be used to support personalized decision making in operational problems that share a similar modeling framework, such as incarceration diversion decisions for reducing recidivism in criminal justice systems~\citep{zhang2022routing}; revenue management (deciding which customers to follow up with to increase their chance of subscribing to paid services); and pricing~\citep{bastani2022meta}. 
The empirical success of our algorithm, as well as our proved performance guarantee, provide a stepping stone to using historical data to facilitate the learning of policies in many other data-limited settings.

\noindent\textbf{Implications. } 
Improving decision-making quality when facing the small-data regime is of great practical relevance and importance for healthcare management. Caution should be taken when using the conventional wisdom of pooling data by similar features as in the regression-based methods. The performance of these methods can deteriorate significantly when the imposed parametric model deviates from the ground truth (more than 40 folds in the case study). Practitioners should consider model-free methods like ours when the underlying individuals are highly heterogeneous. Moreover, our paper shows the success of data-sharing only via aggregate statistics. This provides a viable option for individual organizations to leverage public datasets or published studies to augment their own dataset and improve the outcomes of their own patients. It also provides the basis for public policy-makers to encourage organizations to share their aggregate data to facilitate population health outcomes for the broader community.  See more discussion in Section~\ref{subsec:robust-perform-extension}. 

\vspace{-0.1in}
\section{Literature Review}
\label{sec:lit}


\noindent\textbf{Personalized Medicine.}
\cite{ayer2018} provide a comprehensive survey on personalized interventions to contrast with traditional population-level intervention that typically uses a one-size-fits-all strategy; see references therein. Our focus is fundamentally different from that of the MDP-based personalized medicine research: we focus on joint learning (of individual class's risk trajectory) and personalized decision making. Our paper is also different from (i) those focusing on prediction, e.g., predictive models for estimating the readmission risk \citep{min2019predictive, choudhury2018evaluating}; and (ii) papers using the ETO framework. We prescribe optimal decision support with the estimation errors in the small-data regime being explicitly accounted for. 
A related field is personalization in clinical trial design~\citep{chick2022bayesian}. \cite{alban2022learning} use contextual multi-armed bandit to design a sequential clinical trial with patient-specific covariates. We note the different focus of clinical trial design as its objective is to identify the best treatment strategy (arm) over a population. The strategy remains fixed once being assigned to a patient, hence, individual patient's state transitions are usually not accounted for. One exception is \cite{bonifonte2022analytics}, who propose offline, population-based blood pressure management strategies to guide trial design, where the state (blood pressure) change is modeled via a Brownian motion. 


\noindent\textbf{RL in Healthcare Interventions.} 
Bandit learning has been used in personalized healthcare interventions; see, e.g., \cite{Feiyun2018ActorCritcMhealth} and \cite{keyvanshokooh2019contextual}. We focus on multi-stage sequential decision making in the online learning setting. 
The dynamic nature greatly complicates our problem. There is no simple heuristic to bypass the needs to estimate the risk evolution; see Section~\ref{subs:MDP}. In multi-stage learning, two types of regression-based methods are commonly used. The first type is clustering (classification), whereby patient clusters are first formed based on features and the response to treatment. Then, new patients are grouped into the corresponding cluster when learning and prescribing policies~\citep{shi2021timing,utomo2019personalised}.
The second type is contextual learning, whereby a regression model is used to estimate the value functions~\citep{liao2020personalized,zhou2018personalizing}. 
We will show the insufficiency of these regression-based methods analytically (Section~\ref{sec:model-misspec}) and numerically (Section~\ref{sec:robust-under-misspec}) due to model misspecification. In contrast, our model-free algorithm enjoys built-in robustness to misspecification.  

\noindent\textbf{Data-pooling for Small-data Regime. } 
The idea of ``learning from others" originated from empirical Bayesian statistics~\citep{efron_hastie_2016} for estimation/prediction problems. Data pooling in decision making is fundamentally different, since the design of the data-pooling estimators should tie to the performance of decisions, not just the prediction accuracy. In \cite{gupta2021data}, the authors compare between data pooling for estimation and pooling for optimization. They consider stochastic optimization in an offline setting, i.e., all data are collected in advance and fixed throughout the optimization procedure, which is different from the online setting we consider. For online decision-making, data pooling has been used under the contextual bandit framework.  \cite{tomkins2021intelligentpooling} consider personalized treatment in mobile health applications and develop an adaptive pooling method combining the empirical Bayesian idea with Thompson Sampling (TS) on Gaussian regression models. \cite{bastani2022meta} propose a meta-learning algorithm for dynamic pricing, which learns across sequential TS experiments. 
We stress that the RL setting has unique challenges so methods for the bandit setting do not directly extend. 
Among others, \cite{miao2019context} is the closest to our paper, in the sense that they also treat data from different populations in a model-free manner. Based on a semi-myopic pricing policy, \cite{miao2019context} propose an adaptive clustering method to pool data across different customer classes. We make a detailed comparison with their method in Sections~\ref{subs:EmpResult} and~\ref{subsec:robust-perform-extension}, showing the advantages of our algorithm in overcoming the three challenges discussed in Section~\ref{sec:challenge}.


\noindent\textbf{Regret Analysis for RL.} Our theoretic analysis roots in regret analysis for finite-horizon RL problems. There is a vast literature in this area; see, e.g., \cite{ortner2007logarithmic, jaksch2010near, wen2014efficient, osband2016generalization,azar2017minimax}, and \cite{russo2019worst}. Our regret bound analysis framework is developed based on \cite{russo2019worst} and connects between the data-pooling estimator design and the two key components impacting the regret. We stress that most research in this area focuses on the asymptotic dependence of regret bound on the problem scales, e.g., state-space size, while our work focuses on reducing regret bound in finite iterations via data pooling.
    


\vspace{-0.1in}
\section{Model Formulation}
\label{sec:model-formulation}

In this section, we specify our modeling framework. We introduce the general MDP model in Section~\ref{subsec:prob-setting}, and the online learning framework in Section~\ref{sec:RL-learning-framework}. 
We provide context for the framework via a specific example (post-discharge intervention) in Section~\ref{sec:motivate-example}, where we analyze its structural properties to showcase the challenges associated with data-driven personalized intervention.

\medskip

\subsection{MDP Model}
\label{subsec:prob-setting}

We consider a finite-horizon MDP $\mathcal{M} = (\mathcal{H}, \mathcal{S}, \mathcal{A}, R,P)$ for each \emph{target} individual (patient) class, with the goal of optimizing the policy for them. Here, $\mathcal{H}=\{1,2,...,H\}$ is the decision horizon, and $\mathcal{S}=\{0, 1, \dots, S\}$ and $\mathcal{A} = \{0, 1, \dots, A\}$ are the state and action spaces, respectively. The action is taken at each decision epoch $h\in\mathcal{H}$. To align with the learning literature, we consider maximizing reward as the objective. For simplicity, we further assume that the per-epoch reward is bounded and normalized, i.e., for all $h, s, a$, the collected random reward $R(h,s,a)\in [0,1]$ with the mean reward denoted as $r(h,s,a)$. For given $h,s,a$ and $s'\in\mathcal{S}$, we use $P(s';h,s,a)$ to denote the transition probability into state $s'$ given that, at epoch $h$, the system is in state $s$ and action $a$ is taken. We denote by $P(\cdot;h,s,a)$ or simply $P(h,s,a)$ the $S$-dimension vector of transition probabilities for given $h,s,a$. We use $s_1$ to denote the initial state. We call $\mathcal{M}$ the \emph{target MDP}. Our goal is to find an optimal policy $\pi:\mathcal{S}\times \mathcal{H}\to \mathcal{A}$ for $\mathcal{M}$ with the objective as 
\begin{equation}
\max_{\pi} V(\mathcal{M}, \pi) := \max_{\pi} \mathbb{E}\left[ \sum_{h=1}^H R\big( h, s_h, a_h=\pi(s_h,h) \big) ~{\Big|} s_1 \right].
\label{eq:target-obj}
\end{equation}

This MDP model is fairly general and applies to many applications. We take preventative intervention as the example to illustrate these model parameters. The state typically corresponds to the health status of the patient. Following the literature, the states $s \in \mathcal{S} = \{0, 1, \dots, S\}$ are ranked monotonically from healthy to less healthy, where $0$ typically stands for healthy, and $S$ stands for some absorbing state for adverse events (e.g., mortality, readmission, disengagement from mobile health programs, or ending of the preventative phase). The state in epoch $h$ is observed first, and if not entering the absorbing state, we then decide on the intervention $a \in \mathcal{A} = \{0, 1, \dots, A\}$, with $0$ being no intervention, $A$ being the strongest level of intervention, and other actions in-between being middle levels. $P(s';h,s,a)$ corresponds to the transition probability among different health states under choice $a$, where a useful intervention increases the chance of transitioning to healthier status. To capture the tradeoffs in intervention, a common form of the cost (negative reward) is, for $s\neq S$, $-r(h,s,a) = c(h, s, a) = c_a + c_R\cdot P(s'=S;h,s,a)$, where $c_a$ is the intervention cost for action $a$ ($c_0=0$ for no intervention), and $c_R$ is the penalty for entering $S$ (adverse event occurring).

Note that the Markovian intervention can be easily extended to be history-dependent by enlarging the state (e.g., add a state dimension to track the number of previous interventions). We focus on the Markovian policy setting in the rest of the paper for ease of exposition. 
Similarly, a partially observable setting could be accounted for by augmenting the state with an information state~\citep{piri2022individualized}; we focus on the observable state setting in this paper.

The optimization problem~\eqref{eq:target-obj} can be solved via the Bellman equation and value (or policy) iteration methods; see a brief review in Appendix~\ref{section:VIPI}. Our MDP model formulation covers a wide range of multi-stage decision making in operational problems beyond the preventative intervention planning. For example, in the criminal justice setting, community-based services for post-incarcerated individuals are shown to help their reintegration into society and reduce the recidivism rate~\citep{zhang2022routing}. This planning problem can be modeled with a similar MDP. 

\vspace{-0.15in}
\subsection{Reinforcement Learning Framework}
\label{sec:RL-learning-framework}

The MDP model introduced in Section~\ref{subsec:prob-setting} is for one target class. For personalized intervention, the reward $R$ and transition probabilities $P$ are often highly \emph{heterogeneous} across different target classes, but the true model parameters for each class are not known exactly. Estimating the reward and transition probabilities $R^{(i)}, P^{(i)}$ accurately for each class $i$ is critical when managing a panel of heterogeneous patients in the data-driven context; however, this is a difficult task given the far too few data samples for each class. The ETO framework plugs in the estimated parameters and solves the corresponding MDP model. The resulting policies can be highly suboptimal due to the inaccuracy of the estimation~\cite{gupta2022optimization}. 

To account for the parameter uncertainty, we deviate from ETO and use the online learning framework -- specifically, the reinforcement learning (RL) framework for multi-stage, dynamic decisions. 
For exposition, our main framework considers the learning task for each target class $i$ separately, without interaction among the classes. When referring to target individuals or patients, we always consider a given class and we drop the superscript $i$; see Remark~\ref{rmk:model-ext} for an extension on different classes interacting. In RL, we do not know the parameters of the target MDP $\mathcal{M}$ a priori. This is a major difference from conventional research, in which all the model parameters are assumed to be known. Instead, we gradually collect data when we choose an action according to some policy and interact with $\mathcal{M}$. The goal is to learn an optimal policy for $\mathcal{M}$ during the interaction. The data collected from the interaction with $\mathcal{M}$ is called the target data and we denote it as $\mathcal{D}$.

\noindent\textbf{Performance Analysis. } The primary performance metric for a given RL algorithm is its \emph{regret}, defined as the difference between the maximum total reward under the true (oracle) optimal policy $\pi^*$ and the expected reward collected when decisions are made by the RL algorithm under consideration. An RL algorithm is more efficient if it has a smaller regret, indicating that the decision quality is better when we do not know the true parameters exactly. Mathematically, the total regret accumulated in $T$ iterations is defined as
\begin{equation}
Regret(T) := \sum_{t=1}^T \left( V(\mathcal{M},\pi^*) - \EE[V(\mathcal{M},\pi_t)] \right), \label{eq:regret-def}
\end{equation}
where $\pi_t$ is the policy prescribed by the RL algorithm in iteration $t$.  
In Section~\ref{sec:overview-alg-regret} we introduce a meta framework covering a class of commonly used RL algorithms in the literature, for the learning task using target data $\mathcal{D}$ and its associated regret analysis.

A significant challenge with applying the existing RL algorithm to our problem is that $\mathcal{D}$ contains no data or very few data points at the beginning stage of learning -- the small-sample issue as discussed in the introduction. To recap, in the online learning setting, to account for the parameter uncertainty, exploration over possibly non-optimal policies during the learning process is required to avoid getting stuck in a sub-optimal policy (a main differentiator from ETO). The small sample size, however, leads to large variance and less ``confidence'' in the estimated parameters and, thus, requires more exploration of non-optimal policies and increases the regret. We formalize this intuition via the regret analysis in Section~\ref{sec:overview-alg-regret} and numerically demonstrate in Section \ref{sec:numerical} the high regret when there are just dozens of target patients at the beginning.

\noindent\textbf{Improving Performance with Historical Data. }
To tackle the small-sample issue and improve the performance of existing RL algorithms, we consider developing a new class of algorithms with offline historical data from other users/systems that are similar to $\mathcal{M}$. Reducing regret over the existing algorithms is not only beneficial from the technical perspective, but also has important practical implications. In the preventative intervention setting, targeting the expensive intervention resources to patients who need them most is critical to generating system-wide improvements. Moreover, even though the typical preventative interventions are non-medically based, the additional exploration efforts associated with using only target data may still cause undesirable outcomes for patients (such as sending too many follow-up text messages) and lead to the waste of the intervention resources. Thus, the performance improvement is essential for healthcare management. 

We denote the MDP for the historical patient populations as $\mathcal{M}_0=(\mathcal{H},\mathcal{S},\mathcal{A}, R_0,P_0)$. The historical patients share the same state and action spaces with the target patients, as well as the same horizon length $H$. However, they differ in the reward and transition probabilities ($R_0,P_0$), which are also not known. We do not interact with the historical patients during the learning but have access to data generated from $\mathcal{M}_0$ offline, denoted as $\mathcal{D}_0$ -- namely, the \emph{historical data}. Next, we formalize the ``closeness'' between the target and historical MDPs in Assumption~\ref{assmpt: difference bound}, followed by two remarks.
	
\begin{assumption}\label{assmpt: difference bound}
		There exists a known constant $\Delta>0$, such that $\forall (s,a)\in\mathcal{S}\times\mathcal{A}$, $1\leq h\leq H$,
		\begin{align*}
		&|r(h,s,a)-r_0(h,s,a)|\leq \Delta, \quad \|P(\cdot; h,s,a) -P_0(\cdot;h, s,a)\|_1\leq \Delta.
		\end{align*}
\end{assumption} 
It is worth emphasizing that we do not need to make any parametric assumptions on how the target patients and historical patients are connected in terms of the model parameters. All we need is this non-parametric Assumption~\ref{assmpt: difference bound} to measure their closeness. This is a significant difference from regression-based methods such as contextual learning or clustering: their key assumption is that target and historical patients who share similar features have similar outcomes (reward) and transition probabilities, so a parametric regression model is often in place to capture this similarity. In other words, our learning framework is ``model-free'' and is more general than the existing regression-based methods. Importantly, regression-based methods are prone to model misspecification while ours is not; we show this analytically in Section~\ref{sec:model-misspec}. Our learning algorithm developed under this model-free framework is specified in Section~\ref{sec:data-pool-for-value}. We will further discuss how to deal with $\Delta$ in implementation after specifying the algorithm there.


\vspace{-0.05in}
\begin{remark}[Aggregate Statistics]
\label{rmk:mix-MDP}
Another primary benefit of our model-free formulation is that only aggregate statistics in terms of $(R_0,P_0)$ are required from the historical data $\mathcal{D}_0$, instead of individual-level data. Our algorithm developed under this framework naturally is model-free, a very appealing feature for settings in which data privacy is a main concern. Because it allows the use of public datasets or published studies (which usually just provide only aggregate statistics) to facilitate the learning of individuals; see further discussions in Section~\ref{subsec:robust-perform-extension}. This representation via $(R_0,P_0)$ also covers cases in which the historical data observations are from a mixture of patients from different classes/populations; that is, $\mathcal{D}_0$ is generated by a mixture of multiple (unknown) MDPs~\citep{iclr2020}. In these cases, $(R_0,P_0)$ corresponds to the mixture of different MDPs (according to their sample proportions) and can be estimated via the empirical averages. 
\end{remark}

\vspace{-0.08in}
\begin{remark}[Model Extensions]
\label{rmk:model-ext}
First, in the baseline setting, the historical data $\mathcal{D}_0$ are treated as though they are from a single group. When more information is available, such as which subsets of observations in $\mathcal{D}_0$ belongs to the same subgroup (e.g., from the same hospital or region), we can further refine our analysis and algorithm with this information; see Appendix \ref{subsec: extension-to-multi-group}.
Second, we treat the learning task for each target patient class independently, assuming the same access to historical data during the learning. We can incorporate the interactions among target classes in a batched fashion: we learn and collect target data separately in each batch and add them into the historical data before starting the next batch. Lastly, the baseline model focuses on the setting without resource constraints. To deal with resource constraints, one can adapt the well-known Whittle index and incorporate it into the learning, e.g., see~\cite{grand2020robust}. 
\end{remark}

\vspace{-0.15in}		
\subsection{Structural Analysis: Post-discharge Intervention Management}
\label{sec:motivate-example}


In this section, we provide a detailed structural analysis of one specific preventative intervention example: the post-discharge intervention planning to reduce hospital readmissions. This is also the main application used in our case study in Section~\ref{sec:numerical}. The analysis in this example serves three purposes: (i) it provides an elaborated operational context for the general MDP and RL framework (Section~\ref{sec:ops-context-readm}); (ii) the structural properties for this example problem show the criticality of tailored intervention and the need to accurately estimate the MDP parameters for each target class (Section~\ref{subs:MDP}); and (iii) showcasing the insufficiency of existing learning methods (Section~\ref{sec:model-misspec}). 

\vspace{-0.1in}
\subsubsection{Operational Context}
\label{sec:ops-context-readm}

Patients who fail to comply with the medical advice given upon discharge (e.g., they forget to take their medications as prescribed) are not only contributing to millions of dollars in avoidable healthcare costs and increased hospital congestion, but are also putting themselves at higher risks for adverse events and relapses that land them back in the inpatient ward~\citep{bresnick_2016}. Follow-ups via text messages, calls or home visits to increase patients' adherence to medical advice have proven  effective in reducing hospital readmissions~\citep{yiadom2020impact,takchi2020extending}. Specifically, after a patient is discharged, the hospital makes a plan for whether to follow up with the patient at each ``check point'', typically each week~\citep{liu2018missed}. This plan is carried out by a case manager, who may dynamically adjust the plan given the state of the patient (e.g., healthy or at-risk) each week. Despite the benefits, due to the additional staffing requirements for the followups, such interventions could result in an overall net loss if not properly designed~\citep{chan2022dynamic}.

This post-discharge planning problem can be modeled via the MDP introduced in Section~\ref{subsec:prob-setting}. Each checkpoint is a decision epoch $h$, and without loss of generality, each week is an epoch. We consider a 30-day or 90-day readmission window, hence 4 or 12 decision epochs in total. For each decision epoch, the home visit is the strongest level of intervention, followed by phone calls and text messages. As mentioned, we focus on the Markovian intervention policy, but it is easy to incorporate history-dependent policies by enlarging the state space. In the simplest case with  $|\mathcal{S}|=2$, the state just corresponds to healthy (0) and readmitted (1), and the transition probability from $s=0$ to $s'=1$ corresponds to the readmission risk. The cost (negative reward) $-r(h,s,a) = c(h, s, a) = c_a\cdot 1(a\neq 0) + c_R\cdot P(s'=1;h,s,a)$ for $s\neq 1$, where $c_a$ is the intervention cost, and $c_R$ is the penalty cost for readmission. See more illustration in Appendix~\ref{app:post-charge}. 


\vspace{-0.1in}
\subsubsection{Structural Properties and Criticality of Risk Trajectory Estimation}
\label{subs:MDP}

We derive structural properties of the optimal post-discharge intervention policy, focusing on the simplest setting: $\mathcal{S}=\{0,1\}$ and $\mathcal{A}=\{0,1\}$ with $a=1$ being the intervention taken. All MDP parameters are known (no learning needed).  
We consider a given target class and use $p_{h,a}$ to denote its readmission probability in epoch $h$ under action $a$. We use $\beta \equiv c_a/c_R$ to denote the treatment cost ratio, and $\Delta_h\equiv p_{h,0}-p_{h,1}$ the intervention effect in epoch $h$. Proposition \ref{prop: temporal pattern} shows the temporal pattern of the optimal decision, i.e., when to start and stop intervention. We consider the typical risk evolution pattern in the post-discharge phase: the risk increases first and then decreases~\citep{shi2021timing}. This pattern also covers the scenario in which the risk only increases or decreases. 
\begin{proposition}[Temporal Pattern]
\label{prop: temporal pattern}
	Assume that the following two conditions hold:
		\begin{enumerate}
			\item[(a)] For fixed $a\in\{0,1\}$, the readmission probability $p_{h,a}$ is monotone non-decreasing in $h$ for $1\leq h\leq  h_c$, and $p_{h,a}$ is monotone non-increasing in $h$ for $h_c \le h \le H$.
			\item[(b)] The change in treatment effect $\Delta_h - \Delta_{h+1}$ satisfies that
			$$ \frac{\Delta_h}{p_{h,0}} - \frac{\Delta_{h+1}}{p_{h+1,0}}\begin{cases}
				< \frac{\beta}{\prod_{i=h+1}^{H}\left(1-p_{i, 1}\right)},& \text{ for }  1 \leq h \leq h_c-1,\\
				> \frac{\beta }{\prod_{i=h+1}^{H}\left(1-p_{i, 0}\right)}, &\text{ for } h_c \le h \leq H - 1.\end{cases}.$$
		\end{enumerate}
	Then, the optimal intervention policy takes the form of
		$$a^*_0=...=a^*_{h_0}=0,\quad a^*_{h_0+1}=...=a^*_{h_1}=1,\quad a^*_{h_1+1}=...=a^*_{H}=0,$$
		for some $1\leq h_0\leq h_c\leq h_1\leq H$.
\end{proposition}
This proposition says that the optimal policy is to initiate and then stop intervention at some proper epochs. Intuitively, the two conditions jointly imply that treatments in the middle are more effective, so it is optimal to allocate interventions in the middle. 
Note that Condition (b) involves the relative treatment effect $\Delta_h/p_{h,0}$ and the entire future risk trajectory consisting of all $p_{i,0}$'s and $p_{i,1}$'s after the current epoch $h$. Though the structure of the temporal pattern is intuitive, the detailed times of the start epoch $h_0$ and the ending epoch $h_1$ depend on the actual values of these parameters. Thus, when managing a panel of heterogeneous patients with potentially very different risk trajectories, tailored intervention planning is necessary. To achieve this, decision makers need to track the risk trajectory for each target patient class, a task that soon becomes unmanageable for clinicians if purely relying on human memory, even in this simplest setting. Furthermore, we show in Appendix~\ref{app:pat-pri} that there is no simple rule-of-thumb to bypass the need to estimate each individual class's risk trajectory, e.g., neither prioritizing with respect to the absolute risk level nor prioritizing to the treatment effect is optimal.

The complex dependence of the optimal policy structure on the entire risk trajectory, demonstrated via this simplest MDP setting, is due to the dynamic nature of our problem. This is in sharp contrast to static or one-stage intervention~\citep{tomkins2021intelligentpooling, gupta2020maximizing}. One cannot rely solely on human intuitions or ad-hoc decision criteria.  
Moreover, the optimal policies critically depend on the estimation of the parameters (the readmission probabilities). If the estimations are inaccurate, the resulting policy could be far from  optimal. Algorithmic decision-support based on effective learning of the MDP parameters in a data-driven fashion is critical. 

\subsubsection{Insufficiency of Regression-Based Methods}
\label{sec:model-misspec}

Regression-based methods such as contextual learning and clustering are commonly used tools for online learning, as we reviewed in Section~\ref{sec:lit}. They typically assume that the dependence of the reward or transition on the features is similar among patients. 
In this section, we construct an example in the readmission context with just one patient feature $x$. We show that the common wisdom of grouping patients by similar features may never recover the true model due to misspecification.  

We consider two classes of patients: target (class 1) and historical (class 2), and 30-day readmission window (so $H=4$); other settings remain the same as in the simplest MDP setting. We assume that the 30-day readmission risk follows a common linear model for class $i$: 
\begin{align} \label{eqn: synthetic linear model-main}
p_i(a,x) = c_{1} a + c_{2} x + c_{3}, \quad i=1,2,  
\end{align}
where the readmission risk depends on the action $a=\{0,1\}$ (with treatment effect $c_{1} < 0$), a continuous feature drawn from a Gaussian distribution $x\sim \mathcal{N}(\mu_i, \sigma_i^2)$, and the intercept $c_3$. We omit the state variable $s\in\{0,1\}$ (readmitted or not) for simplicity. The temporal pattern, conditional on being readmitted, is given by the proportions $\{\alpha_i^h, h=1,\dots,4\}$ for class $i=1,2$. 
A typical contextual-based method tries to estimate time-dependent readmission probabilities via a linear regression in the form of
$p^{h}(a,x) = \theta_1^h a + \theta_2^h x + \theta_3^h$ for each $h$ with $\theta$'s being the parameters to estimate. The following lemma shows that when the dependence on $x$ is different between the two classes of patients, an imposed parametric model, even if it is the correct form for each class separately, differs from the ground truth for the mixture of the two classes.

\begin{proposition}
\label{lemma: linear model for period} 
\begin{enumerate}[label=(\roman*)]
\item For each class $i=1,2$, the readmission risk in week $h$ follows the linear model
$p_i^h(a,x) =  c_{1,i}^h a +  c_{2,i}^h x +  c_{3,i}^h$, 
with $c^h_{j,i}=\alpha_i^h c_{j}$ where $c_{j}$'s are the linear coefficients in~\eqref{eqn: synthetic linear model-main}.
\item For a mixture of the two classes of patients with mixture probabilities $q_i$'s, if $\alpha_1^h\neq \alpha_2^h$ and the feature distributions are different, the readmission risk $p^h(a,x)$ over the mixture of patients is \emph{not} a linear function of $(a,x)$.
\end{enumerate}
\end{proposition} 

The proof of this proposition is detailed in Appendix~\ref{app:proof-structural}. Part (i) says that for each class of patients, the linear format is indeed the correct form, and a linear regression can recover the true model. However, part (ii) implies that when we mix these patients, the weekly readmission risk cannot be represented by a linear model. In other words, when we mix data from historical and target patients and try to fit a linear model, the fitted values will deviate from the ground truth. This is exactly the issue caused by \emph{model misspecificiation}: even though past experience tells us that a linear model is sufficient for historical patients, the linear model is no longer sufficient when we impose it on the new patients while mixing data from the historical patients. We will demonstrate through numerical experiments in Section~\ref{subsec:numerical on bandit} that such model misspecification can cause the performance of contextual-learning methods to deteriorate significantly. 

\vspace{-0.02in}
\section{Background: Perturbed-LSVI Reinforcement Learning and Regret Analysis}
\label{sec:overview-alg-regret}

In this section, we lay the foundation for our RL algorithm. 
In Section \ref{subsec: PLSVI}, we introduce a meta-algorithm framework, the \textit{perturbed least-squares value iteration} (perturbed LSVI), and highlight two of its key components. In Section \ref{subsec:regret-framework}, we provide an overview of the performance analysis of regret under this meta framework, with an emphasis on the connection between the regret bound and the two key components. 
The primary benefit of using this meta framework is that it unifies several commonly used algorithms in the RL literature. This provides wide applicability for our data-pooling algorithm developed based on this meta framework. That is, for these commonly used algorithms, we have a unified approach to deriving their data-pooling version and improving the regret performance (to be detailed in Section~\ref{sec:data-pool-for-value}). Note that, in this section, we use only the target data $\mathcal{D}$ when describing the meta framework, as that is the status quo in the literature.

\vspace{-0.1in}
\subsection{Perturbed LSVI Framework for RL algorithm}\label{subsec: PLSVI}

The perturbed LSVI framework is based on the fundamental value iteration method for solving MDP problems. Recall that value iteration first solves the state-action value function -- namely, the \emph{Q-function} -- from the Bellman optimality equation
$$Q^*(h,s,a):=r(h,s,a)+\sum_{s'}P(s';h,s,a)\max_{y\in\mathcal{A}}Q^*(h+1,s',y), \quad\text{with }Q^*(H+1,\cdot,\cdot)\equiv 0,$$ and then obtain the optimal policy $\pi^*(s,h)=\argmax_{a\in\mathcal{A}}Q^*(h,s,a)$. As the term perturbed LSVI suggests, it adapts the value iteration method for the online learning setting with two key components: (i) value function estimation via least squares fitting, and (ii) adding perturbation terms. We elaborate on these two components in sequence below. 

The first key component is the estimation of the Q-functions by approximating it with a linear combination of basis functions and solving the coefficients for each basis via the least squares method~\citep{bradtke1996linear}; see a brief review of the Q-function approximation in Appendix~\ref{section:LSVI}. In the case of tabular representation, the basis functions are chosen as the indicator functions for each $(s,a)$ pair. Correspondingly, the estimated Q-function can be solved with an explicit expression 
\begin{align}
 \hat{Q}(h,s,a)&
 =\bar{r}(h,s,a)+\sum_{s'\in\mathcal{S}}\bar{P}(s';h,s,a)\max_{y\in\mathcal{A}}\hat{Q}(h+1, s',y) \notag\\
 &\triangleq\bar{r}(h,s,a)+\bar{P}(h,s,a)\cdot\max_y\hat{Q}(h+1,\cdot,y)  
\label{eq:convent-LSVI-tabular}
\end{align}
for $h=H, H-1, \dots, 1$, with $\hat{Q}(H+1,s,a)=0$. Here, $\cdot$ represents the inner product of two vectors, $\bar{r}(h,s,a)\in\mathbb{R}$ and $\bar{P}(h,s,a)\in\mathbb{R}^S$, which are the maximum likelihood estimators (MLE) calculated from the target data $\mathcal{D}$. We focus on this tabular representation in the rest of this paper.

The second key component is the addition of a perturbation term to address the exploration issue in the online learning setting. It is known that value iteration based solely on the estimated value function could get stuck in sub-optimal solutions without any exploration~\citep{osband2016generalization}. The perturbed LSVI framework introduces a set of perturbation terms $\{w_t(h,s,a)\}$ and injects them into the estimated value functions, i.e.,  
\begin{align}
	\tilde{Q}(h,s,a)=\bar{r}(h,s,a)+w_t(h,s,a)+\bar{P}(h,s,a)\cdot\max_y\tilde{Q}(h+1,\cdot,y) . 
	\label{eq:perturbed-LSVI-tabular}
\end{align} 
Then, the value iteration is applied using the perturbed version $\tilde{Q}(h,s,a)$ instead of  $\hat{Q}(h,s,a)$.  

\begin{algorithm}[ht]
	\caption{Perturbed LSVI}
	\begin{algorithmic}
	\small{
		\STATE Initialize $\tilde{r}_{1}=0$ and $\bar{P}_1=0$, data set $\mathcal{D}=\{\}$,  $\tilde{Q}_t(H+1,s,a)=0,~\forall t, s,a$. 
		\FOR{$t=1$ {\bfseries to} $T$}
		\STATE Compute $\bar{r}_{t},\bar{P}_{t}$ from $\mathcal{D}$ 
		\FOR{$h=H,H-1,\ldots,1$}
		\FORALL{$(s,a)$ pairs}
		\STATE Generate $\xi_t(h,s,a)$ and compute $w_t(h,s,a) =  \varepsilon_{V}(n_t(h,s,a))\cdot \xi_t(h,s,a)$
		\STATE Compute $\tilde{Q}_t(h,s,a)=\bar{r}_t(h,s,a)+w_t(h,s,a)+\bar{P}_t(h,s,a)
		    \cdot\max_y\tilde{Q}_{t}(h+1,\cdot,y)$ 
		\ENDFOR
		\ENDFOR
		\FOR{$h=1,\ldots,H$}
		\STATE Collect state $S^t_h$
		\STATE Select action $A^t_h=\arg\max_a\tilde{Q}_t(h,S^t_h,a)$
		\STATE Collect reward $R^t_h$
		\ENDFOR
		\STATE Update  $\mathcal{D}=\mathcal{D}\cup\{S^t_1,A^t_1,R^t_1,\ldots,S^t_H,A^t_H,R^t_H\}$
		\ENDFOR	
	}
	\end{algorithmic}
	\label{Alg:P-LSVI}
\end{algorithm}

The perturbed LSVI covers two popular classes of RL methods depending on the design of the perturbation terms $w_t(h,s,a)$: \texttt{UCB-based} RL algorithm and \texttt{TS-based} RL algorithm, where UCB stands for Upper-Confidence-Bound, and TS stands for Thompson-Sampling. We elaborate $w_t(h,s,a)$ in Section~\ref{subsec:regret-framework}. The TS-based version is also known as the Randomized Least-Squares Value Iteration (RLSVI), initially developed in \cite{wen2014efficient} and extended in \cite{osband2016generalization,russo2019worst}. A complete description of the perturbed LSVI algorithm is given in Algorithm \ref{Alg:P-LSVI}.

\vspace{-0.1in}
\subsection{Regret Analysis Framework}
\label{subsec:regret-framework}

Recall the definition of regret given in~\eqref{eq:regret-def}. 
For the perturbed LSVI algorithm, the regret can be separated into two parts: estimation errors and exploration efforts, which are contributed by the two key components, respectively. 
Below, we provide a high-level idea on how the regret connects to each of the two key components. Streamlining these connections provides the important basis to derive our data-pooling algorithm in Section~\ref{sec:data-pool-for-value}, which directly connects its design with the goal of improving the regret performance.   

The first part of the regret relates to the accuracy of the value function estimation, which further depends on the estimates of the MDP parameters: reward and transition probabilities $(R,P)$. 
This connection is intuitive since the perturbed LSVI algorithm belongs to the class of value-iteration algorithms; it can prescribe better decisions if the value function and MDP parameters are estimated more accurately. Through standard regret analysis, one can show that the accuracy of these estimates can be measured by their corresponding \emph{confidence radius}, which is derived from Hoeffding's inequality and probability models for tabular RL. Formally, we use $\varepsilon^H_R(n)$ and $\varepsilon^H_P(n)$ to denote the (Hoeffding) confidence radius in estimating the mean reward and transition probabilities, respectively, from $n$ samples in the target data $\mathcal{D}$. We can then bound the part of the regret caused by estimation error by some increasing function in terms of these confidence radii, $f_1(\varepsilon^H_R, \varepsilon^H_P)$.

The second part of the regret comes from the exploration cost, contributed by the perturbation terms $\{w_t(h,s,a)\}$. Intuitively, $w_t(h,s,a)$ must compensate the estimation error in the value function in order to ensure sufficient exploration. A common practice in the design of the perturbed LSVI algorithm is to set $w_t(h,s,a)$ ``in proportion" to the confidence radii of the estimated value functions, denoted as $\varepsilon^H_{V}(\cdot)$, with
\begin{equation}\label{eq: perturbation terms}
	w_t(h,s,a) = \varepsilon^H_{V}(n_t(h,s,a))\cdot \xi_t(h,s,a),
\end{equation}
where $n_t(h,s,a)$ is the number of target samples right before iteration $t$ at the $(h,s,a)$ tuple and $\xi_t(h,s,a)$ are i.i.d. copies of some random variable $\xi$. For example, in \texttt{UCB-based} RL, $\xi$ is set to be a fixed constant, while in \texttt{TS-based} RL, $\xi$ is a normal random variable.
; see more details in Appendix~\ref{app:additional-PLSVI}
Under~\eqref{eq: perturbation terms}, the corresponding exploration cost can be bounded by another function $f_2(\varepsilon^H_V, \varepsilon^H_P,\bar{W})$, increasing in each term, where $\bar{W}>0$ depends on the distribution of $\xi$. The appearance of $\varepsilon^H_P$ here is due to the multi-stage nature of the RL setting, i.e., the exploration cost across different stage $h$ accumulates through state transitions. The exploration cost connects to the estimation accuracy of the value function via $\varepsilon^H_V$.

Based on the two parts of the regret, one can derive an upper bound for the total regret of the perturbed LSVI algorithm in the following form
\begin{equation}\label{eq: regert informal}
Regret(T) \leq \sum_{n=1}^{\lceil T/SA\rceil} f_1\big( \varepsilon_R^H(n), \varepsilon_P^H(n)\big)
+ f_2\big( \varepsilon^H_V(n), \varepsilon^H_P(n),\bar{W} \big).
\end{equation}
The rigorous definition of these confidence radii, the detailed expression of $f_1(\cdot)$ and $f_2(\cdot)$, and the regret bound and its proof are given in Appendix \ref{subsec: connectiong regret to CI}. The main idea of the proof is based largely on~\cite{russo2019worst}, 
and our primary adaptation is to extend the proof for the unified meta-framework (the perturbed LSVI) instead of the focal case (the \texttt{TS-based} RL, RLSVI) in~\cite{russo2019worst}. The bound in \eqref{eq: regert informal} indicates that the key to designing an improved version of a perturbed LSVI algorithm is to reduce these confidence radii to achieve a better regret bound. 


\section{Data-pooling Algorithm}
\label{sec:data-pool-for-value}
In this section, we develop a novel data-pooling RL algorithm framework. This framework provides a unified approach to convert any perturbed LSVI algorithm given in Section \ref{sec:overview-alg-regret} into its data-pooling version, which effectively leverages the historical data to address the small-sample issue and improves the performance over the original no-pooling version. We first provide the overall design ideas, and then formalize the details in Sections \ref{subsec:data-pool-estimator} and \ref{subsec:main-result}. 

{\textbf{Desgin Ideas.}}
To leverage both the target and historical data, a natural idea is to estimate the model parameters, i.e., mean rewards and transition probabilities, in the form of a weighted average of the MLE obtained from the two datasets (given by \eqref{eq: data-pooling estimates fixed n} below). This leads to a data-pooling estimator for the value function that potentially is more accurate and would improve both parts of the regret bound in~\eqref{eq: regert informal}: the first part on estimation errors and the second part on the exploration efforts. The intuition is that more accurate estimates lead to smaller estimation errors, more confidence in the estimates and less exploration efforts on suboptimal policies, which then reduce regret. Once the new estimators are designed, we also need to implement them within the perturbed LSVI framework with proper modifications to the two key components: (i) the value function estimation and (ii) the perturbation design, through which we get the final product: the data-pooling RL algorithm.

The rest of this section works on the tasks in getting the final algorithm, 
and we provide a roadmap here. First, consider a fixed $(h,s,a)$ tuple; let $\bar{r}_n$ and $\bar{P}_n$ be the MLE of the mean reward and vector of transition probabilities obtained from $n$ samples from the target data, and let $\bar{r}_0$ and $\bar{P}_0$ be the MLE obtained from $N$ samples from the historical data (satisfying the model-free RL framework in Section~\ref{sec:RL-learning-framework}).
The data-pooling estimators for the mean reward and the vector of transition probabilities follow a weighted average form as  
\begin{equation}
\label{eq: data-pooling estimates fixed n}
	\hat{r}^{DP}_n=\lambda \bar{r}_n+(1-\lambda)\bar{r}_0,
	\quad 
	\hat{P}^{DP}_n = \lambda \bar{P}_n+(1-\lambda)\bar{P}_0,
\end{equation}
where the weight $\lambda\in [0,1]$. The core for our data-pooling design is to choose a proper $\lambda$. Our major innovation is to choose this $\lambda$ based on the decision quality -- the regret performance. Recall that reducing the regret bound~\eqref{eq: regert informal} essentially depends on the reduction in the confidence radii (because both $f_1$ and $f_2$ in \eqref{eq: regert informal} are increasing functions in the radii). This leads us to choose $\lambda$ by minimizing the confidence radii, specified in Section \ref{subsec:data-pool-estimator}. Based on the chosen $\lambda$, we then incorporate the new estimators into the perturbed LSVI framework, specified in Section~\ref{subsec:main-result}.


\vspace{-0.12in}
\subsection{Data-pooling Estimators via Minimizing Confidence Radii}
\label{subsec:data-pool-estimator}


As highlighted in the design ideas, we choose the weight $\lambda$ to minimize the confidence radii to improve the decision quality.  The formal definition of the confidence radii, in the context of regret bound analysis, is as follows.  For a given confidence parameter $\delta$ and each $(h,s,a)$ tuple, we say that a sequence of functions  $\{\varepsilon_V(n),\varepsilon_P(n),\varepsilon_R(n): 0\leq n\leq T\}$ are the confidence radii for given estimators $\{\hat{r}_n, \hat{P}_n\}$ if  	
\begin{align}\label{eq: CI fixed n}
	&\PP\left(|\hat{r}_n-r(h,s,a)|>\varepsilon_R(n)\right)<\delta/HSAT,\notag\\
	&\PP\left(\|\hat{P}_n-P(h,s,a)\|_1>\varepsilon_P(n)\right)<\delta/HSAT,\\
	&\PP\left(\left\vert\hat{r}_n-r(h,s,a)+\langle\hat{P}_n-P(h,s,a), V^*_{h+1}\rangle\right\vert>\varepsilon_V(n)\right)<\delta/HSAT ,\notag
\end{align}
where $V^*_{h+1}(\cdot) = \max_{a\in\mathcal{A}} Q^*(h+1,\cdot,a)\in \mathbb{R}^{S}$ is the optimal value-to-go from horizon $h+1$. For example, the Hoeffding confidence radii $\{\varepsilon_{V}^H(n), \varepsilon_{P}^H(n),\varepsilon_{R}^H(n)\}$ mentioned in Section~\ref{subsec:regret-framework} satisfy the above inequalities for the MLE $\{\bar{r}_n,\bar{P}_n\}$ from the target data without pooling.

Next, for a given confidence parameter $\delta$ and the difference gap $\Delta$ as specified in Assumption \ref{assmpt: difference bound}, we derive an expression of the confidence radius for the data-pooling reward estimator $\hat{r}^{DP}$ (given in~\eqref{eq: data-pooling estimates fixed n}) as a function $\varepsilon(\lambda)$ of the weight $\lambda$. 
The function $\varepsilon(\lambda)$ is composed of two terms: one measures the deviation of $\hat{r}^{DP}_n$ from its mean, caused by its variability; and the other measures the bias of $\hat{r}^{DP}_n$ from $r(h,s,a)$, caused by the underlying difference between the target and historical patients. In other words, optimizing $\lambda$ is to strike a balance in the variance-bias trade-off. However, this variance-bias trade-off is different from the common ones in the statistics literature, in that the trade-off comes from the regret analysis via the confidence radii instead of prediction accuracy, which we will elaborate on later. Given $\varepsilon(\lambda)$, we then choose its minimizer $\argmin_\lambda \varepsilon(\lambda)$ as the weight for our data-pooling estimator. In Theorem \ref{thm: data-pooling estimator}, we specify the explicit form of the weight $\lambda^{DP}_{n,N}$, for $n$ target samples and $N$ historical samples, solved from the minimization problem. 
\begin{theorem}
	\label{thm: data-pooling estimator} 
	Let $\delta>0$ be a given confidence parameter. For any triple of $(h,s,a)$, let $\bar{r}_0$ and $\bar{P}_0$ be the MLE (empirical averages) for the reward and transition probabilities estimated from $N\geq 0$ i.i.d. observations in the historical data. Let $\bar{r}_n$ and $\bar{P}_n$ be the MLE from $n\geq 1$ i.i.d. observations in the target data. Under Assumption \ref{assmpt: difference bound} with $\Delta>0$, we define a data-pooling estimator as 
	
	$$\hat{r}^{DP}_n=\lambda_{n,N}^{DP} \bar{r}_n+(1-\lambda_{n,N}^{DP})\bar{r}_0,
	\quad 
	\hat{P}^{DP}_n = \lambda_{n,N}^{DP} \bar{P}_n+(1-\lambda_{n,N}^{DP})\bar{P}_0,$$
	with
	\begin{equation}\label{eq: lambda-RL}
		\lambda^{DP}_{n,N}=\lambda(n,\Delta, N,\delta)\triangleq\begin{cases}
			\frac{n+Nn\Delta/\sqrt{(N+n)\log(2HSAT/\delta)/2-\Delta^2Nn}}{N+n}& \text{if }n <\log(2HSAT/\delta)/2\Delta^2,\\
			~~~~~~~1& \text{otherwise.}
		\end{cases}
	\end{equation}
	Then, there exists a sequence of confidence radii  $\{\varepsilon^{DP}_R(n,N), \varepsilon^{DP}_P(n,N),\varepsilon^{DP}_V(n,N)\}$ for the data-pooling estimators $\{\hat{r}_n^{DP}, \hat{P}_n^{DP} \}$ such that $$\varepsilon_R^{DP}(n,N)\leq \varepsilon_R^H(n),\quad \varepsilon_P^{DP}(n,N)\leq \varepsilon_P^H(n), \quad \varepsilon_V^{DP}(n,N)\leq \varepsilon_V^H(n). 
	$$
	The inequalities are strict when $ n<\log(2HSAT/\delta)/2\Delta^2$ and $N\geq 1$. 
\end{theorem}

The proof for Theorem \ref{thm: data-pooling estimator}, including detailed expression of $\varepsilon(\lambda)$, is given in Appendix \ref{subsec: CI for data-pooling estimators}. 
Note that the explicit expression of the weight $\lambda^{DP}_{n,N}$ provides great benefit in the computational time, which is in contrast to other methods that are computationally burdensome when updating the estimates (e.g., contextual learning requires re-fitting the regression model). Moreover, the explicit form provides intuitive insights into the weight of using own data versus historical data. For example, $\lambda^{DP}_{n,N}$ is increasing in $n$ and $\Delta$. This implies that (i) when we collect more data from the target patients (larger $n$) during the learning, we can trust the estimate from the target data more and gradually reduce the reliance on using the historical data; and (ii) the larger the difference between the historical and target patients, the less (more cautious) we will use the historical data. 


\begin{remark}[Difference Gap]\label{rmk: difference gap}
In our theoretic analysis, the difference gap $\Delta$ is assumed known with sufficient prior knowledge. In practical implementation, in which $\Delta$ is not known exactly, we propose an easy-to-implement approach to estimate $\Delta$. We set $\Delta$ as the difference between the MLEs from the historical and target data multiplied by a constant $\gamma>0$. We treat $\gamma$ as a hyper-parameter and fine tune its value around one via cross-validation. We implement this approach in our case study in Section~\ref{sec:numerical} and show that there is a fairly large range of values of $\gamma$ that produce similar performance for our data-pooling algorithm (and beat other benchmarks), implying the robustness to the choice of $\Delta$ and the practical implementability of our approach. A detailed description of this approach is in Appendix~\ref{Appendix:algorithm details} with robustness check in~\ref{subsection:robust_check}. 
\end{remark}


\vspace{-0.2in} 
\subsection{Data-pooling Perturbed LSVI and Regret Bound}
\label{subsec:main-result}

With the new estimators provided in Theorem \ref{thm: data-pooling estimator}, the next step is to incorporate them within the perturbed LSVI framework to get the final data-pooling RL algorithm. This involves modifications to the two key components: new value function estimates and new perturbation term design. 
Note that in each iteration $t$ of an RL algorithm, the number of samples observed at a $(h,s,a)$ tuple from the target data, denoted by $n_t(h,s,a)$, is random and path-dependent. Theorem~\ref{thm: data-pooling estimator} is for a fixed sample size $n$. The results on fixed $n$ can be applied to the random and path-dependent $n_t(h,s,a)$ by conditioning on a ``good event'' and the pigeon's hole principle as in standard regret analysis.
We specify these details in the proof in the appendix and use the path-dependent $n_t(h,s,a)$ in the following analysis. We use $N(h,s,a)$ to denote the sample size for the tuple in the historical data, which is fixed through the learning procedure. 

\noindent\textbf{Modified Value Function Estimate. }
In each iteration $1\leq t\leq T$, substituting $n$ with $n_t(h,s,a)$ for the pooling estimators in Theorem~\ref{thm: data-pooling estimator}  we get the following  estimates
\begin{align}
	\label{eq: data-pooling estimates}
		\hat{r}_t^{DP}(h,s,a) &=\lambda_t(h,s,a)\bar{r}_t(h,s,a) + \big(1-\lambda_t(h,s,a)\big)\hat{r}_0(h,s,a),\notag\\
		\hat{P}^{DP}_t(h,s,a) &=\lambda_t(h,s,a)\bar{P}_t(h,s,a) + \big(1-\lambda_t(h,s,a)\big)\hat{P}_0(h,s,a),
\end{align}
where the weight $\lambda_t(h,s,a) = \lambda^{DP}_{n_t(h,s,a),N(h,s,a)}$ is computed according to \eqref{eq: lambda-RL}. Correspondingly, the data-pooling estimate for the value function $Q(h,s,a)$ follows 
\begin{equation}\label{eq: data-pooling value}
	\hat{Q}^{DP}_t(h,s,a)=\hat{r}^{DP}_t(h,s,a) +  \hat{P}^{DP}_t(h,s,a)\cdot\max_y\hat{Q}^{DP}_{t}(h+1,\cdot,y) 
\end{equation}
for $h=H, H-1, \dots, 1$ with $\hat{Q}^{DP}_t(H+1,s,a)\equiv 0$.

\noindent\textbf{Modified Perturbation Term. } 
The last piece in designing the data-pooling RL algorithm is to modify the perturbation terms. Recall that in~\eqref{eq: perturbation terms}, the perturbation terms $w_t(h,s,a)$ in the no-pooling algorithm are specified by the random variable $\xi$. We show that if $\xi$ can ensure sufficient exploration (to prevent getting stuck in suboptimal policies caused by estimation error) in the original no-pooling version, it also works for the data-pooling version and its associated confidence radii. As a consequence, we set the data-pooling perturbation terms as  
\begin{equation}\label{eq:data-pooling perturbation}
	w_t^{DP}(h,s,a) \equiv \varepsilon_{V}^{DP}(n_t(h,s,a))\cdot\xi_t(h,s,a),\quad\xi_t(h,s,a)\stackrel{i.i.d.}{\sim} \xi.
\end{equation}
Since $\varepsilon_{V}^{DP}(\cdot)\leq \varepsilon_{V}^H(\cdot)$, this modification leads to a smaller amount of perturbation while maintaining sufficient exploration, which reduces the part of the regret from the exploration cost.

\noindent\textbf{Data-pooling RL.} Given the modified value function estimator \eqref{eq: data-pooling value} and perturbation terms \eqref{eq:data-pooling perturbation}, we plug them into Algorithm \ref{Alg:P-LSVI} and get our data-pooling perturbed LSVI algorithm. See the complete specification in Algorithm \ref{Alg:DP-LSVI} below. Note that we use the perturbed value function: 
\begin{equation}
\label{eq: data-pooling value-perturb}
\tilde{Q}^{DP}_t(h,s,a) = \hat{r}^{DP}_t(h,s,a)+w^{DP}_t(h,s,a) +
\hat{P}^{DP}_t(h,s,a)\cdot\max_y\tilde{Q}^{DP}_t(h+1,\cdot,y),
\end{equation}
with $\tilde{Q}_t^{DP}(H+1,s,a)\equiv 0$.

\begin{algorithm}[thbp]
	\caption{Data-pooling Perturbed LSVI}
	\begin{algorithmic}
		\small{
			\STATE Initialize $\tilde{r}_{1}=0$ and $\bar{P}_1=0$, data set $\mathcal{D}=\{\}$,  $\tilde{Q}^{DP}_{t}(H+1,s,a)\equiv 0,~\forall t, s,a$. 
			\STATE Input hyperparameter $\Delta>0$ and $\delta$.
			\STATE Input $(N(h,s,a),\bar{r}_0(h,s,a),\bar{P}_0(h,s,a))$ obtained from the historical data.
			\FOR{$t=1$ {\bfseries to} $T$}
			\STATE Compute $\bar{r}_{t},\bar{P}_{t}$ from $\mathcal{D}$ 
			\FOR{$h=H,H-1,\ldots,1$}
			\FORALL{$(s,a)$ pairs}
			\STATE Generate $\xi_t(h,s,a)$ and compute $w_t^{DP}(h,s,a) =  \varepsilon_{V}^{DP}\big( n_t(h,s,a) \big) \cdot \xi_t(h,s,a)$ 
			\STATE Compute the data-pooling estimators $\hat{r}^{DP}_t(h,s,a)$ and $\hat{P}^{DP}_t(h,s,a)$ according to \eqref{eq: data-pooling estimates}
			\STATE Compute data-pooling perturbed value function $$\tilde{Q}_t^{DP}(h,s,a)=\hat{r}_t^{DP}(h,s,a)+w^{DP}_t(h,s,a)+\hat{P}^{DP}_t(h,s,a)
			\cdot\max_y\tilde{Q}_{t}^{DP}(h+1,\cdot,y)$$
			\ENDFOR
			\ENDFOR
			\FOR{$h=1,\ldots,H$}
			\STATE Collect state $S^t_h$
			\STATE Select action $A^t_h=\arg\max_a\tilde{Q}^{DP}_t(h,S^t_h,a)$
			\STATE Collect reward $R^t_h$
				\ENDFOR
				\STATE Update  $\mathcal{D}=\mathcal{D}\cup\{S^t_1,A^t_1,R^t_1,\ldots,S^t_H,A^t_H,R^t_H\}$
				\ENDFOR	
			}
		\end{algorithmic}
		\label{Alg:DP-LSVI}
	\end{algorithm}

Although we call our algorithm the data-pooling algorithm, it is a general representation of the data-pooling version of a given perturbed LSVI algorithm (the base algorithm). For example, if the base is \textit{TS-based} RL, then the corresponding data-pooling algorithm is the TS-based data-pooling algorithm. We use ``data-pooling RL'' as a general term and stress that what we provided here is a unified approach to covert a wide class of existing RL algorithms to their data-pooling version. Next, we state the performance guarantee of the data-pooling RL in Theorem \ref{thm: main}.

\begin{theorem}[Informal]\label{thm: main} Suppose that a perturbed LSVI algorithm satisfies certain technical conditions such that its regret is bounded by $$
		Regret(T) \leq (1+2p_0^{-1})HSA\sum_{n=1}^{\lceil  T/SA \rceil}\left(\varepsilon^H_R(n)+H(\bar{W}+1)\varepsilon^H_P(n)+\bar{W}\varepsilon_{V}^H(n)\right)+4HT\delta,
$$
where $\delta$ is the confidence parameter and the positive constants $\bar{W}$ and $p_0$ are determined by specific choice of the random variable $\xi$ for the perturbation terms in \eqref{eq: perturbation terms}.
Under Assumption~\ref{assmpt: difference bound}, the corresponding data-pooling version as given in Algorithm \ref{Alg:DP-LSVI} satisfies the same set of technical conditions and its regret 
	$$
	Regret(T) \leq (1+2p_0^{-1})HSA\sum_{n=1}^{\lceil  T/SA \rceil}\left(\varepsilon^{DP}_R(n,N)+H(\bar{W}+1)\varepsilon^{DP}_P(n,N)+\bar{W}\varepsilon_{V}^{DP}(n,N)\right)+4HT\delta,
	$$
where $N=\min_{h,s,a}N(h,s,a)$. Consequently, if $N(h,s,a)\geq 1$ $\forall (h,s,a)$ and $\Delta\leq \sqrt{\log(2HSAT/\delta)/2}$, the data-pooling version has a strictly smaller regret bound than the no-pooling version.   
\end{theorem}
The formal statement of Theorem~\ref{thm: main} and its proof are given in Appendix \ref{subsec: regret bound for data-pooling RL}. 
This theorem shows that, as long as three conditions are met (i.e., the base algorithm is well designed; the data difference $\Delta$ is small; and there are sufficient observations in the historical data\footnote{The data-pooling weight $\lambda_t(h,s,a)$ is computed according to $N(h,s,a)$ in the algorithm while we use $N$ in the theoretic regret analysis, mainly to obtain a simple expression of the regret bound. See more in remark~\ref{rmk:hist-sample-size} in Appendix~\ref{subsec: regret bound for data-pooling RL}.}), the data-pooling RL algorithm achieves a smaller regret bound compared to the original no-pooling algorithm. 

To recap, the entire design of our data-pooling estimator is driven by improving the regret performance, as now shown by Theorem~\ref{thm: main}. This highlights the fundamental difference between the design of our data-pooling estimator and those in the statistics literature: pooling for better decision quality versus for better prediction. Our data-pooling design does improve the accuracy of the estimator,
but this improvement is more intricate: the design improves value function estimation and reduces the unnecessary exploration in learning with only target data and, hence, supports  the end goal of improving decision quality. Note that, although our regret analysis is based largely on the perturbed LSVI framework and its regret bound~\eqref{eq: regert informal}, our approach of connecting the data-pooling design with regret analysis can be applied within other algorithm frameworks and motivate the design of even-more-efficient algorithms, upon better regret bound available. We leave this as an important future direction. The key novelty in this paper is to identify this new design approach.


\medskip
\noindent\textbf{Extension to Multiple Historical Groups. }
Our analysis has been focused on the case in which all historical data are generated from a single mixture MDP $\mathcal{M}_0$. In practice, it is common that multiple sets of historical data are available from different sources (populations). If the decision maker has prior information on this, we show that it is possible to choose different values of the weight parameter $\lambda_t$ for different historical datasets and one can further improve the regret performance. We formalize this idea in Appendix~\ref{subsec: extension-to-multi-group} with proof. 

\vspace{-0.07in}
\section{Case Study: Post-Discharge Intervention Planning}
\label{sec:numerical}

We demonstrate the practical
applicability and empirical success of our data-pooling RL algorithm via a specific healthcare problem: the personalized post-discharge intervention to reduce hospital readmissions. We first specify the dataset and calibrate the evaluation platform in Section~\ref{subs:Data}. 
In Section~\ref{subs:EmpResult}, we compare our data-pooling algorithm with other benchmarks and show the superior performance of ours in reducing the regret, the total cost, and the overall readmission rates. In Section~\ref{subsec:numerical on bandit}, we examine the key drivers behind the better performance of our algorithm to provide explainability. 
In Section~\ref{sec:robust-under-misspec}, we show the robustness of our algorithm with respect to model misspecification. 
In Section~\ref{subsec:robust-perform-extension}, we discuss insights for practical implementation. 

\vspace{-0.15in}
\subsection{Data Description and Simulator Calibration}
\label{subs:Data}

Our dataset comes from a large teaching hospital in Asia. It contains 21,596 records of patients discharged from July 2010 to February 2011. We focus on the 30-day readmission rate as the main outcome variable; extension to the 60-day or 90-day readmissions is straightforward. There are 21 feature variables, including basic demographics (e.g., age, gender), medical information (e.g., medical specialty, number of prior visits), insurance and diagnostic information, and intervention records (whether follow-up was scheduled upon discharge). 




We use the 9,766 samples from July to October 2010 as the historical data and calibrate a simulator for patients discharged between November 2010 and February 2011. The simulator provides the environment (the ``oracle'' $\mathcal{M}$ for each target class) for us to generate target data and evaluate policies. Note that we use the first part of the data from the same hospital to be the historical dataset, instead of using public datasets or published studies, for the benefits of benchmark algorithms such as contextual-based methods because they require individual-level information to train and learn the model. In Section~\ref{subsec:robust-perform-extension}, we show the advantage of our algorithm in requiring aggregate statistics only from the historical data. 


We consider four decision epochs corresponding to weekly decisions in the 30-day readmission window. The MDP setup follows the simple setting specified in Section~\ref{subs:MDP}, with two states $(s\in\{0,1\})$ and two actions $(a\in\{0,1\})$. The key to building the simulator lies in calibrating the transition probabilities, i.e., the readmission probability $P(s'=1 | h, s=0,a)$ for different individual classes in each week $h$ with or without the intervention. In calibrating these probabilities, the main difficulty is properly estimating the treatment effect, i.e., how the followup intervention impacts the readmission probability when \emph{selection bias} is present. 
That is, patients were not randomly selected for the follow-up intervention in the historical data. To begin with, the hospital would select patients for intervention that they believed to have a higher risk of readmission. These selected patients had higher baseline readmission rates than those who were not selected; hence, even though they received the intervention, they still had a higher chance of being readmitted than those who were not selected. 
To correct for the selection bias, we apply the Heckman's two-step correction~\citep{heckman1976common}. See Appendix~\ref{app:case-study} for the complete details.  

The performance of our two-step regression is measured by the out-of-sample Area Under the Curve (AUC), which is 0.71. This AUC performance is comparable with those reported in the literature for readmission prediction. The estimated treatment effect is significant, along with the following predictive features:
age, number of transfers, whether admitted from the emergency department (ED), number of prior visits, the Charlson comorbidity score, had an operations or not, and the hospital length-of-stay (LOS). This is consistent with findings in the medical literature on predictors of readmission risk, such as the LACE score~\citep{robinson2017hospital}. 
Based on these predictive features, we classify the target patients into $146$ individual classes as the ground truth. Table~\ref{table: target info example} shows five representative classes, and we observe that patients in different target classes can have very different optimal policies. These optimal policies are calculated by solving the corresponding MDP model (assuming that the true parameters are known) under the intervention cost and readmission cost, $c_a = 0.13$ and $c_R = 10$, respectively.\footnote{The cost is not normalized to $[0,1]$, as in the theoretical analysis, but that does not affect our comparison. The values for $c_a$ and $c_R$ are set to ensure enough difference in the optimal policies across patients. The values can be varied to generate tradeoff curves -- tradeoff between the number of interventions versus the readmission rates.} 
This significant difference in the optimal policies confirms the necessity of tailored intervention planning, as discussed in Section~\ref{subs:MDP}. The full table with all $146$ classes is available in Appendix~\ref{app:case-study}. 
For historical patients, we consider two setups: one treats all historical data points as generated from one mixed MDP $\mathcal{M}_0$ (i.e., single group for all historical patients); the other setup classifies the historical patients into eight groups based on the predictive features (see Table \ref{Table: historical data classification} in the Appendix). In the baseline experiments, we use the latter setup with eight groups; the former setup is discussed in Section~\ref{subsec:robust-perform-extension} for sensitivity analysis. 


\vspace{-0.06in}
\begin{table}[htbp] 
\centering
\scalebox{0.8}{
\begin{tabular}{ccccccccccc}
\toprule
Index & \begin{tabular}[c]{@{}l@{}}Sample \\ Size\end{tabular} & \begin{tabular}[c]{@{}l@{}}Weekly \\ Arrival\end{tabular} & \begin{tabular}[c]{@{}l@{}}Optimal \\ Policy\end{tabular} & \begin{tabular}[c]{@{}l@{}}Admission \\ Source\end{tabular} & Age         & \#Transfers & \begin{tabular}[c]{@{}l@{}}\#Prior \\ Visits\end{tabular} & \begin{tabular}[c]{@{}l@{}}Had oper-\\  ations\end{tabular} & \begin{tabular}[c]{@{}l@{}}Charlson \\ Score\end{tabular} & \begin{tabular}[c]{@{}l@{}}Length \\ of Stay\end{tabular} \\
\midrule
20    & 139                                                    & 7                                                        & 1111                                                      & ED                                                          & {[}61,80{]} & $\ge$ 2     & 0                                                        & 0,1                                                   & $\ge$ 4                                                   & 2-4                                                       \\
33    & 35                                                     & 2                                                        & 1111                                                      & ED                                                          & {[}41,61)   & 1           & 0                                                        & 0                                                         & 1-3                                                       & 2-4                                                       \\
93    & 139                                                    & 5                                                        & 1100                                                      & ED                                                          & {[}41,61)   & $\ge$ 3     & 1-3                                                      & 0,1                                                   & 0                                                         & 1                                                         \\
11    & 108                                                    & 4                                                        & 1000                                                      & nonED                                                          & {[}41,61{)} & $\ge$ 2     & 0                                                        & 1                                                   & 0                                                         & $\ge$ 2                                                         \\
144   & 128                                                    & 5                                                        & 0000                                                      & nonED                                                          & {[}18,41)   & 3           & 0                                                        & 1                                                         & 0                                                         & 1  
\\
\bottomrule                               
\end{tabular}
}
\caption{Representative target classes and their optimal policies (1 = intervention; 0 = no intervention), total sample size and number of arrival per week (rounded to integer). The length of stay is measured in days.
}\label{table: target info example}
\end{table}


\vspace{-0.15in}
\subsection{Algorithm Comparison and Results}
\label{subs:EmpResult} 

We use the calibrated simulator to evaluate the performance of our data-pooling algorithm and other benchmark algorithms. 
To evaluate a given algorithm, we run 100 replications with 50 iterations. In each iteration (equivalent to 1 week), new patients in each of the $146$ target classes are generated according to their empirical weekly arrival rates. The simulator generates patient evolution trajectories (data points) under the given algorithm. The data points collected in this iteration are added to the corresponding target dataset to help the algorithm prescribe actions for new patients in the next iteration. The historical data remains fixed. We report the average performance across the 100 replications and all target classes.

For our data-pooling algorithm, we use the \texttt{TS-based} RL as our base algorithm, i.e., the data-pooling version of the RLSVI algorithm. We compare the performance of our data-pooling algorithm with a few benchmarks in the following three categories. 
\begin{enumerate}[leftmargin=*]
\item Estimate-then-optimize (ETO) methods:
    \begin{itemize}
        \item The \textit{JS} method applies the James-Stein (JS) estimator~\citep{efron_hastie_2016} to estimate the reward and transition probability by merging the historical and target data. 
        \item The \textit{Contextual-P} method fits a regression model for the reward and transition probability by merging the target and historical data. 
    \end{itemize}
We adapt these ETO methods for the online setting for a more fair comparison, i.e., we continue to update parameters as new target data are collected, and perturbation terms are added to the value function estimation to encourage exploration.  
\item Learning with only target data: the \textit{Personalized} method uses RLSVI for learning but only utilizes the target data, i.e., the standard RLSVI without pooling. 
\item Learning with both target and historical data:
    \begin{itemize}
        \item The \textit{Complete} method uses RLSVI by simply merging the historical and target data in learning. Due to the dominantly large historical sample size, this method behaves almost like an offline method, as the optimal policies are driven by the historical data.
        \item The \textit{Contextual-Q} method fits a regression model for the Q-function. 
        \item The \textit{Clustering} method adaptively pools the target data and groups of historical data according to the distance of the estimated transition probability for each $(h,a)$ pair; see \cite{miao2019context} for motivation of this method and Appendix \ref{Appendix:algorithm details} for adaptation to our RL setting. 
    \end{itemize} 
\end{enumerate} 
The implementation details of data-pooling RLSVI and benchmark algorithms are in Appendix~\ref{Appendix:algorithm details}. Moreover, to give enough ``benefit of the doubt'' to the benchmark algorithms, we tune the necessary hyper-parameters for each benchmark and choose the best one to report. All the hyper-parameters used are specified in Appendix~\ref{Appendix: Hyper-parameters}.

Figure \ref{fig:regret personalized} displays the regret performance of our data-pooling RLSVI and other benchmarks. We see that the average regret of our method decreases much faster than others, especially in the first 10 iterations, in which the small-sample issue is particularly prominent. Table \ref{table: Performance using Corrected Historical data} reports the numeric values of the total regret for each algorithm, as well as the total cost, the main patient outcome (30-day readmission rate), and the computational time. From the table, it is clear that our data-pooling algorithm achieves the smallest total regret, the smallest cost, and the lowest readmission rate (except for two contextual methods) over 50 iterations. Compared to the two contextual methods, our algorithm enjoys faster computational performance. This benefits from the fact that the weight $\lambda_t$ in our data-pooling algorithm can be calculated very efficiently in each iteration. In contrast, contextual methods would require that we recalculate the regression in each iteration and this results in excessively long computational times. Among all the benchmark algorithms, the clustering and contextual-Q methods are the most competitive ones, with the performance gap with our algorithm being 37\% and 25\%, respectively. We take a closer look at the reasons that our algorithm outperforms these benchmarks in the next section. 

\begin{figure}[htbp]
	\begin{center}
		\centerline{\includegraphics[width=0.8\columnwidth]{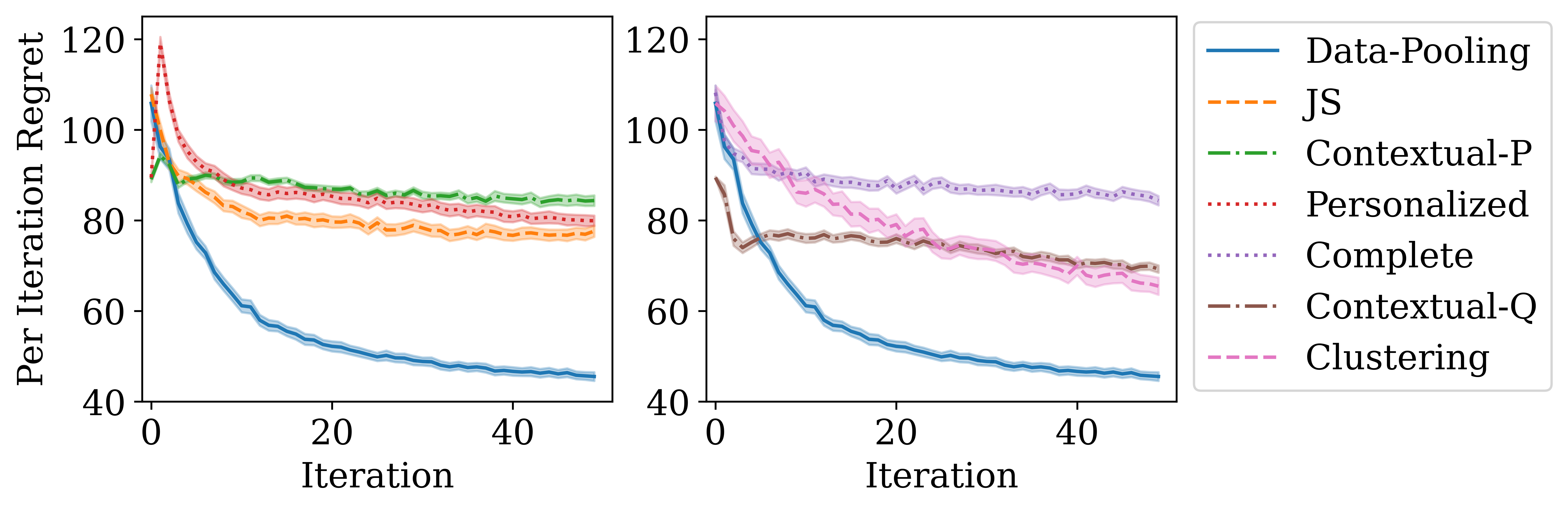}}
		\caption{Regret per iteration of 146 target patient classes. 
		The solid line represents the average from 100 replications, and the shaded area corresponds to the 95\% confidence interval.}
		\label{fig:regret personalized}
	\end{center}
	\vskip -0.2in
\end{figure}

\vspace{-0.15in}
\begin{table}[htbp]
\centering 
\scalebox{0.9}{
\begin{tabular}{@{}cllcc@{}}
\toprule
Algorithm             & Total Regret & Total Cost & \begin{tabular}[c]{@{}c@{}}Readmission \\  Rate\end{tabular} & CPU time/class \\ \midrule
 Data-pooling            &       2813.33$\pm$35.77 & 50663.38$\pm$115.06 & 9.80\% & 25.401 \\ 
 JS                      &        3994.96$\pm$43.89 & 51868.8$\pm$135.1 & 9.82\% & 16.93 \\ 
 Contextual-P &         4342.01$\pm$9.27 & 47297.71$\pm$91.78 & 9.73\% & 891.87 \\
 Personalized            &       4405.01$\pm$55.27 & 52261.2$\pm$138.66 & 10.58\% & 7.36 \\ 
  Complete                &       4095.43$\pm$14.83 & 51971.87$\pm$120.58 & 9.98\% & 6.47 \\ 
  Contextual-Q     &       3529.49$\pm$19.35 & 51392.18$\pm$112.42 & 8.93\% & 666.38 \\ 
  Clustering              &        3848.24$\pm$49.72 & 51692.41$\pm$117.9 & 10.29\% & 11.85 \\ 
 \hline
\end{tabular}
}
\caption{Performance of different algorithms in 50 iterations estimated via 100 rounds of simulation. 
The numbers following the $\pm$ sign are half-width of the 95\% confidence intervals. CPU times are from using the same server.}
\label{table: Performance using Corrected Historical data}
\end{table}

\subsection{Why Data-pooling Outperforms the Benchmarks}
\label{subsec:numerical on bandit}

The two benchmark algorithms in the first category (ETO methods), i.e., JS and contextual-P, use both the target data and the historical data. However, the way that they pool data is based on prediction accuracy, ignoring the subsequent needs of decision making. For example, the JS weight is chosen based on minimizing the MSE in the prediction. The numerical results clearly show that, without connecting to the decision-making quality, the conventional estimate-then-optimize procedure results in suboptimal performance even with proper adaptation to the online setting.

For the second category, the personalized method suffers from the small sample issue in the initial learning by using only the target data. It is evident from Figure~\ref{fig:regret personalized} that the regret per iteration is particularly high before iteration 20.

For the third category, we note that the optimal policy derived from the historical data is $(0000)$, i.e., no treatment (action $a=0$) in any of the four weeks. In the complete algorithm, the historical data dominates. Hence, the algorithm would prescribe this no-treatment policy for all patients at the beginning, which results in more than half of the patients not receiving their optimal policies. This sub-optimality demonstrates the significant performance gap between the one-size-fits-all policy and tailored intervention policies. Moreover, the large sample size in the historical data leads to a slow update in the policy for target patients, and the sub-optimality persists in every iteration, leading to the highest regret after 50 iterations. Next, we discuss the comparison with clustering in details, and leave the comprehensive comparison with the contextual-Q algorithm in Section~\ref{sec:robust-under-misspec}.

\noindent\textbf{Comparison with Clustering. }
A careful examination reveals that our data-pooling algorithm and the clustering algorithm are properly grouping the target patient class with similar patients from the historical data, where the similarity is in terms of the readmission risk in each week. However, the grouping was done in a ``soft'' way in our algorithm, but in a ``hard'' way in clustering. 

\begin{figure}[H]
\centering
\vspace{-25pt}
\includegraphics[width=0.4\textwidth]{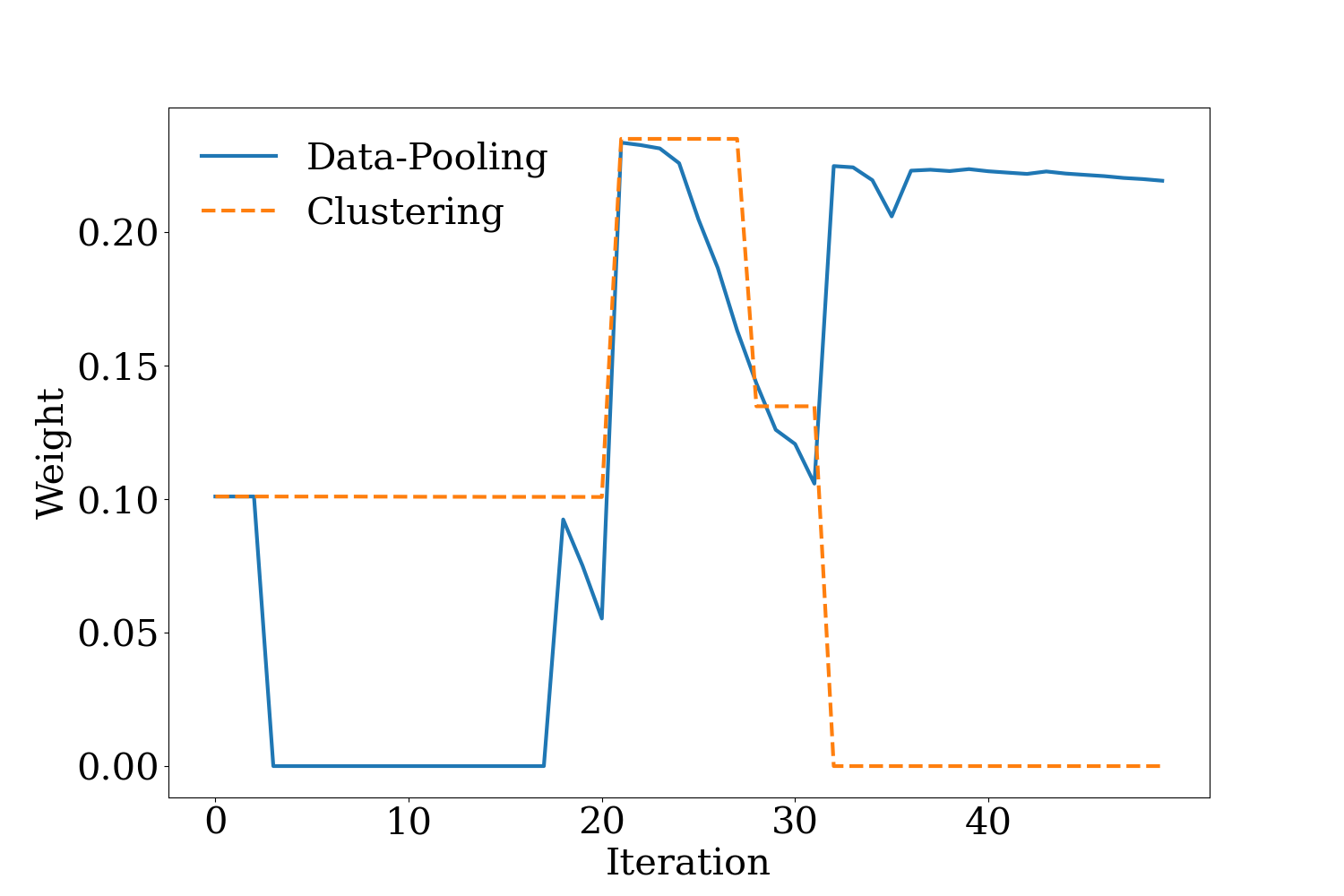}
\caption{Weight $\lambda$ of historical group $3$ at epoch $h=2$ with intervention ($a=1$).}
\label{fig:weight_comparison_clustering_pooling}
\end{figure}

\vspace{-25pt}
To illustrate, we use target class 1 as an example. 
Among the eight groups in the historical data, groups~ 3, 4, and 8 are the closest to target class 1, in the sense that they have the closest readmission risks with and without interventions. Zooming into epoch $h=2$ and action $a=1$, we find that within 20 iterations, data-pooling is able to learn the transition probability fairly closely to the optimal value on nearly all the sample paths. In contrast, clustering may end up staying at a point that is far away from the optimal on a significant number of paths, leading to the choice of the suboptimal action $a=0$ (even the exploration is not enough to compensate). To see why, we plot in Figure~\ref{fig:weight_comparison_clustering_pooling} the weight in pooling historical group 3 with the target class on one such sample path. Indeed, both algorithms correctly used the data from historical group 3 to estimate the transition probability. However, the weight under the clustering algorithm can change drastically from iteration to iteration, whereas the weight under our data-pooling algorithm is much smoother. This is because clustering essentially uses ``hard'' pooling. That is, if a historical group is deemed to be close enough to the target group under the radius calculation, to estimate the transition probability, the entire group is merged with the target data as if they are from the same population; otherwise, the entire group is excluded from estimation. This hard pooling renders the algorithm sensitive to bias caused by sampling variances and, hence, could wrongly include or exclude the group to be pooled. 
Group 3 was wrongly excluded from iterations 31 to 50 in Figure~\ref{fig:weight_comparison_clustering_pooling} as a result of this under clustering. In contrast, our data-pooling algorithm \emph{properly} accounts for group 3 in the learning, which leads to the more stable performance. See more analysis in Appendix~\ref{app:comparison-cluster-context}. 

\vspace{-0.1in}
\subsection{Synthetic Experiment: Significantly Different Historical vs Target Patients}
\label{sec:robust-under-misspec}

In this section, we compare the algorithm performance via a synthetic experiment for scenarios in which the reaction to responses is very different between the historical and target populations. We show that our algorithm, designed with model-free properties, significantly outperforms contextual-based methods when there are model misspecifications.

\noindent\textbf{Experiment Setting. } This synthetic experiment is constructed in a fashion similar to that of our real-data case study, except that we limit the target and the historical patients to two classes each. The underlying (true) linear model for the 30-day readmission risk is 
$p(a,x) = c_1 a + c_2 x + c_3$, 
which is the same as in the example constructed in Section~\ref{sec:model-misspec} (the state variable $s\in\{0,1\}$ is omitted). 
The continuous feature $x\sim \mathcal{N}(\mu, \sigma^2)$.  The coefficients for the four classes, as well as the feature distributions, are shown in Table~\ref{tab:synth-exp-setting}. We note that though $c_2$, $c_3$, and the feature distributions are different across the four classes, the overall readmission probabilities are similar. Thus, at a glance, it would be reasonable for decision-makers to pool the target and historical data to learn the readmission probability and policy for the target patients.


\vspace{-0.2in}
\begin{table}[H]
\centering
\begin{tabular}{c|c|c|c|c|c|c}
class & $c_1$ & $c_2$ & $c_3$ & mean($\mu$) & std($\sigma$) & average readm prob \\ \hline
target 0                            & -0.055                             & 0.5                                  & 0.l                        & 0.18       & 0.03  & 0.192 (0.137)      \\
target 1                            & -0.055                             & 0.5                                  & 0.l                          & 0.22       & 0.03 & 0.209 (0.154)      \\
history 0                           & -0.055                             & -0.5                                  & 0.3                          & 0.18       & 0.03 & 0.210 (0.155)     \\
history 1                           & -0.055                             & -0.5                                  & 0.3    & 0.22       & 0.03    & 0.190 (0.135)
\end{tabular}
\caption{Synthetic experiment settings. The average readmission probabilities column displays the probability without and with intervention(in the parentheses). }
\label{tab:synth-exp-setting}
\end{table}

\vspace{-0.3in}
\noindent\textbf{Different Time-dependent Response. }
Though the 30-day readmission rates are similar, the timing of readmission is significantly different between the historical and target groups, with two temporal patterns: high-low (35\% in the first two weeks, and 15\% in the latter two) for target 1 and history 0; and low-high (reverse of high-low) for target 0 and history 1. 
Under this construction, the final readmission probability in each week is summarized in Table~\ref{table: synthetic probability}, which also shows the optimal policy for each class. We can see that even though target 0 and history 0 (or target 1 and history 1) share similar feature values, their optimal policies are completely opposite due to the temporal patterns -- recall Proposition~\ref{prop: temporal pattern} on the structural properties for temporal patterns. 


\vspace{-0.15in}
\begin{table}[H]
\centering
\begin{tabular}{c|ccccccccc}
\multicolumn{1}{c|}{\textbf{class}} & \multicolumn{1}{c}{\textbf{$p_{00}$}} & \multicolumn{1}{c}{\textbf{$p_{01}$}} & \multicolumn{1}{c}{\textbf{$p_{10}$}} & \multicolumn{1}{c}{\textbf{$p_{11}$}} & \multicolumn{1}{c}{\textbf{$p_{20}$}} & \multicolumn{1}{c}{\textbf{$p_{21}$}} & \multicolumn{1}{c}{\textbf{$p_{30}$}} & \multicolumn{1}{c}{\textbf{$p_{31}$}} & \multicolumn{1}{c}{\textbf{opt policy}} \\ \hline
target  0                            & 2.9\%                            & 2.1\%                            & 2.9\%                            & 2.1\% & 6.7\% & 4.8\% & 6.7\% & 4.8\% & 0011   \\
target  1                            & 7.3\%                            & 5.4\%                            & 7.3\%                            & 5.4\% & 3.1\% & 2.3\% & 3.1\% & 2.3\% & 1100   \\
history 0                            & 7.4\%                            & 5.4\%                            & 7.4\%                            & 5.4\% & 3.2\% & 2.3\% & 3.2\% & 2.3\% & 1100   \\
history 1                             & 2.8\%                            & 2.0\%                            & 2.8\%                            & 2.0\% & 6.6\% & 4.7\% & 6.6\% & 4.7\% & 0011 
\end{tabular}
\caption{Average readmission rates over the four-week time window and the optimal policies for each class.}
\label{table: synthetic probability}
\end{table}

\vspace{-0.2in}
\noindent\textbf{Performance Comparison. } 
We compare the performance of three policies, generated from three algorithms, in this synthetic setting: (i) our data-pooling algorithm; (ii) the contextual-Q algorithm; and (iii) the contextual-P algorithm; the latter two are described in Section~\ref{subs:EmpResult}. 
We still run 100 replications with 50 iterations in each replication when evaluating each policy. Table \ref{table: regret synthetic} in Appendix~\ref{app:synth-exp} summarizes the performance in numeric values for the three policies. Our data-pooling method outperforms the two contextual-based methods in terms of the regret and cost, with the performance gaps of 81\% and 134\%, more significant than the baseline setting.


To elaborate on the reasons behind this larger performance gap, we plot the Q-function values under actions $a=0$ and $a=1$, as well as the difference in the Q-functions, respectively, for the three policies in Figure~\ref{fig:q_comparing_synthetic}. The estimated Q-function values from our data-pooling method are close to the optimal ones, whereas the estimates from the two contextual methods are far from optimal. So is the difference in the Q-function values. 

\vspace{-0.1in}
\begin{figure}[htbp]
\centering
\subfigure[$a=0$]{\includegraphics[width=0.28\linewidth]{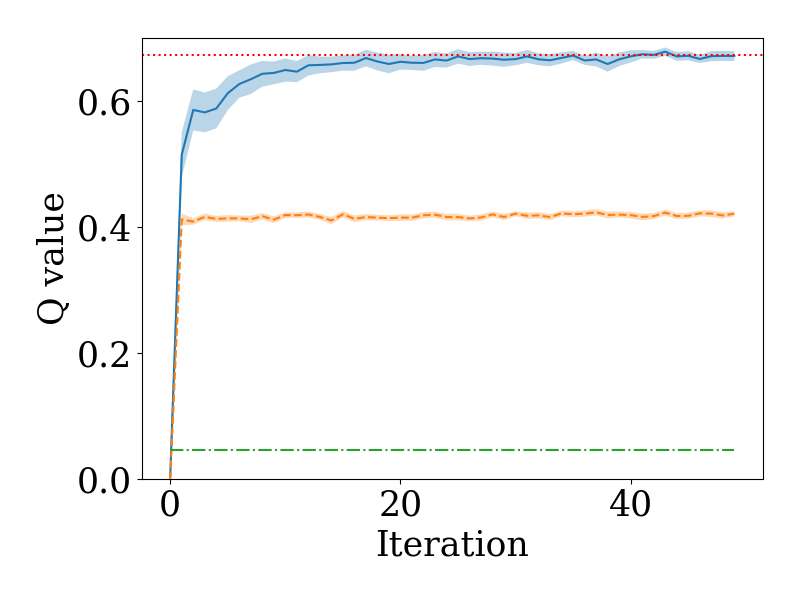}\label{fig:syn_q_0}} 
\subfigure[$a=1$]{\includegraphics[width=0.28\linewidth]{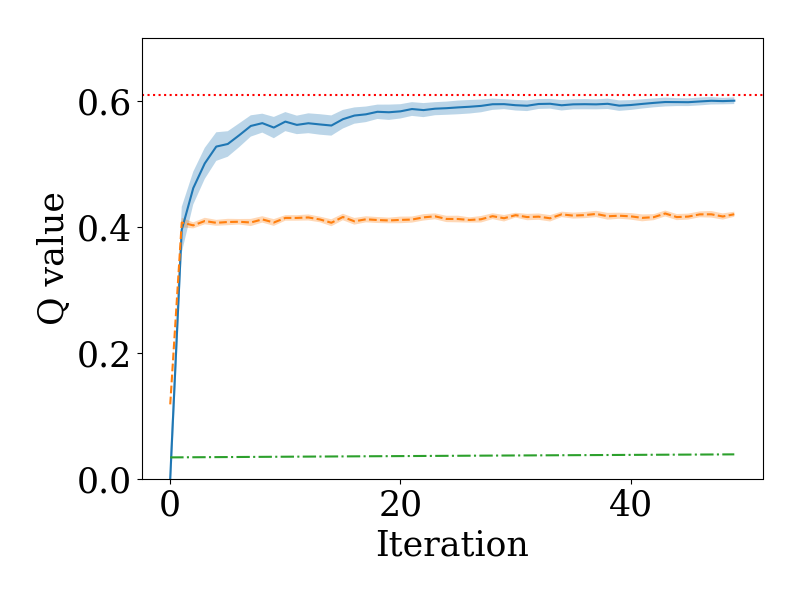}\label{fig:syn_q_1}} 
\subfigure[$Q_0 - Q_1$]{\includegraphics[width=0.42\linewidth]{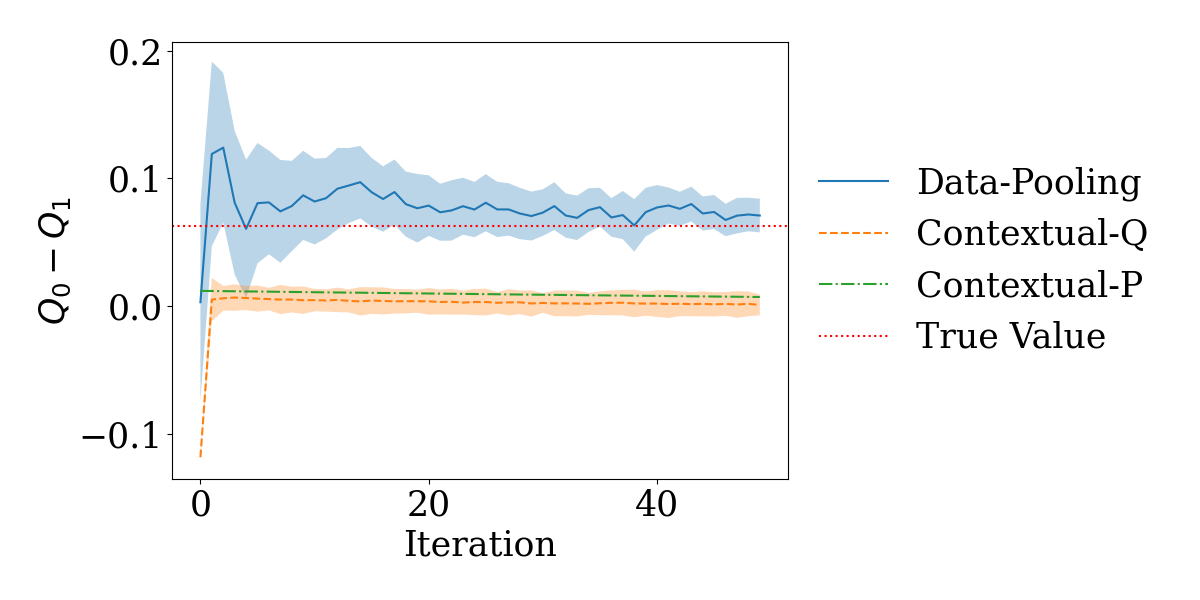}\label{fig:syn_q_diff}} 
\vspace{-0.1in}
\label{fig:q_comparing_synthetic}
\caption{Estimated Q-function value plot for target class 0 at period 0.
}
\end{figure}

\vspace{-0.2in}
The drastic difference in the estimation accuracy lies in how data were pooled by our method versus the contextual methods. Due to the regression mechanism, the contextual methods pool data by similar features. For example, the contextual-Q method directly estimates Q-function via a linear regression $Q_{h}(a;x) = \theta_1^h a +\theta_2^h x+\theta_3^h$. Since target class 0 and historical class 0 share the same feature distributions, the contextual method would tend to pool their data to learn a common Q-function depending on $x$. However, even though these two classes share similar feature values, their different temporal patterns lead to very different optimal Q-functions; for example, the optimal coefficient of action $a$, $\theta_1^h$, has opposite signs for target class 0 and historical class 0. As a result, when the contextual-Q method tries to fit a common function $Q_{h}(a;x)$ for the two classes, the fitted values become way off from the true values. This is exactly the issue caused by the model misspecification, as we formalized in Lemma~\ref{lemma: linear model for period}. The same explanation applies to the contextual-P method. In contrast, our data-pooling method pools data based on the closeness of the transition probabilities; for example, it pools target class 0 with historical class 1, and gets more accurate estimates $P_h(a)$ for $a=0,1$. This benefits from our model-free learning design. 

\subsection{Sensitivity Analyses and Practical Implications}
\label{subsec:robust-perform-extension}


\noindent\textbf{Sensitivity Analysis. } 
When implementing algorithm-based decision support in practice, there could be different operational barriers. We discuss how to deal with a few possible ones in sequence: data-privacy environment, robustness to hyper parameters, and different treatment effect. First, sharing health data across different organizations is extremely challenging due to privacy concerns. 
We design a numerical experiment in which data-sharing is restricted to aggregate statistics, and compare our mode-free algorithm with the clustering method adapted to this data-privacy environment (contextual methods no longer work here). Our algorithm significantly outperforms the other; see details in Appendix~\ref{app:data-privacy}. 
Second, the difference gap $\Delta$ is used as a tuning parameter for implementing the data-pooling algorithm. We run a robustness check and show that, without much tuning, our algorithm outperforms others; see Appendix~\ref{subsection:robust_check}. 
Finally, in~\ref{app:diff-treat}, 
we construct synthetic experiments in which the treatment effect is different and show that our data-pooling method continues to enjoy more robust performance than contextual methods.

\noindent\textbf{Practical Insights. }
Via the case study, we show that personalized intervention, rather than the one-size-fits-all policy, is necessary to improve patient outcomes, substantiating earlier analytical results. However, using only target data suffers greatly from the small-sample issue at the beginning of learning. In a data-limited environment, our algorithm can effectively facilitate the learning of new individuals' health trajectories and intervention strategies by properly pooling historical data. Even though the pooling idea is intuitive, accurately calculating the pooling weights in real time is unrealistic via rule-of-thumb or high-level medical intuition, particularly in a highly uncertain environment. This necessitates our sophisticated pooling development.


A somewhat ``surprising'' finding is that our algorithm developed under the model-free framework can achieve better performance than contextual-based methods, even though our algorithm does not use detailed contextual information. Importantly, our algorithm has even greater benefits than these methods when the target and historical patients are significantly different. Conventional wisdom usually suggests pooling data by similar features as in contextual or clustering methods, and a regression model is often in place to capture this similarity. 
The implicit assumption is that historical patients and new patients react to the treatment in a similar way as long as their features are similar. We show both analytically and numerically that when this assumption does not hold, the performance of conventional methods deteriorates significantly since the imposed parametric model can deviate from the true model. In contrast, our data-pooling method does not need to rely on such parametric assumptions and, hence, is robust to model misspecifications. This suggests that practitioners should consider model-free methods like ours when the underlying patients are highly heterogeneous, and historical data cannot be directly extended to a new environment.

Finally, the model-free nature of our algorithm has the built-in benefit of only requiring the sharing of aggregate statistics. This has significant implications for healthcare researchers and policy makers in the public sectors.
Healthcare researchers usually have access to many public datasets with aggregate-level data, but this has traditionally had limited value for personalized intervention/treatment. Our algorithm allows health organizations and researchers to augment their own datasets -- which may be detailed enough but limited in sample size -- with public datasets or published studies to improve the outcomes of individuals whom they manage. This was particularly useful when different interventions (for managing the same disease/condition) were used in different regions in the past. Our data-pooling algorithm allows one to quickly identify a smaller set of effective interventions in a new environment without the need to try all available options. Moreover, concerns about Private Health Information (PHI) make organizations reluctant to share individual-level data. The algorithm developed in this paper alleviates this major concern. It would make data owners who are concerned about PHI more willing to share aggregate statistics, which provides the basis for public policy makers to encourage organizations to share their aggregate data to facilitate population health outcomes for the broader healthcare community. This could further increase the volume of publicly available data, creating a positive feedback loop. 


\vspace{-0.1in}
\section{Conclusion and Future Work}
\label{sec:conclu}

In this paper, we study personalized preventative intervention planning as a multi-stage decision problem in the online learning setting, with unknown rewards and transitions. Via both analytical and numerical results, we show that managing a panel of heterogeneous patients is highly nontrivial and cannot rely solely on heuristics or human experience. Algorithmic-based, personalized decision support is necessary. To tackle the small-sample issue in personalization and improve performance over existing learning algorithms, we propose a novel data-pooling algorithm that properly incorporates data collected from other individuals (e.g., historical data, published studies), overcoming shortages associated with existing methods. 
We establish a theoretical performance guarantee and demonstrate the empirical success via a case study. From the technical perspective, our algorithm develops a unified and generalizable approach to convert a broad class of RL algorithms into their data-pooling version. For adoption in practice, we demonstrate how to use this algorithm under various operational constraints, particularly accounting for data privacy, which is often a barrier in the healthcare world. We also investigate the drivers behind the better performance of our algorithm over other benchmarks, providing explainability that is important for adoption by practitioners. We stress that, even though the algorithm works consistently with intuitions, the key is the delicate design of the pooling weight, with our innovation being tied into decision quality.

Regarding future directions, on the theoretical side, the current RL formulation we consider requires interactive engagement with target patients, i.e., on-policy evaluation. We leave the extension to the off-policy setting as an important but challenging future research direction. 
One could also design a more rigorous algorithm for on-policy learning in which target data from different individual classes are used in real time. On the practical side, we are in the process of working with some partner hospitals to design randomized control trials in the readmission context to evaluate the algorithm's performance in a real environment.

\bibliographystyle{pomsref}  \let\oldbibliography\thebibliography
\renewcommand{\thebibliography}[1]{%
	\oldbibliography{#1}%
	\baselineskip10pt 
	\setlength{\itemsep}{4.5pt}
}
\bibliography{ref2.bib}

\newpage
\ECSwitch 

\ECHead{
\begin{center}
Appendix for\\
``Data-pooling Reinforcement Learning for Personalized
Healthcare Intervention''
\end{center}
}

\section{Post-Charge Planning Problem}\label{app:post-charge}
Let $p_{ha}$ denote the transition (readmission) probability in decision epoch $h$ under action $a$. As shown in Figure \ref{fig:trans}, the readmitted state ($S=1$) is absorbing. In each epoch $h$, a cost of $c_a$ units occurs if the intervention is taken, and a  cost of $c_R$ units occurs if the patient transits to the readmitted state. We consider the setting in which $c_R \gg c_a$ and $p_{h0} > p_{h1}$; otherwise, there is no need for interventions.
\begin{figure}[H]
    \centering
    \includegraphics[scale=0.22]{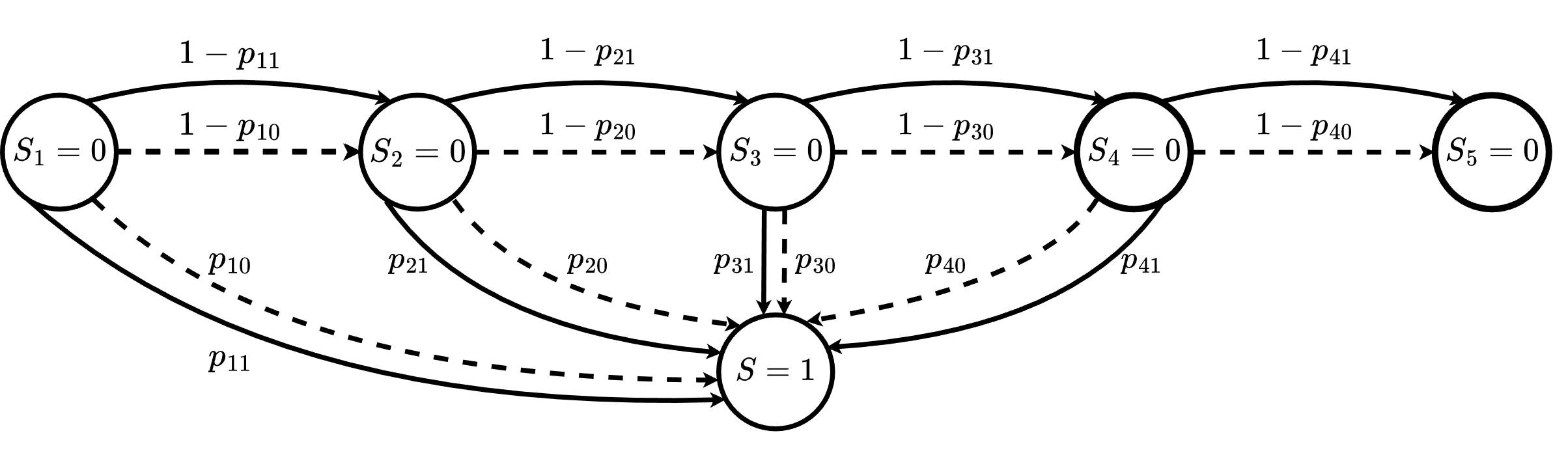}
    \caption{The transition diagram of the MDP model for post-discharge planning}
    \label{fig:trans}
\end{figure}

\section{Proofs of Structural Properties}
\label{app:proof-structural}

\subsection{Temporal Pattern}

\begin{proof}{Proof of Proposition \ref{prop: temporal pattern}}
		It suffices to prove that the optimal intervention policy $\{a_h^*:1\leq h\leq H\}$ must satisfy the following two statements:
		\begin{enumerate}
		    \item $a^*_1=...=a^*_{h_0-1}=0, a^*_{h_0}=...=a^*_{h_c}=1, \text{ for some }1\leq h_0 \le h_c.$
		   \item $a^*_{h_c}=...=a^*_{h_1-1}=1, a^*_{h_1}=...=a^*_{H}=0, \text{ for some }h_c \le h_1 \leq H.$
		\end{enumerate}

		We now prove the first statement by contradiction. In particular, we show that there must not exist $1 \leq h\leq h_c$ such that $a^*_h=1$ and $a^*_{h+1}=0$. Otherwise, by Bellman optimality equation,
		\begin{align*}
			a^*_{h}= 1  \Leftrightarrow \Delta_h > \frac{\beta}{1-V_{h+1}/c_R}, \quad
			a^*_{h+1}=0  \Leftrightarrow \Delta_{h+1} \le \frac{\beta}{1-V_{h+2}/c_R}.
		\end{align*}
		As a consequence, we have
		\begin{align*}
			\Delta_{h}-\Delta_{h+1} 
			& > \frac{\beta(V_{h+1}-V_{h+2})/c_R}{(1-V_{h+1}/c_R)(1-V_{h+2}/c_R)}\\
			& = \frac{\beta(\min\{c_{a}+p_{h+1, 1} (c_{R} -V_{h+2}), p_{h+1, 0} (c_{R} -V_{h+2})\})/c_R}{(1-V_{h+1}/c_R)(1-V_{h+2}/c_R)}\\
			& \stackrel{(i)}{=} \frac{\beta(p_{h+1, 0} (c_{R} -V_{h+2}))/c_R}{(1-V_{h+1}/c_R)(1-V_{h+2}/c_R)}\\
			& = \frac{\beta p_{h+1, 0} }{(1-V_{h+1}/c_R)} \ge \frac{\beta p_{h+1, 0} }{\prod_{i=h+1}^{H}\left(1-p_{i, 1}\right)}~,
		\end{align*}
		where equation (i) follows from the fact that $a^*_{h+1}=0$ and the last inequality follows from Lemma~\ref{lemma: lb value func ub} below. Since $p_{h,a}$ is monotone non-decreasing in $h$ for $1 \leq h \leq h_c - 1$ and $\Delta_h \geq 0$, then
		$$
		\frac{\Delta_h}{p_{h,0}} - \frac{\Delta_{h+1}}{p_{h+1,0}} \geq 
		\frac{\Delta_h}{p_{h+1,0}} - \frac{\Delta_{h+1}}{p_{h+1,0}} \geq  \frac{\beta}{\prod_{i=h+1}^{H}\left(1-p_{i, 1}\right)}
		$$
		This leads to a contradiction with Condition (b) and completes the proof of the first statement.
		
		We prove the second statement also by contradiction. We show that there must not exist $h_c \leq h \leq H$ such that $a^*_h=0$ and $a^*_{h+1} =1$. Otherwise, by Bellman optimality equation,
		\begin{align*}
			a_{h}^*= 0 & \Leftrightarrow \Delta_h \le \frac{\beta}{1-V_{h+1}/c_R}, \quad
			a_{h+1}^*=1  \Leftrightarrow \Delta_{h+1} > \frac{\beta}{1-V_{h+2}/c_R}.
		\end{align*}
		As a consequence,
		\begin{align*}
			\Delta_{h}-\Delta_{h+1}
			& < \frac{\beta(V_{h+1}-V_{h+2})/c_R}{(1-V_{h+1}/c_R)(1-V_{h+2}/c_R)}\\
			& = \frac{\beta(\min\{c_{a}+p_{h+1, 1} (c_{R} -V_{h+2}), p_{h+1, 0} (c_{R} -V_{h+2})\})/c_R}{(1-V_{h+1}/c_R)(1-V_{h+2}/c_R)}\\
			& < \frac{\beta(p_{h+1, 0} (c_{R} -V_{h+2}))/c_R}{(1-V_{h+1}/c_R)(1-V_{h+2}/c_R)}\\
			& = \frac{\beta p_{h+1, 0} }{(1-V_{h+1}/c_R)} \le \frac{\beta p_{h+1, 0} }{\prod_{i=h+1}^{H}\left(1-p_{i, 0}\right)},
		\end{align*}
		where the last inequality follows from Lemma \ref{lemma: lb value func ub}. Since $p_{h,a}$ is monotone non-increasing in $h$ for $h_c \leq h \leq H$ and $\Delta_h \geq 0$, then
		$$
		\frac{\Delta_h}{p_{h,0}} - \frac{\Delta_{h+1}}{p_{h+1,0}} \leq 
		\frac{\Delta_h}{p_{h+1,0}} - \frac{\Delta_{h+1}}{p_{h+1,0}} \leq  \frac{\beta}{\prod_{i=h+1}^{H}\left(1-p_{i, 0}\right)}
		$$
		This leads to a contradiction with Condition (b) and completes the proof.
		\hfill $\Box{}$
	\end{proof}


	\begin{lemma} \label{lemma: lb value func ub}
		If $~V_{H+1} = 0$ and $p_{h,0} \geq p_{h,1}$, then
		$$\left(1-\prod_{i=h}^{H}\left(1-p_{i, 1}\right)\right) c_{R} \leq V_{h} \leq\left(1-\prod_{i=h}^{H}\left(1-p_{i, 0}\right)\right) c_{R}, \forall~1\leq h\leq H.$$
	\end{lemma}
	\begin{proof}{Proof of Lemma \ref{lemma: lb value func ub}}
		We first prove the second inequality by induction. For $h = H$, by Bellman optimality equation,
		\begin{align*}
			V_H & = \min\{c_a + p_{H,1} c_R + (1 - p_{H,1}) V_{H+1}, p_{H,0} c_R + (1 - p_{H,0}) V_{H+1}\}   \leq p_{H,0} c_R,
		\end{align*}
		where we used the condition $V_{H+1} = 0$. Thus, the inequality holds for $h = H.$ Suppose the inequality holds for $h + 1$. Then, by Bellman optimality equation,
		\begin{align*}
			V_h & = \min\{c_a + p_{h,1} c_R + (1 - p_{h,1}) V_{h+1}, p_{h,0} c_R + (1 - p_{h,0}) V_{h+1}\} \\
			& \leq p_{h,0} c_R + (1 - p_{h,0}) V_{h+1} \\
			& \leq p_{h,0} c_R + (1 - p_{h,0}) \left(1-\prod_{i=h+1}^{H}\left(1-p_{i, 0}\right)\right) c_{R} \\
			& =\left(1-\prod_{i=h}^{H}\left(1-p_{i, 0}\right)\right) c_{R},
		\end{align*}
		where the last inequality follows from the inductive assumption. Thus, we can conclude that, if the second inequality holds for $h+1$, it also holds for $h$ and, therefore, it holds for all $h$ by induction.
		
		Next, we prove the first inequality by induction. For $h = H$, by Bellman optimality equation,
		\begin{align*}
			V_{H} & = \min\{c_a + p_{H,1} c_R + (1 - p_{H,1}) V_{H+1}, p_{H,0} c_R + (1 - p_{H,0}) V_{H+1}\} \\
			& \geq \min\{p_{H,1} c_R, p_{H,0}c_R \} \geq p_{H,1} c_R
		\end{align*}
		where we used the condition $V_{H+1} = 0$. Thus, the inequality holds for $h = H.$ Suppose the inequality holds for $h+1$. Then, by Bellman optimality equation, 
		\begin{align*}
			V_h & = \min\{c_a + p_{h,1} c_R + (1 - p_{h,1}) V_{h+1}, p_{h,0} c_R + (1 - p_{h,0}) V_{h+1}\} \\
			& = \min\{c_a + p_{h,1} (c_R - V_{h+1}) + V_{h+1}, p_{h,0} (c_R - V_{h+1}) + V_{h+1}\} \\
			& \geq \min\{p_{h,1} (c_R - V_{h+1}) + V_{h+1}, p_{h,0} (c_R - V_{h+1}) + V_{h+1}\} \\
			& \geq p_{h,1} c_R + (1 - p_{h,1}) V_{h+1} \\
			& \geq \left(1-\prod_{i=h}^{H}\left(1-p_{i, 1}\right)\right) c_{R},
		\end{align*}
		where the last inequality follows from the inductive assumption. By the same induction argument as above, we complete the proof. 
		\hfill$\Box{}$
	\end{proof}

\subsection{Necessity of Estimating Risk Trajectory}
\label{app:pat-pri}

One may argue that some simple rule-of-thumb on patient prioritization could avoid the hassle of estimating the entire risk trajectory. Although there are special cases in which such ranking exists (see Proposition~\ref{prop: different patients} below), the conditions are often too restrictive to hold, and, in general, there is no easy ranking structure for practitioners to use. We first state the special-case proposition and then use two examples to show that prioritizing with respect to neither the absolute risk level nor the treatment effect is optimal.

Without loss of generality, consider two patients with different levels of readmission risks. We use superscripts $u$ and $l$ to denote the patients with high risk and low risk, respectively. 
\begin{proposition}[Prioritizing in Special Cases]
\label{prop: different patients}
	Suppose that for all $1\leq h\leq H$,
		\begin{equation}\label{eq: high and low risk patients}
			\Delta^l_h - \Delta^u_h < \frac{\beta}{\prod_{i=h+1}^{H}\left(1-p^l_{i, 1}\right)} - \frac{\beta}{\prod_{i=h+1}^{H}\left(1-p^u_{i, 0}\right)}, ~\forall 1 \leq h \leq H.
		\end{equation}	
		Then, for all $1\leq h\leq H$, if intervention $a=1$ is optimal for low-risk patient at epoch $h$, $a=1$ must also be optimal for high-risk at the same period, i.e.
		$a^{*, l}_{h}=1\quad \Rightarrow \quad a^{*, u}_{h}=1.$
	\end{proposition}

\begin{proof}{Proof of Proposition \ref{prop: different patients}}
We prove the statement by contradiction. We will show that there must not exist $1 \leq h \leq H$ such that $a^{l,*}_h = 1$ and $a^{u,*}_h = 0$. By Bellman optimality equation,
\begin{align*}
			a^{l,*}_{h} = 1 \Leftrightarrow \Delta^l_h > \frac{\beta}{1-V^l_{h+1}/c_R}, \quad
			a^{u,*}_{h} = 0 \Leftrightarrow \Delta^u_h \le \frac{\beta}{1-V^u_{h+1}/c_R}.
\end{align*}
As a consequence,
\begin{align*}
			\Delta^l_h - \Delta^u_h & > \frac{\beta}{1-V^l_{h+1}/c_R} -\frac{\beta}{1-V^u_{h+1}/c_R}\geq \frac{\beta}{\prod_{i=h+1}^{H}\left(1-p^l_{i, 1}\right)} - \frac{\beta}{\prod_{i=h+1}^{H}\left(1-p^u_{i, 0}\right)},
\end{align*}
where the inequality follows Lemma \ref{lemma: lb value func ub}. This leads to a contradiction with         \eqref{eq: high and low risk patients}.\hfill$\Box{}$
\end{proof}
We can obtain a similar sufficient condition for prioritizing low-risk patients via modifying inequality~\eqref{eq: high and low risk patients} accordingly.

\noindent\textbf{Counter-examples. }
As illustrated by the following two examples, the conditions identified in Proposition \ref{prop: different patients} do not hold in general, hence, no simple rule-of-thumb on prioritization in general. 
\begin{example}\label{exmp: counter example 1} We construct a simple example with $H=2$, and we set $c_a = 1$ and $c_R = 10$. 
	The readmission probabilities and treatment effects of the two patients follow:
\begin{align*}
\textrm{High-risk:\quad} & p^u_{1,0} = 0.9, p^u_{1,1} = 0.5 \textrm{ for } h=1; p^u_{2,0} = 0.55, p^u_{2,1} = 0.5 
\textrm{ for } h=2;
\\
\textrm{Low-risk:\quad} & p^l_{1,0} = 0.5, p^l_{1,1} = 0.1 \textrm{ for } h=1; p^l_{2,0} = 0.5, p^l_{2,1} = 0.1 
\textrm{ for } h=2. 
\end{align*} 
We can solve that the optimal policy for the low-risk patient is $\{1,1\}$ and for the high-risk patient is $\{1, 0\}$; i.e., the optimal policy prioritizes the low-risk patient at $h=2$.
	\end{example}

\begin{example}\label{exmp: counter example 3}
This example illustrates that prioritizing over treatment effect is not optimal either. Consider the same setting as in Example~\ref{exmp: counter example 1}, except that the costs $c_a = 1.5$ and $c_R = 10$ ($\beta=0.15$) and the readmission risks change to: 
		\begin{align*}
			\text{High-risk: \quad}& p^u_{1,0} = 0.43, p^u_{1,1} = 0.2 \textrm{ for } h=1; p^u_{2,0} = 0.43,p^u_{2,1} = 0.2 \textrm{ for } h=2;\\ 
			\text{Low-risk: \quad}& p^l_{1,0} = 0.19,p^l_{1,1} = 0.01 \textrm{ for } h=1; p^l_{2,0} = 0.19,p^l_{2,1} = 0.01 \textrm{ for } h=2. 
		\end{align*}
In this case, the high-risk patient also yields more significant treatment effects; i.e. $\Delta^u_h=0.23>0.18 = \Delta^l_h$, for $h=1,2.$ 
However, the high-risk patient is still not prioritized. 
The optimal intervention decisions solved for the low-risk and high-risk patients are $\{1,1\}$ and $\{0,1\}$, respectively. 
\end{example}

\subsection{Model Misspecification}

	\begin{proof}{Proof of Proposition \ref{lemma: linear model for period}}
Since
	\begin{align*}
	p_i^h(a, x)  & = \alpha^h_i p(a,x)
	 = \alpha^h_i(c_{1} a+c_{2} x + c_{3})\\
	& = \alpha^h_i c_{1} a + \alpha^h_i c_{2} x + \alpha^h_i c_{3} = c^h_{1,i} a + c^h_{2,i} x + c^h_{3,i},
	\end{align*} 
	with $c_{j, i}^{h}=\alpha^{h}_{i} c_{j}$, we obtain part (I). 

	 Let $B$ be a random variable corresponding to the class index, i.e., $\PP(B=i)=q_i$ for $i=1,2$. Denote by $f_i(x)$ the density of feature $x$ in class $i$. Then, we can compute the conditional probability
\begin{align*}
    P(B=i|X=x) & 
    = \frac{q_i f_i(x)}{\sum_{j=1}^2 q_j f_j(x)} \triangleq g_i(x), \text{ for }i=1,2.
\end{align*}
Note that under the assumption that the feature distributions $f_1(x)$ and $f_2(x)$ are different, the function $g_i(x)$ varies with $x$ and can not be a constant for $i=1,2$. By definition, 
\begin{align*}
    p^h(a,x) &=  g_1(x)p_1^h(a,x) + g_2(x)p_2^h(a,x) =  \left(g_1(x)\alpha^h_1+g_2(x)\alpha^h_2\right)\left(c_{1}a+c_{2}x+c_{3}\right)\\
    & = \left(\alpha^h_2+g_1(x)(\alpha^h_1-\alpha^h_2)\right)\left(c_{1}a+c_{2}x+c_{3}\right),
\end{align*}
where the second equality follows from the fact that $g_1(x)$ and $g_2(x)$ sum up to one. 
As $\alpha_1^h\neq \alpha_2^h$ and  $g_1(x)$ is not constant in $x$, we can conclude that $p^h(x,a)$ cannot be a linear function of $x$.\hfill$\Box{}$

   

\end{proof}
\section{Additional Details for the Perturbed LSVI Algorithm}
\label{app:additional-PLSVI}

\subsection{Value Iteration and Policy Iteration for Finite-horizon MDP}\label{section:VIPI}
For an MDP model $\mathcal{M}=(H, \mathcal{S}, \mathcal{A}, R, P)$ and a policy $\pi$, we define the state-value function (value-to-go) for $0 \leq h \leq H$ as
$$
V_h(\mathcal{M}, \pi)(s):= \mathbb{E}\left[\sum_{t=h}^H R\left(t, s_t, a_t=\pi\left(s_t,t\right)\right) \mid s_h=s\right],
$$
and the optimal value function
$
V_h(\mathcal{M}, \pi^*):= \max\limits_{\pi}V_h(\mathcal{M}, \pi)$. The state-action value function (Q-function) is defined as
$$
Q_h(\mathcal{M}, \pi)(s, a):= \mathbb{E}\left[\sum_{t=h}^H R\left(t, s_t, a_t=\pi\left(s_t,t\right)\right) \mid s_h=s, a_h=a\right], 
$$
and the optimal Q-function
$
Q_h(\mathcal{M}, \pi^*)(s, a):= \max\limits_{\pi}Q_{h}(\mathcal{M}, \pi)(s, a)$. For a given policy $\pi$, the value functions can be computed recursively as: 
\begin{align*}
 V_h(\mathcal{M}, \pi)(s) 
& =r(h,s,\pi(s,h))+\sum_{s^{\prime} \in \mathcal{S}} P(s'; h, s, a) V_{h+1}(\mathcal{M}, \pi)(s')  ,\\
Q_h(\mathcal{M}, \pi)(s,a) & =r(h,s,a) + \sum_{s^{\prime} \in \mathcal{S}} P(s'; h, s, a)  Q_{h+1}(\mathcal{M}, \pi)(s',\pi(s',h+1)),
\end{align*}
with $V^{H+1}(\mathcal{M},\pi)$ and $Q^{H+1}(\mathcal{M},\pi)\equiv 0$.
For simplicity of notation, we denote the optimal Q-functions by $Q^*(h,s,a):=Q_h(\mathcal{M},\pi^*)(s,a)$.
The Bellman optimality equation characterizes $Q^*$: 
\begin{equation}\label{eq: bellman optimal Q function}
Q^*(h,s,a)  =r(h,s,a) + \sum_{s^{\prime} \in \mathcal{S}} P(s'; h, s, a) \max\limits_{y\in \mathcal{A}} Q^*(h+1,s',y),\quad \text{ with }Q^*(H+1,\cdot,\cdot)\equiv 0.
\end{equation}
Policy iteration (see in Algorithm \ref{alg:policy iteration}) and value iteration (see in Algorithm \ref{alg:value iteration}) are among the most fundamental methods for computing optimal policy of an MDP. The policy iteration algorithm iteratively evaluates and improves the policy until convergence, while the value iteration algorithm computes the optimal Q-function and decides the policy accordingly. 

\vspace{-0.05in}
\begin{algorithm}[htbp]
    \caption{Policy Iteration Algorithm}
    \label{alg:policy iteration}
\begin{algorithmic}
\STATE Step 1. Initialization. Start with initialization $Q(H+1, s, a)=0, \forall s\in\mathcal{S}, a\in\mathcal{A}$
\STATE Step 2. Policy Evaluation.
 \FOR{$h=H,H-1,\ldots,1$}
\STATE 
$$
Q(h,s,a) \leftarrow r(h,s,a) + \sum_{s^{\prime} \in \mathcal{S}} P(s'; h, s, a) Q(h+1,s',\pi_h(s'))
$$
\ENDFOR
\\
\STATE Step 3. Policy Improvement
\quad $\pi_h(s) \leftarrow \arg \max _a Q(h,s,a), \forall 1 \leq h \leq H$
\IF {$\pi_h(s)$ changes} 
    \STATE go to Step 2.
\ELSE
    \STATE stop and return $\pi_h(s)$
\ENDIF
\end{algorithmic}
\end{algorithm}


\vspace{-0.2in}
\begin{algorithm}[htbp]
\caption{Value Iteration Algorithm}
\label{alg:value iteration}
    \begin{algorithmic}
 \STATE  1. Start with initialization $Q(H+1, s, a)=0, \forall s\in\mathcal{S}, a\in\mathcal{A}$
 \STATE 2. Compute Optimal Q-function
 \FOR{$h=H,H-1,\ldots,1$}
\STATE 
$$
Q(h,s,a) \leftarrow r(h,s,a) + \sum_{s^{\prime} \in \mathcal{S}} P(s'; h, s, a) \max\limits_{y\in \mathcal{A}} Q(h+1,s',y)
$$
\ENDFOR
\STATE Step 3. Get optimal policy
\FORALL{$s \in \mathcal{S}, 1 \leq h \leq H$}
\STATE $$\pi_h(s) \leftarrow \arg \max _a Q(h,s,a)$$
\ENDFOR
    \end{algorithmic}
\end{algorithm}

 
\subsection{The Least-Squares Value Iteration(LSVI)}\label{section:LSVI}
The original LSVI algorithm was proposed by \cite{bradtke1996linear}. For a given target MDP $\mathcal{M}$ and its generated target data $\mathcal{D}$, LSVI uses a linear approximation with feature function $\phi(s,a)$ and parameter $\beta_h\in \mathbb{R}^{K}$, which is estimated via least squares method, to approximate the value function for each state-action pair. Specifically, we set the optimal Q-function $Q^* (h, s,a) \approx \beta_h \cdot \phi(s,a)\equiv Q^{\beta_h}(s,a)$. 
Given target data $\mathcal{D}_h$ in decision epoch $h$, we learn the coefficients $\beta_h$ via minimizing the empirical temporal-difference (TD) loss function: 
\begin{equation}
\mathcal{L}(\beta_h; \hat{\beta}_{h+1}, \mathcal{D}_h) 
= \sum_{i\in\mathcal{D}_h} \left(r_i + \max_{y\in\mathcal{A}} Q^{\hat{\beta}_{h+1}}(s'_i, y) - \beta_h\cdot \phi(s_i, a_i) \right)^2,
\label{eq:convent-LSVI}
\end{equation}
where $s_i'$ is the state observed in the next decision epoch after data point $i$ and $Q^{\hat{\beta}_{h+1}}(s,a)=\hat{\beta}_{h+1} \cdot \phi(s,a)$ is the estimated value of $Q^*(h+1,s,a)$ for epoch $h+1$, with least square estimator $\hat{\beta}_{h+1}$.
The TD loss in \eqref{eq:convent-LSVI} function comes from the Bellman optimality equation \eqref{eq: bellman optimal Q function}. 

In the tabular case, the feature functions become indicators for $(s,a)$ pairs and the true value of parameter $\beta_{h,s,a}=Q^*(h,s,a)$. Then, for each $(h,s,a)$ tuple, the solution for \eqref{eq:convent-LSVI} yields
$$\hat{\beta}_{h,s,a} = \bar{r}(h,s,a) + \sum_{s'}\bar{P}(s';s,a)\max_{y\in\mathcal{A}}\hat{\beta}_{h+1,s',y},$$
or equivalently,
$$\hat{Q}(h,s,a) = \bar{r}(h,s,a) + \sum_{s'}\bar{P}(s';s,a)\max_{y\in\mathcal{A}}\hat{Q}(h+1,s',y),$$
where $\bar{r}$ and $\bar{P}$ are the empirical averages (MLE) given by 
\begin{align*}
  \bar{r}(h,s,a)   = \frac{ \sum_{i\in\mathcal{D}_h}r_i\cdot 1\left\{\left(s_i, a_{i}\right)=(s, a)\right\} }{\sum_{i\in\mathcal{D}_h} 1\left\{\left(s_i, a_{i}\right)=(s, a)\right\} }, ~
 \bar{P}(s';h, s, a)=\frac{ \sum_{i\in\mathcal{D}_h} 1\left\{\left(s_{i}, a_{i}, s'_{i}\right)=\left(s, a, s^{\prime}\right)\right\}}{\sum_{i\in\mathcal{D}_h} 1\left\{\left(s_{i}, a_{i}\right)=\left(s, a\right)\right\}}.  
\end{align*}
\subsection{UCB-based VI}

In UCB-based algorithm, given the history up to iteration $t$, the perturbation term is deterministic and is proportional to the confidence radius of the value function. Specifically, 
\begin{equation}
w^t(h,s,a)=\min\Big(H-h+1,~ L(H,S,A)\varepsilon_V^H\big( n_t(h,s,a)\big) 
\Big),
\label{eq:UCB-perturb-term}
\end{equation}
where $L(H,S,A)$ is some constant depending on the number of horizons $H$, state space size $S$ and action space size $A$, and $n_t(h,s,a)$ the number of observations in $(h,s,a)$ by the $t$-th iteration. 
Intuitively, the perturbation term is injected to $(h,s,a)$ tuples that are visited less frequently (sampled less) to encourage exploration. UCB-based VI algorithm has many different versions, e.g., UCRL \citep{ortner2007logarithmic}, UCRL2 \citep{jaksch2010near}, and UCBVI \citep{azar2017minimax}. 

\subsection{RLSVI}

In RLSVI, a random Gaussian noise is injected into the estimated value functions from the target data. \cite{osband2016generalization} and \cite{russo2019worst} show that this perturbation is equivalent to sampling an MDP, whose value function is drawn from some Gaussian distribution centered around the empirical mean of the true MDP of the target class (i.e., the maximum likelihood estimator). 
For the tabular representation, \citet{russo2019worst} suggested using the perturbation term 
\begin{equation} 
w_t(h,s,a)=\varepsilon_V^H(n_t(h,s,a))\cdot\xi_t(h,s,a),\quad \xi\stackrel{i.i.d.}{\sim}N(0,SH).
\label{eq:Russo-Q-estimator}
\end{equation}
This perturbation term is particularly important for finite-horizon problems, in which the variance (in estimating the value function) propagates when we do backward induction; see more discussion in \cite{osband2016generalization}. 
The RLSVI algorithm resembles the Thompson Sampling (TS) algorithm for bandit problems in the sense that getting the optimal policy from the sampled MDP is similar to picking the optimal arm with respect to the sampled parameters from the posterior distribution in bandit problems via TS. 


\section{Proof of the Main Results}
\label{app:proof-main}

This section provides the proof for Theorem \ref{thm: data-pooling estimator} and Theorem \ref{thm: main}. In Section \ref{subsec: connectiong regret to CI},  we introduce the formal regret analysis framework (debriefed in Section \ref{subsec:regret-framework} of the main paper) and establish a regret upper bound for the (no-pooing) perturbed LSVI algorithm in terms of the Hoeffding confidence radii. In Section~\ref{subsec: CI for data-pooling estimators}, we analyze the confidence radii for the data-pooling estimator as a function of the weight $\lambda$ and optimize $\lambda$ by minimizing the confidence radii, thus completing the proof of Theorem \ref{thm: data-pooling estimator}. Based on these results, we are able to connect the regret bound of Algorithm \ref{Alg:DP-LSVI} with the data-pooling weight $\lambda$, and hence reduce the regret bound via the properly chosen $\lambda$. In Section~\ref{subsec: regret bound for data-pooling RL}, we formally state our main results on the regret bound of Algorithm \ref{Alg:DP-LSVI} and complete the proof of Theorem \ref{thm: main}. In Section \ref{subsec: extension-to-multi-group}, we extend the main results to a more general setting, in which the historical data are collected from multiple sources.

\noindent\textbf{Summary of Notations.}
Before the proof, we summarize the notations here.

\begin{itemize}
	\item $\mathcal{M}=(r, P)$: true model of the target class. 
	\item $\tilde{\mathcal{M}}_t = (\tilde{r}_t, \bar{P}_t)$: model estimated with the perturbed rewards in iteration $t$.
		\item $\pi^*$: the optimal policy under the true model $\mathcal{M}$.
	\item $\pi_t$: the policy learned in iteration $t$.
	\item $V(\mathcal{M},\pi)$: value function of model $\mathcal{M}$ under policy $\pi$. 
	\item $\tilde{V}_{t,h}$: value-to-go from horizon $h$ of model $\tilde{\mathcal{M}}_t $ under policy $\pi_t$.

	
	\item  $\{\varepsilon_V^H(n),\varepsilon_P^H(n),\varepsilon_R^H(n)\}$: Hoeffding confidence radii. 
	
	\item $\{\varepsilon_V^{DP}(n,N),\varepsilon_P^{DP}(n,N),\varepsilon_R^{DP}(n,N)\}$: data-pooling confidence radii. 
\end{itemize}

\subsection{Connecting Regret Bound with the Confidence Radii}
\label{subsec: connectiong regret to CI}

In this part, we explain the framework of regret analysis for perturbed LSVI algorithms and specify the regret bound given in~\eqref{eq: regert informal} in the main paper, connecting the bound with the confidence radii.

Recall that $\bar{P}_t(\cdot; h,s,a)$ and $\bar{r}_t(h,s,a)$ are estimated values of the transition probabilities and mean rewards in iteration $t$ of the perturbed LSVI algorithm without data pooling, and $n_t(h,s,a)$ is the number of observed target data samples at $(h,s,a)$ by iteration $t$. The Hoeffding confidence radii $\{\varepsilon_{R}^H(n),\varepsilon_{P}(n)^H, \varepsilon_{V}^H(n)\}$ satisfying  \eqref{eq: CI fixed n} with fixed sample size $n$ can be applied to random and path-dependent sample sizes $n_t(h,s,a)$ as specified in Lemma \ref{lmm: CI definition} below.

\begin{lemma}[Hoeffding Confidence Radii]\label{lmm: CI definition} 
	Let $\{\varepsilon_V^H(n),\varepsilon_P^H(n),\varepsilon_R^H(n): 1\leq n\leq T\}$ be the Hoeffding confidence radii with confidence parameter $\delta$ that satisfy \eqref{eq: CI fixed n} for fixed sample size, i.e.
	\begin{align}\label{eq: RL Hoeffding confidence radius}
		\varepsilon_R^H(n) &:=\sqrt{\log(2HSAT/\delta)/2n}~,\notag\\
		\varepsilon_P^H(n) &:= \sqrt{2(S\log(2)+\log(2HSAT/\delta))/n}~,\\
		\varepsilon_V^H(n) &:=H\sqrt{\log(2HSAT/\delta)/2n}~,\notag
	\end{align}
	for $n\geq 1$. For $n=0$, i.e., no observation, the trivial bounds are $\varepsilon_R^H(0)=1, \varepsilon_P^H(0)=2$, and $\varepsilon_R^V(0)=H$. Then, the following inequalities hold:
	\begin{enumerate}[label=(\roman*)]
		\item $\PP\Big( |\bar{r}_t(h,s,a)-r(h,s,a)+\langle\bar{P}_t(\cdot;h,s,a)-P(\cdot;h,s,a),V^*_{h+1}\rangle|\leq\varepsilon_V^H \big( n_t(h,s,a) \big), \forall~1\leq t\leq T, 1\leq h\leq H, s\in \mathcal{S}, a\in\mathcal{A}  \Big) \geq 1-\delta$. 
		\item $\PP\Big( \|\bar{P}_t(\cdot; h,S_h^t,A_h^t )-P(\cdot;h,S_h^t,A_h^t)\|_1\leq\varepsilon_P^H\big( n_t(h,S_h^t,A_h^t)\big), \forall~ 1\leq t\leq T, 1\leq h\leq H \Big) \geq 1-\delta$.
		\item $\PP\Big( |\bar{r}_t(h,S_h^t,A_h^t)-r(h,S_h^t,A_h^t)|\leq\varepsilon_R^H\big( n_t(h,S_h^t,A_h^t) \big),\forall~ 1\leq t\leq T, 1\leq h\leq H \Big) \geq 1-\delta$.
	\end{enumerate}
\end{lemma}
The proof of Lemma \ref{lmm: CI definition} follows the standard recipe, i.e., applying concentration inequalities to  a so-called ``stack-of-rewards" probability model for tabular reinforcement learning. We omit the proof here and refer the readers to Appendix~A of~\cite{russo2019worst} and the references therein for more discussions.

\begin{remark}
	Inequalities (conditions) (ii) and (iii) are commonly used in the analysis of bandit and tabular RL algorithms to bound the ``on-policy" estimation errors. Inequality (condition) (i) will be used later to bound exploration cost in the proof of Proposition \ref{lmm: RL exploartion to estimation error}.
\end{remark}

A key component in the design of perturbed LSVI algorithm is the perturbation term $w_t(h,s,a)\equiv \varepsilon_{V}^H(n_t(h,s,a))\cdot\xi_t(h,s,a)$. 
On the one hand, to ensure enough perturbation, $w_t(h,s,a)$ should not be too small; on the other hand, to reduce exploration cost, $w_t(h,s,a)$ should not be too large. We impose the following assumption on the perturbation terms. 
\begin{assumption}[Perturbation Terms]\label{assmpt: perturbed terms}
	There exist  constants $p_0\in (0,1] ,\delta>0$ and $\bar{W}>0$ such that, 
	\begin{enumerate}
		\item[(iv)] For any sequence of non-negative real numbers $\varepsilon_{1},...,\varepsilon_{HS}$, $$\PP\left(\sum_{i=1}^{HS} (\xi_{i}-1)\cdot \varepsilon_{i}>0\right)\geq p_0>0,\text{ with } \xi_i\stackrel{i.i.d.}{\sim} \xi.$$
		\item[(v)] $\PP(\max_{h,s,a,t}|\xi_t(h,s,a)|\leq \bar{W})\geq 1-\delta.$
	\end{enumerate}
\end{assumption}
Following Lemma 6 of \cite{russo2019worst}, condition (iv) ensures sufficient exploration in perturbed LSVI algorithms. In particular, one can verify that for \texttt{UCB-based} RL, condition (iv) holds for any constant $\xi\geq 1$ and $p_0=1$; for the RLSVI algorithm (\texttt{TS-based} RL), it holds when $\xi$ is a normal r.v. with mean 0 and variance $HS$ and $p_0=\Phi(-1)$, where $\Phi(\cdot)$ is the CDF of a standard normal distribution. The constant $\bar{W}$ in condition (v) will be used to bound the exploration cost in the regret analysis. 

Now, we are ready to prove a regret bound for perturbed LSVI algorithms in terms of the confidence radii and the perturbation terms. Although the main idea of the analysis is based on \cite{russo2019worst}, we prove our version in this paper to \emph{unify} all the results within the perturbed LSVI framework, covering both the \texttt{TS-based} version (as in \cite{russo2019worst}) and the \texttt{UCB-based} version. This unified result also provides insights into how one should design the data-pooling estimator.

\begin{proposition} 
	\label{lmm: RL exploartion to estimation error} 
	Suppose a perturbed LSVI algorithm with $T$ iterations satisfies conditions (i) to (v) for constants $\delta>0$, $p_0\in (0,1)$, $\bar{W}>0$ and the Hoeffding confidence radii.
	Then,   
	\begin{equation}
		\label{eq: PLSVI general regret bound}
		Regret(T) \leq (1+2p_0^{-1})HSA\sum_{n=1}^{\lceil  T/SA \rceil}\left(\varepsilon^H_R(n)+H(\bar{W}+1)\varepsilon^H_P(n)+\bar{W}\varepsilon_{V}^H(n)\right)+4HT\delta.
	\end{equation}
\end{proposition}

\begin{proof}{Proof of Proposition \ref{lmm: RL exploartion to estimation error}}
	
	Recall that $\text{Regret}(T)=V(\mathcal{M},\pi^*)T-\EE[V(\mathcal{M},\pi_t)]$, where $\pi^*$ is the (unknown) true optimal policy. For convenience, we define $\rho(t)$ as the regret accumulated in each iteration $1\leq t\leq T$, which can be calculated and bounded as 
	\begin{align}
		\rho(t) &= V(\mathcal{M},\pi^*)-\mathbb{E}[V(\mathcal{M}, \pi_t)] 
		\nonumber\\ 
		& = \mathbb{E}\Big[\left(V(\mathcal{M},\pi^*)-\mathbb{E}[V(\tilde{\mathcal{M}}_t, \pi_t)|\mathcal{F}_{t-1}]\right) + \left(\mathbb{E}[V(\tilde{\mathcal{M}}_t, \pi_t)-V(\mathcal{M}, \pi_t)|\mathcal{F}_{t-1}]\right)\Big]
		\nonumber\\
		& \leq \mathbb{E}\Big[\underbrace{\left(V(\mathcal{M},\pi^*)-\mathbb{E}[V(\tilde{\mathcal{M}}_t, \pi^*)|\mathcal{F}_{t-1}]\right)}_{I_1:\text{ exploration}} + \underbrace{\left(\mathbb{E}[V(\tilde{\mathcal{M}}_t, \pi_t)-V(\mathcal{M}, \pi_t)|\mathcal{F}_{t-1}]\right)}_{I_2:\text{ estimation error} + \text{perturbation}}\Big].
		\label{eq:PLSVI-regret-decomp}
	\end{align}
	The inequality above follows from the fact that $\pi_t$ is chosen from the greedy-optimal under the estimated model $\tilde{\mathcal{M}}_t$, so the gap increases when changing to policy $\pi^*$. 
	The regret bound is separated into two terms for the ease of analysis: term $I_1$ captures the difference (when using the same policy $\pi^*$) between the value-to-go in the true model ${\mathcal{M}}$ and the estimated model $\tilde{\mathcal{M}}_t$; term $I_2$ captures a similar difference, between ${\mathcal{M}}$ and $\tilde{\mathcal{M}}_t$ but under the actual policy $\pi_t$ that is applied in each iteration $t$. 

	The $I_1$ term plays a key role in the proof and captures how the perturbation term $\{w_t(h,s,a)\}$ leads to better exploration. To bound $I_1$, we follow the analysis in \cite{russo2019worst} and relate the bound of $I_1$ to that of $I_2$. Below we briefly describe the main ideas and refer the readers to \cite{russo2019worst} for more details. First, conditional on the good event that $$\left\{|\bar{r}_t(h,s,a)-r(h,s,a)+\langle\bar{P}_t(\cdot;h,s,a)-P(\cdot;h,s,a),V^*_{h+1}\rangle|\leq \varepsilon_V^H(n_t(h,s,a)), \forall 1\leq t\leq T, 1\leq h\leq H\right\},
	$$
one can derive that, conditional on $\mathcal{F}_{t-1}$, i.e., the information revealed right before iteration $t$,
	\begin{align*}
		V(\tilde{\mathcal{M}}_t,\pi^*)-V(\mathcal{M},\pi^*)&\geq \EE_{\pi^*,\bar{\mathcal{M}}_t}\left[\sum_{h=1}^H(w_t(h,S_h,A_h)-\varepsilon^H_V(n_t(h,S_h,A_h)))\right]\\
		&= \sum_{h=1}^H\sum_{s\in\mathcal{S}}\Big( \xi_t(h,s,\pi^*(s))-1)\cdot \varepsilon_{V}^H(n_t(h,s,\pi^*(s)) \Big)
		\PP(S_h=s).
	\end{align*}
Note that $\xi_t(h,s,\pi^*(s))\stackrel{i.i.d.}{\sim} \xi$, by condition (iv), we can conclude that  $\PP(V(\tilde{\mathcal{M}}_t,\pi^*)-V(\mathcal{M},\pi^*)>0)\geq p_0>0$. Given condition (i) and applying Lemma~6 in~\cite{russo2019worst}, one can conclude that
	\begin{equation}
		\rho(t)\leq (1+2p_0^{-1})\EE\left[\EE[V(\tilde{\mathcal{M}}_t,\pi_t)-V(\mathcal{M},\pi_t)|\mathcal{F}_{t-1}]\right]+HT\delta.
		\label{eq-app:russo-rho}
	\end{equation}

The remaining task is to bound $I_2$. We apply a representation of $V(\tilde{\mathcal{M}}_t,\pi_t)-V(\mathcal{M},\pi_t)$ that is commonly used in RL literature. Specifically, conditional on $\mathcal{F}_{t-1}$,
	\begin{align}\label{eq: value difference decomposition}
		&V(\tilde{\mathcal{M}}_t,\pi_t)-V(\mathcal{M},\pi_t)\notag\\
		=~&\EE_{\pi_t,\mathcal{M}}\left[\sum_{h=1}^H (\bar{r}_t(h,S_h,A_h)-r(h,S_h,A_h))+w_t(h, S_h,A_h)+\left<\bar{P}_t(h,S_h,A_h)-P(h,S_h,A_h), \tilde{V}_{t,h+1}\right> \right].
	\end{align}
	Here the expectation is taken with respect to the distribution of a trajectory $\{(S_h, A_h)\}_{h=1}^H$ of the target MDP $\mathcal{M}$ under policy $\pi_t$.
Define a good event 
\begin{align*}
\mathcal{E}=&\left\{\|\bar{P}_t(h, S_h^t,A_h^t)-P(h,S_h^t,A_h^t)\|_1\leq \varepsilon^H_P\big( n_t(h,S_h^t,A_h^t) \big),\right.
\\
&~~ |\bar{r}_t(h, S_h^t,A_h^t)-r(h, S_h^t,A_h^t)|\leq \varepsilon_R^H\big( n_t(h,S_h^t,A_h^t) \big),  \forall 1\leq t\leq T, 1\leq h\leq H,
\\
&~~\left. \max_{h,s,a,t}|\xi_t(h,s,a)|\leq \bar{W}\right\}.
\end{align*}
Following conditions (ii), (iii) and (v), $\PP(\mathcal{E})\geq 1-3\delta$. Conditional on $\mathcal{E}$, we have
	\begin{align*}
		V(\tilde{\mathcal{M}}_t,\pi_t)-V(\mathcal{M},\pi_t)
		\leq \EE_{\pi_t,\mathcal{M}}\left[\sum_{h=1}^H \varepsilon^H_R(n_t(h,S_h^t,A_h^t))+\bar{W}\varepsilon_{V}^H(n_t(h,S_h^t,A_h^t))+\varepsilon_P^H(n_t(h,S_h^t,A_h^t))\|\tilde{V}_{t,h+1}\|_\infty \right].
	\end{align*}
	As the rewards are bounded in $[0,1]$, conditional on $\mathcal{E}$, the value function 
	$$\|\tilde{V}_{t,h}\|_\infty\leq H+H\max_{h,s,a}|w_t(h,s,a)|\leq H+\bar{W}H^2 \text{ for all }1\leq h\leq H,$$
	as $|w_t(h,s,a)|\leq \bar{W}\varepsilon_{V}^H(n_t(h,s,a))\leq \bar{W}H$.

Then, we have
	\begin{align*}
		&V(\tilde{\mathcal{M}},\pi_t)-V(\mathcal{M},\pi_t)\\
		\leq~& \EE_{\pi_t,\mathcal{M}}\left[\sum_{h=1}^H \varepsilon^H_{R}(n_t(h,S_h^t,A_h^t))+\bar{W}\varepsilon_{V}^H(n_t(h,S_h^t,A_h^t))
		+(H^2\bar{W}+H)\varepsilon_P^H(n_t(h,S_h^t,A_h^t))\right],
	\end{align*}
	conditional on $\mathcal{E}$. Then, by the pigeonhole principle, we have that, conditional on the good event,
	\begin{align*}
		\sum_{t=1}^T\EE\left[V(\tilde{\mathcal{M}},\pi_t)-V(\mathcal{M},\pi_t)\right]\leq HSA\sum_{n=1}^{\lceil  T/SA \rceil}\left(\varepsilon_R^H(n)+\bar{W}\varepsilon_{V}^H(n)+(H^2\bar{W}+H)\varepsilon_P^H(n)\right),
	\end{align*}
	and via~\eqref{eq-app:russo-rho}, we conclude that
	$$Regret(T)\leq (1+2p_0^{-1})HSA\sum_{n=1}^{\lceil  T/SA \rceil}\left(\varepsilon_R^H(n)+\bar{W}\varepsilon_{V}^H(n)+(H^2\bar{W}+H)\varepsilon_P^H(n)\right)+4HT\delta.$$

Recall that for ease of exposition, we write the regret bound \eqref{eq: regert informal} in the main paper into two parts. Based on the above regret bound, the expressions of the two parts, in terms of the confidence radii, are as follows. 
\begin{align*}
f_1(\varepsilon^H_R(n),\varepsilon^H_P(n)) &= (1+2p_0^{-1})HSA(\varepsilon^H_R(n)+H\varepsilon^H_P(n))+2HSA\delta,
\\
f_2(\varepsilon^H_V(n),\varepsilon_P^H(n),\bar{W}) &= (1+2p_0^{-1})HSA(H^2\bar{W}\varepsilon_P^H(n) + \bar{W}\varepsilon^H_V(n))+2HSA\delta.
\end{align*}
\hfill $\Box{}$ 
\end{proof}

\begin{remark}[Recovering Results in \cite{russo2019worst}] 	
In the RLSVI algorithm, 
the random variable $\xi$ is set to be normal with mean 0 and variance $SH$. If we pick $\delta = (HT)^{-1}$, the Hoeffding confidence radii for value functions becomes $\varepsilon^H_V(n)= H\sqrt{\log(2H^2T^2SA)/2n}=\tilde{O}(H\sqrt{1/n})$. Then, Assumption \ref{assmpt: perturbed terms} holds with $p_0 = \Phi(-1)>0$ and $\bar{W}=\tilde{O}(\sqrt{SH})$, where $\Phi(\cdot)$ is the CDF of a standard normal distribution. Plugging the other two Hoeffding confidence radii $\varepsilon_{R}(n)=\sqrt{\log(2H^2SAT^2)/2n}=\tilde{O}(\sqrt{1/n})$ and $\varepsilon_{P}(n)=\sqrt{2(S\log(2)+\log(2H^2SAT^2))/n}=\tilde{O}(H\sqrt{S/n})$ into \eqref{eq: PLSVI general regret bound}, we recover the regret bound in \cite{russo2019worst}.\footnote{The actual regret bound derived in the paper should be of $\tilde{O}(H^{7/2}S^{3/2}\sqrt{AT})$. The proof for Theorem 1 there missed a $H^{1/2}$ term. }
\end{remark}


\subsection{Confidence Radii of Data-pooling Estimators}
\label{subsec: CI for data-pooling estimators}
In this part, we first derive the confidence radii for the data-pooling estimators as a function of the weight $\lambda$ and choose the optimal data-pooling weight  $\lambda^{DP}$ by minimizing the confidence radii. Then, we compare the minimized confidence radii (referred to as the \emph{data-pooling confidence radii}) with the standard Hoeffding confidence radii, showing that using data-pooling estimators can help reduce confidence radii,  and thus complete the proof of Theorem \ref{thm: data-pooling estimator}.

In the following analysis, we consider a fixed $(h,s,a)$ triple. We assume that there are $n$ random samples from the target data for the reward $R_k\sim R(h,s,a)$ and transition $S_k\sim P(\cdot|h, s,a)$, and there are $N$ random samples $(R^0_1, S^0_1),...,(R^0_N,S^0_N)$ from the historical data. We first analyze the data-pooling confidence radii for mean rewards, through which we derive the data-pooling weight $\lambda_{n,N}^{DP}$ as in Theorem \ref{thm: data-pooling estimator}. Then, we show that using $\lambda_{n,N}^{DP}$ can also reduce the confidence radii for transition probabilities and value functions.

\noindent\textbf{Confidence radii for mean rewards:} 	For any $\lambda\in[0,1]$, we define the mapping $h^n_{\lambda}: [0,1]^{n+N_i}\to \mathbb{R}$ as 
\begin{align*}
	h^n_{\lambda}(\mathbf{x})\triangleq~&h^n_{\lambda}(x_1,x_2,...,x_{n}, x^0_1,...,x^0_{N_i}) =\lambda\sum_{j=1}^{n}\frac{x_j}{n} + (1-\lambda)\sum_{j=1}^{N}\frac{x^0_j}{N}.
\end{align*}
The data-pooling estimate corresponding to weight $\lambda$ can be written as 
$$\hat{r}_n^{\lambda}=h^n_{\lambda}(R_1,...,R_n, R_1^0,...,R_{N}^0).$$ 
For any given $\mathbf{x}\in [0,1]^{n+N}$, we have 
\begin{align*}
	&|h^n_{\lambda}(x_1,...,x_j,...x_{n}, x^0_1,...,x^0_{N_i}) -h^n_{\lambda}(x_1,...,x'_j,...x_{n}, x^0_1,...,x^0_{N_i})|\leq \lambda/n,\\
	&|h^n_{\lambda}(x_1,...,x_{n}, x^0_1,...,x^0_j,..,.x^0_{N_i}) -h^n_{\lambda}(x_1,...,x_{n}, x^0_1,...,x^{0'}_j,...,x^0_{N_i})|\leq (1-\lambda)/N.
\end{align*}
By the Bounded Difference Inequality, 
$$\mathbb{P}(|h^n_{\lambda}-\EE(h^n_{\lambda})|>\varepsilon_{1}(\lambda))\leq 2\exp\left(\frac{-2(\varepsilon_{1}(\lambda))^2}{\left(\frac{\lambda^2}{n}+\frac{(1-\lambda)^2}{N}\right)}\right)=\frac{\delta}{HSAT} , $$
where the equality is achieved by choosing
$$\varepsilon_{1}(\lambda)=  \sqrt{\log\left(2HSAT/\delta\right)\left(\frac{\lambda^2}{2n}+\frac{(1-\lambda)^2}{2N}\right)}.$$
On the other hand, by the assumption that $|r_0(h,s,a)-r(h,s,a)|<\Delta$, we have
$$|\mathbb{E}(h^n_{\lambda}) - r(h,s,a)|\leq \Delta(1-\lambda)\triangleq\varepsilon_{2}(\lambda) . $$
Since $\hat{r}_n^\lambda - r(h,s,a) = \big(h_{\lambda} - \EE(h_{\lambda})\big) + \big(\EE(h_{\lambda}) - r(h,s,a) \big)$, we then have 
\begin{align*}
	&\mathbb{P}\Big(|\hat{r}_n^\lambda-r(h,s,a)|>\min\left(1,\varepsilon_{1}(\lambda)+\varepsilon_{2}(\lambda)\right)\Big)\leq \frac{\delta}{HSAT}.
\end{align*}
Now, define 
$$\varepsilon(\lambda) \triangleq \varepsilon_{1}(\lambda) +\varepsilon_{2}(\lambda)= \sqrt{\log\left(2HSAT/\delta\right)\left(\frac{\lambda^2}{2n}+\frac{(1-\lambda)^2}{2N}\right)} + \Delta(1-\lambda) ,$$
which is a convex function with derivative
$$ \varepsilon' (\lambda) = \sqrt{\log(2HSAT/\delta)\cdot\frac{\frac{\lambda}{n} +\frac{\lambda - 1}{N}}{\sqrt{\frac{2\lambda^2}{n} +\frac{2(1-\lambda)^2}{N}}}} - \Delta.$$

\noindent\textbf{Data-pooling weight choice.}
We choose the data-pooling weight $\lambda^{DP}_{n,N}$ to minimize the confidence radius $\varepsilon(\lambda)$, i.e. $\lambda^{DP}_{n,N}\triangleq\arg\min_{\lambda\in[0,1]}\varepsilon(\lambda)$. By the first-order condition and convexity of $\varepsilon(\lambda)$,  
\begin{equation*}\label{eq: lambda}
	\lambda^{DP}_{n,N}=
	\begin{cases}
		1& \text{if }n \geq \log(2HSAT/\delta)/2\Delta^2,\\
		\frac{n+Nn\Delta/\sqrt{(N+n)\log(2HSAT/\delta)/2-\Delta^2Nn}}{N+n}& \text{if }n< \log(2HSAT/\delta)/2\Delta^2.
	\end{cases}
\end{equation*}
In the case where $n< \log(2HSAT/\delta)/2\Delta^2 $, we can check that the derivative $\varepsilon'(1)>0$ and therefore, $\varepsilon_R^{DP}(n,N)\triangleq \varepsilon(\lambda^{DP}_{n,N})<\varepsilon(1)=\varepsilon^H_R(n)$. 
So, we can conclude  
$$\varepsilon^{DP}_R(n,N)\leq \varepsilon_R^H(n), \text{ for  }N\geq 1, \text{ and the inequality is restrict when }n<\log(2HSAT/\delta)/2\Delta^2.$$


\noindent\textbf{Transition Probabilities:} To get the confidence radius of the transition probabilities, we consider any set $A\subseteq\mathcal{S}$. We apply the $L_1$ deviation bound in
Weissman et al. (2003). 
The $\lambda$-pooling estimator for the transition probability to set $P(A|h,s,a)$ is given as 
$$\hat{P}_{n}(A;\lambda)= \lambda\sum_{j=1}^{n}\frac{1(S_j\in A)}{n} + (1-\lambda)\sum_{j=1}^{N}\frac{1(S^0_j\in A)}{N}.$$
Given a set $A$, define a function $g_{n,\lambda}:\mathbb{R}^{n+N}\to\mathbb{R}^+$ as
\begin{align*}
	g_{n,\lambda}(x_1,x_2,...,x_{n}, x^0_1,...,x^0_{N}) =\lambda\sum_{j=1}^{n}\frac{1(x_j\in A)}{n}
	+ (1-\lambda)\sum_{j=1}^{N}\frac{1(x^0_j\in A)}{N}.  
\end{align*}
By the Bounded Difference Inequality, we have 
\begin{align*}
	\mathbb{P}\big( |\hat{P}_n(A;\lambda)-\mathbb{E}[\hat{P}_n(A|h,s,a)]|>\varepsilon_1(\lambda) \big) 
	\leq 2\exp\left(\frac{-2\varepsilon_1(\lambda)^2}{\frac{\lambda^2}{n}+\frac{(1-\lambda)^2}{N}}\right), 
\end{align*}
where the value of $\varepsilon_1(\lambda)$ will be specified in a moment. 
By Assumption \ref{assmpt: difference bound} and the definition of $L_1$ norm, we have
$$\big| P(A|h,s,a)-P_0(A|h,s,a) \big| \leq \big\|P(h,s,a)-P_0(h,s,a)\big\|_1\leq \Delta.$$ Therefore, we have that 
\begin{align*}
	\mathbb{P}\big( |\hat{P}_n(A|h,s,a)-P(A|h,s,a)|>\varepsilon_1(\lambda)+(1-\lambda)\Delta \big) 
	\leq 2\exp\left(\frac{-2\varepsilon_1(\lambda)^2}{\frac{\lambda^2}{n}+\frac{(1-\lambda)^2}{N}}\right), 
\end{align*}
Set $\varepsilon(\lambda) = \varepsilon_1(\lambda)+(1-\lambda)\Delta$. We then get  
\begin{align*}
	&\mathbb{P}\left(\big\|\hat{P}_n(\cdot|h,s,a)-P(\cdot|h,s,a)\big\|_1>\varepsilon(\lambda)\right)
	=\mathbb{P}\left( \max_{A\subseteq\mathcal{S}}|\hat{P}_n(A|h,s,a)-P(A|h,s,a)|>\varepsilon(\lambda)/2\right)\\
	\leq~& \sum_{A\subseteq \mathcal{S}}\mathbb{P}\big( |\hat{P}_n(A|h,s,a)-P(A|h,s,a)|>\varepsilon(\lambda)/2 \big) 
	\\
	\leq~& (2^S-2)\cdot 2\exp\left(\frac{-\varepsilon_1(\lambda)^2/2}{\frac{\lambda^2}{n}+\frac{(1-\lambda)^2}{N}}\right)\leq \frac{\delta}{HSAT},
\end{align*}
where the last inequality is obtained by setting $\varepsilon_1(\lambda) 
= \sqrt{2\left(S\log(2)+\log(2HSAT/\delta)\right)\left(\frac{\lambda^2}{n}+\frac{(1-\lambda)^2}{N}\right)}$. As a result, the confidence radius with weight $\lambda$ equals to
\begin{align*}
	\varepsilon(\lambda)=\sqrt{2\left(S\log(2)+\log(2HSAT/\delta)\right)\left(\frac{\lambda^2}{n}+\frac{(1-\lambda)^2}{N}\right)}+(1-\lambda)\Delta.
\end{align*}
Plugging in $\lambda = \lambda^{DP}_{n,N}$, we have, as $(1-\lambda)\Delta\geq 0$, 
\begin{align*}
	\varepsilon_{P}^{DP}(n,N)&\equiv \varepsilon(\lambda^{DP}_{n,N})\leq \sqrt{\frac{4(S\log(2)+\log(2HSAT/\delta))}{\log(2HSAT/\delta)}}\varepsilon_{R}^{DP}(n,N)\\
	&\leq \sqrt{\frac{4(S\log(2)+\log(2HSAT/\delta))}{\log(2HSAT/\delta)}}\varepsilon_{R}^{H}(n) = \varepsilon_{P}^H(n),
\end{align*}
and the last inequality is strict when $n<\log(2HSAT/\delta)/2\Delta^2$.

\noindent\textbf{Value Functions:} 
For each $n$ and $\lambda\in[0,1]$, we define $h_{n,\lambda}: \mathbb{R}^{2n+2N}\to \mathbb{R}$ as
\begin{align*}
	h_{n,\lambda}(\mathbf{x};\mathbf{y})
	&\triangleq~h_{n}(x_1,x_2,...,x_{n}, x^0_1,...,x^0_{N};y_1,y_2,...,y_{n}, y^0_1,...,y^0_{N}) \\
	&=\lambda\sum_{j=1}^{n}\frac{x_j}{n} + (1-\lambda)\sum_{j=1}^{N}\frac{x^0_j}{N} +\lambda\sum_{j=1}^{n}\frac{V^*_h(y_j)}{n} + (1-\lambda)\sum_{j=1}^{N}\frac{V^*_h(y^0_j)}{N}.
\end{align*}
As all rewards are bounded in $[0,1]$,  $0\leq V^*_{h+1}(s)\leq H-1$. Then,
by the Bounded Difference Inequality, we have,
\begin{align*}
	\mathbb{P}(|\hat{h}_{n,\lambda}-\mathbb{E}(\hat{h}_{n,\lambda})|>\varepsilon_1(\lambda))
	\leq~ 2\exp\left(\frac{-2\varepsilon_1(\lambda)^2}{(1+(H-1)^2)\left(\frac{\lambda^2}{n}+\frac{(1-\lambda)^2}{N}\right)}\right)\leq \frac{\delta}{HSAT},
\end{align*}
with
$$\varepsilon_1(\lambda)= H\sqrt{\log\left(2HSAT/\delta\right)\left(\frac{\lambda^2}{2n}+\frac{(1-\lambda)^2}{2N}\right)}.$$
As
\begin{align*}
	\left| \mathbb{E}(\hat{h}_{n,\lambda}) - r(h,s,a) -\langle P(\cdot|h,s,a), V_{h+1}^*(\cdot)\rangle \right| 
	\leq H(1-\lambda)\Delta,
\end{align*}
we have 
$$\mathbb{P}\left(|\hat{r}^{DP}_n - r(h, s,a)+\langle\hat{P}_n^{DP}-P(h,s,a),V_h^{*}\rangle|>\varepsilon(\lambda)\right)\leq \frac{\delta}{HSAT},$$
where 
$$\varepsilon(\lambda)=H\sqrt{\log\left(2HSAT/\delta\right)\left(\frac{\lambda^2}{2n}+\frac{(1-\lambda)^2}{2N}\right)}+H(1-\lambda)\Delta.$$
Therefore, by setting $\lambda = \lambda^{DP}_{n,N}$, we have
$$\varepsilon_{V}^{DP}(n,N) \equiv\varepsilon(\lambda^{DP}_{n,H})= H\varepsilon_{R}^{DP}(n,N)\leq H\varepsilon_{R}^{H}(n) = \varepsilon_{V}^H(n), $$
and the last inequality is restrict when $n<\log(2HSAT/\delta)/2\Delta^2$. This also concludes our proof for Theorem~\ref{thm: data-pooling estimator}.

\subsection{Regret Bound for Data-pooling Perturbed LSVI}
\label{subsec: regret bound for data-pooling RL}

Given Proposition \ref{lmm: RL exploartion to estimation error} and Theorem \ref{thm: data-pooling estimator}, we are ready to give a regret bound for Algorithm \ref{Alg:DP-LSVI} and the formal statement of Theorem \ref{thm: main} in the main paper. 

\begin{theorem}[Formal Statement of Theorem \ref{thm: main}]\label{thm: main formal}
	Suppose a perturbed LSVI algorithm with $T$ iterations satisfies conditions (i) to (v) as in Proposition \ref{lmm: RL exploartion to estimation error} for constants $\delta>0$, $p_0\in(0,1)$ and $\bar{W}>0$ and  Hoeffding confidence radii $\{\varepsilon_V^H(n),\varepsilon_P^H(n),\varepsilon_R^H(n)\}$. 
	Then, under Assumption \ref{assmpt: difference bound}, the corresponding data-pooling perturbed LSVI as given in Algorithm~\ref{Alg:DP-LSVI} also satisfies  conditions (i) to (v)  for the same set of constants  $\delta>0$, $p_0\in(0,1)$ and $\bar{W}>0$ 
	and the data-pooling confidence radii $\{\varepsilon_V^{DP}(n,N),\varepsilon_P^{DP}(n,N),\varepsilon_R^{DP}(n,N)\}$. Here the constant $N\equiv \min_{h,s,a}N(h,s,a)$ and $N(h,s,a)$ is the sample size from historical data at $(h,s,a)$. The total regret of the data-pooling perturbed LSVI
	\begin{equation}
		\label{eq: data-pooling PLSVI regret bound}
		Regret(T)\leq (1+2p_0^{-1}) HSA\sum_{n=0}^{\lceil  T/SA \rceil}\left(\varepsilon^{DP}_R(n,N)+H(\bar{W}+1)\varepsilon^{DP}_P(n,N)+\bar{W}\varepsilon_{V}^{DP}(n,N)\right) + 4HT\delta.
	\end{equation}
	As a consequence, if $N(h,s,a)\geq 1$ for all $(h,s,a)$ and $\Delta< \sqrt{\log(2HSAT/\delta)/2}$, the regret bound of the data-pooling perturbed LSVI is strictly smaller than that of the original perturbed LSVI algorithm with the Hoeffding confidence radii. 
\end{theorem}

\begin{proof}{Proof of Theorem \ref{thm: main formal}. }
	Conditions (i) to (iii) hold for the data-pooling confidence radii following Theorem \ref{thm: data-pooling estimator} and Lemma \ref{lmm: CI definition}. 
	Conditions (iv) and (v) hold as they only depends on the distribution of  $\xi$.
	Therefore, we can apply the same argument in the proof Proposition \ref{lmm: RL exploartion to estimation error}, replacing the Hoeffding confidence radii with their data-pooling version, and obtain \eqref{eq: data-pooling PLSVI regret bound}.
	
	Finally, following Theorem \ref{thm: data-pooling estimator}, for all $1\leq n\leq \log(2HSAT/\delta)/2\Delta^2$, the data-pooling confidence radius is strictly smaller than the corresponding Hoeffding confidence radius if $N\geq 1$, i.e.
	$\varepsilon_L^{DP}(n,N)<\varepsilon_L^H(n)$, for $L=R, P, V$. Thus, we can conclude that the regret bound of the data-pooling algorithm in \eqref{eq: data-pooling PLSVI regret bound} is strictly smaller than that of original algorithm in  \eqref{eq: PLSVI general regret bound}
	\hfill$\boxed{}$   
\end{proof}

\begin{remark}
\label{rmk:hist-sample-size}
	In Algorithm~\ref{Alg:DP-LSVI}, the data-pooling weight $\lambda_t(h,s,a)$ is computed according to $N(h,s,a)$ instead of $N$. The corresponding confidence radii should be $\varepsilon_\cdot^{DP}(n,N(h,s,a))$ (with $\cdot = R,P,V$), which is  $<\varepsilon_\cdot^{DP}(n,\cdot)$ as long as $N(h,s,a)\geq 1$. In other words, the data-pooling perturbed LSVI could reduce the regret bound when $N(h,s,a)>0$ for some of the $(h,s,a)$ triples. In the theoretic analysis of regret bound (Theorem \ref{thm: main formal}), we use a loose bound 
	$\varepsilon_\cdot^{DP}(n,N(h,s,a))\leq \varepsilon_\cdot^{DP}(n,N)$ with $N$ mainly to obtain a simple expression of the regret bound. This also allows easy comparison to the original (no pooling) algorithm. 
\end{remark}

\subsection{Extension to Multiple Historical Groups}
\label{subsec: extension-to-multi-group}

To formally state our design of the data pooling estimator with multiple historical data sets, we first introduce the notation. Suppose there are $K$ historical data sets indexed by $k=1,2,...,K$. Denote by $\mathcal{M}_{0k} = (H,\mathcal{S},\mathcal{A},R_{0k}, P_{0k})$ the MDP model that generates historical data set $k$. We assume that the decision-maker has some prior knowledge of the similarity between the target MDP $\mathcal{M}$ and each of the MDP $\mathcal{M}_{0k}$, as formalized in the following assumption.  

\begin{assumption}\label{assmpt: difference bound multiple}
	For each $k=1,2,...,K$, there exists a known constant $\Delta_k>0$ such that $\forall~(s,a)\in\mathcal{S}\times\mathcal{A}, 1\leq h\leq H$,
	$$|r(h,s,a)-r_{0k}(h,s,a)|\leq \Delta_k,\quad \|P(\cdot;h,s,a)-P_{0k}(\cdot;h,s,a)\|_1\leq \Delta_k.$$
\end{assumption}

Given any triple of $(h,s,a)$, let $\bar{r}_{0k}$ and $\bar{P}_{0k}$ be the sample mean reward and empirical transition probabilities estimated from $N_{k}$ i.i.d. observations in the historical data set $k$. Let $n$ be the number of i.i.d. observations from the target data with sample mean reward $\bar{r}_n$ and empirical transition probabilities $\bar{P}_n$. We denote by $\boldsymbol{\Delta}=(\Delta_1,...,\Delta_K)$ and $\mathbf{N}=(N_{1},...,N_{k})$. The data-pooling weight $\boldsymbol{\lambda}=(\lambda_1,...,\lambda_k)$ is also a $K$-dimension vector. We define the corresponding date-pooling estimators  
\begin{equation}
\hat{r}_n^{DPK}=(1-\lambda_1-...-\lambda_K)\bar{r}_n +\sum_{k=1}^K \lambda_k\hat{r}_{0k},\quad \hat{P}_n^{DPK}=(1-\lambda_1-...-\lambda_K)\bar{P}_n +\sum_{k=1}^K \lambda_k\hat{P}_{0k},
\label{app-eq:data-pool-multi}
\end{equation}
where the superscript $DPK$ represents for \textit{Data Pooling with K} historical data sets.

Note that if we mix the $K$ historical data sets into a single data set, we can view the mixed data set as generated from $\mathcal{M}_0=(H,\mathcal{S},\mathcal{A}, R_0,P_0)$ with 
$$r_0(h,s,a)=\frac{\sum_k N_kr_{0k}(h,s,a)}{\sum_k N_k},\quad P_0(\cdot;h,s,a)=\frac{\sum_k N_kP_{0k}(\cdot;h,s,a)}{\sum_k N_k},$$
and with sample size $N=\sum_k N_k$ and difference bound $\Delta = \sum_k N_k\Delta_k/\sum_k N_k$. Let $\varepsilon_L^{DP}(n,N), L = R, P, V$ be the corresponding confidence bounds for data-pooling estimators with the single historical dataset as defined in Theorem \ref{thm: data-pooling estimator}.

The next theorem says that when the similarity bounds $\Delta_k$ for different historical data sets are specified, we can construct new data-pooling estimators that have \emph{smaller} confidence radii than those specified in Theorem \ref{thm: data-pooling estimator} (mixing all historical datasets into a single set). This is achieved by assigning different weights to each historical dataset. 
\begin{theorem}
	\label{thm: data-pooling estimator multiple} 
	Under Assumption \ref{assmpt: difference bound multiple} with $\boldsymbol{\Delta}=(\Delta_1,...,\Delta_k)$, for each $n\geq 1$ and any value of $\delta>0$, we define the data-pooling estimator
	as given in~\eqref{app-eq:data-pool-multi} with 
	$\boldsymbol{\lambda}^{DPK}= (\lambda_1^{DPK},...,\lambda_K^{DPK})$ to be defined in \eqref{eq: lambda multiple}.
	For fixed $\mathbf{N}=(N_1,...,N_K)$ sample sizes of the historical data sets,	there exists a sequence of confidence radii  $\{\varepsilon^{DPK}_R(n,\mathbf{N}), \varepsilon^{DPK}_P(n,\mathbf{N}),\varepsilon^{DPK}_V(n,\mathbf{N})\}$ for the data-pooling estimators $\{\hat{r}_n^{DPK}, \hat{P}_n^{DPK}\}$ such that
		$$\varepsilon^{DPK}_R(n,\mathbf{N})\leq \varepsilon_R^{DP}(n,N),\quad \varepsilon^{DPK}_P(n,\mathbf{N})\leq\varepsilon_P^{DP}(n,N), \quad \varepsilon^{DPK}_V(n,\mathbf{N})\leq \varepsilon_V^{DP}(n,N),
		$$
		where $\varepsilon_L^{DP}(n,N), L = R, P, V$ are the confidence bounds for data-pooling estimators using a single historical data set with $N=\sum_k N_k$ and $\Delta = \sum_k N_k\Delta_k/\sum_k N_k$ as defined in Theorem \ref{thm: data-pooling estimator}.
\end{theorem} 

\begin{proof}{Proof of Theorem \ref{thm: data-pooling estimator multiple}}
	The proof is similar to that of Theorem \ref{thm: data-pooling estimator}. Given the form of the estimator $\hat{r}_n$, by the Bounded Difference Inequality, we obtain the confidence radius for the mean reward as
	$$\tilde{\varepsilon}^n_R(\boldsymbol{\lambda})= \sqrt{\log(2HSAT/\delta)\left(\frac{L^2}{2n}+\sum_{k=1}^K\frac{\lambda_k^2}{2N_k}\right)} +\sum_{k=1}^K\Delta_k\lambda_k, $$
	where $L = 1-\sum_k\lambda_k$.
	Then, for any $\lambda\in[0,1]$, if we let
	$$\lambda_k = \frac{N_k\lambda}{\sum_i N_i}, \text{ for } k=1,2,...,K,$$
	we recover the data-pooling estimator with the single historical data set for mean reward with weight $\lambda$, $N=\sum_k N_k$ and $\Delta=\sum_k\Delta_kN_k/\sum_k N_k$. In addition, the confidence bounds also coincide: 
	$$\tilde{\varepsilon}^n_R(\boldsymbol{\lambda})=\varepsilon^n_R(\lambda), \text{ if }\lambda_k = \frac{N_k\lambda}{\sum_i N_i}, \text{ for }k=1,2,...,K.$$ 
	As a consequence, if we choose
	\begin{equation}\label{eq: lambda multiple}
		\boldsymbol{\lambda}^{DPK}\in\arg\min_{\boldsymbol{\lambda}\in\mathcal{B}}\tilde{\varepsilon}^n_R(\boldsymbol{\lambda}),\quad \mathcal{B}=\{\boldsymbol{\lambda}:\lambda_k\geq 0, k=1,2,...,K, \sum_{k=1}^{K}\lambda_k\leq 1\},
	\end{equation} we have
	$$\varepsilon^{DPK}_R(n,\mathbf{N})=\min_{\boldsymbol{\lambda}\in\mathcal{B}}\tilde{\varepsilon}^n_R(\boldsymbol{\lambda})\leq \min_{\lambda\in[0,1]}\varepsilon^n_R(\lambda)=\varepsilon_R^{DP}(n,N).$$
	
	Given the result on $\varepsilon_R^{DPK}$, the analysis for confidences $\varepsilon_P^{DPK}$ and $\varepsilon_V^{DPK}$ follow the same argument as in the proof of Theorem \ref{thm: data-pooling estimator}.
\end{proof}
\begin{remark}
	Although we do not have a closed-form expression for $\boldsymbol{\lambda}^{DPK}$, we can show that $\tilde{\varepsilon}_R^n(\boldsymbol{\lambda})$ is a convex function on the convex set $\mathcal{B}$ and thus one can find $\boldsymbol{\lambda}^{DPK}$ via any standard convex optimization solver. 
	Let 
	$$g(\boldsymbol{\lambda})=\sqrt{\frac{L^2}{n}+\sum_{k=1}^K\frac{\lambda_k^2}{N_k}}, \quad \boldsymbol{\lambda}\in\mathcal{B}.$$
	For any $\boldsymbol{\lambda},\bar{\boldsymbol{\lambda}}\in\mathcal{B}$, we have
	\begin{align*}
		(g(\boldsymbol{\lambda})+g(\bar{\boldsymbol{\lambda}}))^2&=\frac{L^2+\bar{L}^2}{n}+\sum_{k=1}^K\frac{\lambda_k^2}{N_k}+\sum_{k=1}^K\frac{\bar{\lambda}_k^2}{N_k}+2\sqrt{\frac{L^2}{n}+\sum_{k=1}^K\frac{\lambda_k^2}{N_k}}\sqrt{\frac{\bar{L}^2}{n}+\sum_{k=1}^K\frac{\bar{\lambda}_k^2}{N_k}}\\
		&\geq \frac{L^2+\bar{L}^2}{n}+\sum_{k=1}^K\frac{\lambda_k^2}{N_k}+\sum_{k=1}^K\frac{\bar{\lambda}_k^2}{N_k}+ 2\left(\frac{L\bar{L}}{n}+\sum_{k=1}^K\frac{\lambda_k\bar{\lambda}_k}{N_k}\right)\\
		&= 4\left(\frac{((L+\bar{L})/2)^2}{n} +\sum_{k=1}^K\frac{((\lambda_k+\bar{\lambda}_k)/2)^2}{N_k}\right)= (2g((\boldsymbol{\lambda}+\bar{\boldsymbol{\lambda}}))/2))^2,
	\end{align*}
	by H\"{o}lder's inequality. As a consequence,  $g(\boldsymbol{\lambda})$ is convex on $\mathcal{B}$ and so is $\tilde{\varepsilon}_R^n(\boldsymbol{\lambda})=g(\boldsymbol{\lambda})+\sum_k\Delta_k\lambda_k$. 
\end{remark}

\section{Details on the Case study}
\label{app:case-study}

\subsection{Individual MDP Calibration} 
\label{regression}

Table \ref{table:regression stage1} and Table \ref{table:regression stage2} show the regression results for the two-stage Heckman's correction model, including the variables and their coefficients, standard deviations, and p-value. The first stage is to estimate the selection with a probit model and calculate the inverse Mills ratio (IMR). The outcome variable is the follow-up action (i.e., whether follow-up is scheduled or not). 
In the second stage, we include the IMR as an additional feature in the regression model for evaluating the treatment effect on the outcomes. The outcome variable is the 30-day readmission outcome. The independent variables (covariates) in the two stages do not need to overlap and we follow a similar procedure as in~\cite{briggs2004causal} to select the covariates in the two stages.

Once we get the prediction for the readmission probability within 30 days, we then leverage
it to calibrate the readmission probability of each week post discharge. Among those who had been readmitted within four weeks, we divide them into two groups according to whether the follow-up calls have been provided and estimate the conditional readmission probabilities. Below is the result for the conditional readmission rate in each week $\alpha_{a}^h\triangleq P(\text{readmit in week }h|\text{readmit within 4 weeks,~ treatment=}a)$, $h=1,\ldots,4$ and $a=0,1$. 
\vspace{-0.15in}
\begin{align*}
\alpha^{h}_0 &=[0.35, 0.28, 0.21, 0.16], \quad 
\alpha^{h}_1 =[0.39, 0.25, 0.20, 0.15].
\end{align*}

\vspace{-0.15in}
\begin{table}[htbp]
	\centering
\scalebox{0.85}{
	\begin{tabular}{|cccc|}
		\hline
		\multicolumn{4}{|c|}{Probit Regression Results}                                              \\ \hline
		\multicolumn{1}{|c|}{$x_{1i}'$} & \multicolumn{1}{c|}{coef} & \multicolumn{1}{c|}{std} & p value \\ \hline
		const                       & 2.0714                    & 0.077                    & 0       \\
		had\_operations\_1          & -0.0737                   & 0.06                     & 0.219   \\
		prior\_visit\_num           & -0.0327                   & 0.016                    & 0.047   \\
		num\_of\_transfers          & -0.0509                   & 0.014                    & 0       \\
		class\_B2                   & -0.12                     & 0.05                     & 0.016   \\
		class\_C                    & -0.2279                   & 0.048                    & 0       \\
		adm\_source\_DS             & -4.3445                   & 0.105                    & 0       \\
		adm\_source\_EL             & -3.5668                   & 0.089                    & 0       \\
		adm\_source\_ED             & -4.1269                   & 0.083                    & 0       \\
		adm\_source\_SD             & -4.1241                   & 0.122                    & 0       \\
		specialty\_Med              & 0.9346                    & 0.054                    & 0       \\
		specialty\_Ortho            & -0.3835                   & 0.1                      & 0       \\
		adm\_to\_icu\_1             & 0.3141                    & 0.105                    & 0.003   \\ \hline
	\end{tabular}
}
	\caption{Stage 1 regression results. The outcome variable is the follow-up action. }\label{table:regression stage1}
\end{table}

\vspace{-0.15in}
\begin{table}[htbp]
	\centering
\scalebox{0.85}{
	\begin{tabular}{|cccc|}
		\hline
		\multicolumn{4}{|c|}{Logit Regression Results }                                              \\ \hline
		\multicolumn{1}{|c|}{[$x_{2i}'$]} & \multicolumn{1}{c|}{coef} & \multicolumn{1}{c|}{std} & p value \\ \hline
		const              & -1.6385 & 0.577 & 0.005 \\
		follow\_up         & -1.173  & 0.535 & 0.028 \\
		prior\_visit\_num  & 0.387   & 0.018 & 0     \\
		num\_of\_transfers & 0.0783  & 0.015 & 0     \\
		Charlson\_score    & 0.1625  & 0.014 & 0     \\
		age                & 0.0111  & 0.001 & 0     \\
		LOS\_day           & 0.0597  & 0.044 & 0.172 \\
		gender\_M          & 0.0098  & 0.043 & 0.82  \\
		class\_C           & 0.2271  & 0.049 & 0     \\
		adm\_source\_DS    & -1.9791 & 0.539 & 0     \\
		adm\_source\_EL    & -1.512  & 0.512 & 0.003 \\
		adm\_source\_ED    & -1.4803 & 0.528 & 0.005 \\
		adm\_source\_SD    & -1.8509 & 0.535 & 0.001 \\
		specialty\_Surg    & 0.1754  & 0.059 & 0.003 \\
		Weekend\_True      & -0.006  & 0.056 & 0.914 \\
		adm\_to\_icu\_1    & -0.1886 & 0.123 & 0.126 \\
		had\_operations\_1 & -0.3023 & 0.064 & 0     \\
		month\_10          & 0.0654  & 0.081 & 0.418 \\
		month\_11          & 0.0648  & 0.08  & 0.421 \\
		month\_12          & 0.0502  & 0.079 & 0.528 \\
		month\_2           & -0.0118 & 0.083 & 0.888 \\
		month\_7           & 0.5439  & 0.171 & 0.001 \\
		month\_8           & 0.3898  & 0.078 & 0     \\
		month\_9           & 0.1973  & 0.079 & 0.013 \\
		Inverse Mills Ratio                 & 0.5101  & 0.247 & 0.039\\\hline
	\end{tabular}
} 
	\caption{Stage 2 regression results. The outcome variable is the 30-day readmission outcome. }\label{table:regression stage2}
\end{table}

\subsection{Target Patients and Optimal Policies}
To construct the oracle simulator, we classify the target patients into $146$ classes as summarized in Table \ref{table: target info} based on the predictive features from the two-stage regression. The optimal policy for each class is calculated by solving the corresponding MDP model (assuming the true parameters are known) under the intervention cost and readmission cost $c_a = 0.13$ and $c_R = 10$, respectively.

\vspace{-0.1in}
\begin{table}[htbp] 
	\centering
	\scalebox{0.68}{
	\resizebox{\textwidth}{!}{%
	    \begin{tabular}{|c|c|c|c|c|c|c|c|c|c|c|c|}
    \hline
        ~ & Samples & \begin{tabular}[c]{@{}c@{}}weekly  \\ arrival\end{tabular}  & \begin{tabular}[c]{@{}c@{}}optimal  \\ policy\end{tabular} &  ~ & Samples & \begin{tabular}[c]{@{}c@{}}weekly  \\ arrival\end{tabular}  & \begin{tabular}[c]{@{}c@{}}optimal  \\ policy\end{tabular} &  ~ & Samples & \begin{tabular}[c]{@{}c@{}}weekly  \\ arrival\end{tabular}  & \begin{tabular}[c]{@{}c@{}}optimal  \\ policy\end{tabular} \\ \hline
        0 & 32 & 2 & 1111 & 50 & 31 & 2 & 1111 & 100 & 261 & 19 & 0000 \\ \hline
        1 & 65 & 3 & 1111 & 51 & 26 & 2 & 1111 & 101 & 46 & 4 & 1100 \\ \hline
        2 & 124 & 6 & 1111 & 52 & 77 & 5 & 1111 & 102 & 47 & 4 & 1100 \\ \hline
        3 & 123 & 6 & 1111 & 53 & 15 & 2 & 1111 & 103 & 39 & 3 & 0000 \\ \hline
        4 & 74 & 6 & 1111 & 54 & 53 & 4 & 1111 & 104 & 69 & 6 & 1100 \\ \hline
        5 & 204 & 14 & 1110 & 55 & 55 & 5 & 1111 & 105 & 67 & 5 & 1100 \\ \hline
        6 & 362 & 26 & 1111 & 56 & 28 & 3 & 1111 & 106 & 35 & 3 & 1100 \\ \hline
        7 & 282 & 19 & 1111 & 57 & 45 & 4 & 1111 & 107 & 56 & 4 & 1000 \\ \hline
        8 & 16 & 1 & 1111 & 58 & 27 & 3 & 1111 & 108 & 52 & 4 & 1100 \\ \hline
        9 & 85 & 5 & 1111 & 59 & 56 & 5 & 1111 & 109 & 66 & 5 & 0000 \\ \hline
        10 & 50 & 3 & 1100 & 60 & 25 & 2 & 1100 & 110 & 93 & 7 & 0000 \\ \hline
        11 & 81 & 4 & 1000 & 61 & 54 & 4 & 1111 & 111 & 117 & 8 & 1000 \\ \hline
        12 & 58 & 3 & 0000 & 62 & 54 & 4 & 1100 & 112 & 113 & 8 & 0000 \\ \hline
        13 & 17 & 2 & 1110 & 63 & 41 & 3 & 0000 & 113 & 37 & 3 & 0000 \\ \hline
        14 & 118 & 10 & 0000 & 64 & 92 & 7 & 1100 & 114 & 81 & 6 & 0000 \\ \hline
        15 & 120 & 8 & 0000 & 65 & 30 & 3 & 1100 & 115 & 54 & 4 & 0000 \\ \hline
        16 & 25 & 2 & 0000 & 66 & 67 & 5 & 1100 & 116 & 92 & 7 & 0000 \\ \hline
        17 & 98 & 7 & 0000 & 67 & 158 & 12 & 1100 & 117 & 167 & 11 & 0000 \\ \hline
        18 & 85 & 6 & 0000 & 68 & 59 & 4 & 1111 & 118 & 219 & 16 & 0000 \\ \hline
        19 & 238 & 15 & 0000 & 69 & 75 & 5 & 1111 & 119 & 141 & 10 & 0000 \\ \hline
        20 & 155 & 7 & 1111 & 70 & 35 & 3 & 1100 & 120 & 70 & 4 & 1111 \\ \hline
        21 & 17 & 1 & 1111 & 71 & 116 & 9 & 1111 & 121 & 40 & 3 & 1111 \\ \hline
        22 & 37 & 2 & 1111 & 72 & 268 & 19 & 1110 & 122 & 73 & 5 & 1111 \\ \hline
        23 & 58 & 4 & 1111 & 73 & 54 & 4 & 0000 & 123 & 74 & 4 & 1111 \\ \hline
        24 & 47 & 3 & 1111 & 74 & 52 & 5 & 1000 & 124 & 27 & 2 & 1100 \\ \hline
        25 & 66 & 5 & 1111 & 75 & 47 & 4 & 1100 & 125 & 51 & 3 & 1100 \\ \hline
        26 & 52 & 3 & 1111 & 76 & 37 & 3 & 0000 & 126 & 36 & 2 & 1111 \\ \hline
        27 & 83 & 5 & 1111 & 77 & 70 & 6 & 1100 & 127 & 34 & 4 & 1111 \\ \hline
        28 & 86 & 6 & 1111 & 78 & 34 & 4 & 1100 & 128 & 38 & 4 & 1100 \\ \hline
        29 & 63 & 4 & 1111 & 79 & 70 & 5 & 1100 & 129 & 39 & 4 & 1100 \\ \hline
        30 & 87 & 5 & 1111 & 80 & 41 & 3 & 1110 & 130 & 47 & 5 & 1100 \\ \hline
        31 & 159 & 8 & 1111 & 81 & 98 & 7 & 1100 & 131 & 40 & 3 & 0000 \\ \hline
        32 & 33 & 2 & 1111 & 82 & 46 & 4 & 1100 & 132 & 58 & 4 & 0000 \\ \hline
        33 & 18 & 2 & 1111 & 83 & 28 & 2 & 1100 & 133 & 29 & 2 & 0000 \\ \hline
        34 & 39 & 3 & 1111 & 84 & 125 & 9 & 1110 & 134 & 47 & 4 & 0000 \\ \hline
        35 & 41 & 2 & 1111 & 85 & 55 & 4 & 1100 & 135 & 50 & 4 & 0000 \\ \hline
        36 & 51 & 3 & 1111 & 86 & 68 & 5 & 1111 & 136 & 98 & 6 & 1100 \\ \hline
        37 & 26 & 2 & 1111 & 87 & 48 & 4 & 1111 & 137 & 44 & 3 & 0000 \\ \hline
        38 & 68 & 4 & 1111 & 88 & 104 & 7 & 1110 & 138 & 90 & 5 & 0000 \\ \hline
        39 & 44 & 3 & 1111 & 89 & 37 & 4 & 1100 & 139 & 127 & 6 & 0000 \\ \hline
        40 & 64 & 4 & 1111 & 90 & 47 & 4 & 1100 & 140 & 87 & 9 & 0000 \\ \hline
        41 & 26 & 2 & 1111 & 91 & 83 & 5 & 1100 & 141 & 157 & 11 & 0000 \\ \hline
        42 & 73 & 4 & 1100 & 92 & 171 & 13 & 1100 & 142 & 24 & 2 & 0000 \\ \hline
        43 & 86 & 6 & 1111 & 93 & 66 & 5 & 1100 & 143 & 403 & 28 & 0000 \\ \hline
        44 & 20 & 2 & 1111 & 94 & 50 & 3 & 0000 & 144 & 50 & 5 & 0000 \\ \hline
        45 & 95 & 6 & 1111 & 95 & 64 & 5 & 1100 & 145 & 487 & 34 & 0000 \\ \hline
        46 & 77 & 5 & 1100 & 96 & 245 & 17 & 1000 & ~ & ~ & ~ & ~ \\ \hline
        47 & 27 & 2 & 1110 & 97 & 56 & 5 & 0000 & ~ & ~ & ~ & ~ \\ \hline
        48 & 43 & 3 & 1110 & 98 & 44 & 4 & 0000 & ~ & ~ & ~ & ~ \\ \hline
        49 & 22 & 2 & 1111 & 99 & 84 & 6 & 0000 \\ \hline
    \end{tabular}%
    }
    }
	\caption{146 target patient classes and their optimal policies (1 = intervention; 0 = no intervention), total sample size and number of arrival per week (rounded to integer).
		\label{table: target info}}
\end{table}


\subsection{Historical Data}
\label{Appendix:historical}

The information of historical patients is detailed in Tables~\ref{Table: historical data classification} and~\ref{Table: historical data P in RL}. The formation of these eight clusters aligns with medical practice: if hospitals and clinicians intend to use historical data in predicting readmission risks for new patients, a natural way is to partition the historical patients into different groups by their features and use the medically similar group(s) as the basis for prediction on the new patients. The features we use here are consistent with those commonly used in calculating readmission risk scores in the literature, e.g., the LACE score~\citep{robinson2017hospital}. 

\vspace{-0.2in}
\begin{table}[htbp]
\centering
\scalebox{0.67}{
\begin{tabular}{@{}ccccc@{}}
\toprule
\multicolumn{1}{l}{} & \multicolumn{1}{c}{\textbf{adm source}} & \multicolumn{1}{c}{\textbf{prior\_visit\_num}} & \multicolumn{1}{c}{\textbf{age}} & \multicolumn{1}{c}{\textbf{$N$}} \\ \midrule
\textbf{1}           & ED                                      & 0                                              & {[}18,41)                        & 1425                             \\
\textbf{2}           & ED                                      & 0                                              & {[}41,61)                        & 1766                             \\
\textbf{3}           & ED                                      & 0                                              & {[}61,80{]}                      & 1728                             \\
\textbf{4}           & ED                                      & \textgreater{}0                                & -                                & 973                             \\
\textbf{5}           & nonED                                  & 0                                              & {[}18,41)                        & 1015                             \\
\textbf{6}           & nonED                                  & 0                                              & {[}41,61)                        & 1310                             \\
\textbf{7}           & nonED                                  & 0                                              & {[}61,80{]}                      & 859                              \\
\textbf{8}           & nonED                                  & \textgreater{}0                                & -                                & 563                              \\ \bottomrule
\end{tabular}
}
\caption{The medical features in each of the eight clusters of the historical data. The admission source is classified as emergency admission (ED) and non-emergency admission (nonED). }
\label{Table: historical data classification}
\end{table}


\vspace{-0.3in}
\begin{table}[htbp]
\centering
\scalebox{0.8}{
\begin{tabular}{@{}cccccccccc@{}}
\toprule
cluster    & \textbf{policy} & \textbf{$p_{00}$} & \textbf{$p_{01}$} & \textbf{$p_{10}$} & \textbf{$p_{11}$} & \textbf{$p_{20}$} & \textbf{$p_{21}$} & \textbf{$p_{30}$} & \textbf{$p_{31}$} \\ \midrule
     \textbf{1} &  0000 & 0.0311  & 0.0176  & 0.0257  & 0.0131  & 0.0199  & 0.0116  & 0.0156  & 0.0083  \\ 
     \textbf{2} &   1100 & 0.0473  & 0.0275  & 0.0399  & 0.0207  & 0.0314  & 0.0185  & 0.0250  & 0.0133  \\ 
        \textbf{3} &1111 & 0.0622  & 0.0371  & 0.0534  & 0.0282  & 0.0428  & 0.0254  & 0.0346  & 0.0185  \\ 
        \textbf{4} &1111 & 0.0886  & 0.0559  & 0.0791  & 0.0437  & 0.0660  & 0.0403  & 0.0552  & 0.0300  \\ 
        \textbf{5} &0000 & 0.0250  & 0.0143  & 0.0208  & 0.0107  & 0.0162  & 0.0095  & 0.0128  & 0.0068  \\ 
        \textbf{6} &1100 & 0.0406  & 0.0238  & 0.0344  & 0.0179  & 0.0273  & 0.0160  & 0.0217  & 0.0116  \\ 
        \textbf{7} &1111 & 0.0547  & 0.0327  & 0.0470  & 0.0249  & 0.0378  & 0.0225  & 0.0305  & 0.0163  \\ 
        \textbf{8} &1111 & 0.0631  & 0.0386  & 0.0551  & 0.0297  & 0.0449  & 0.0270  & 0.0368  & 0.0198 \\ 
\bottomrule 
\end{tabular}
}
\caption{The transition probability $p_{ha}$ at each $(h,a)$ pair (after being corrected for the selection bias) and the corresponding optimal policies in each of the eight clusters for the historical patients.}
\label{Table: historical data P in RL}
\end{table}

\medskip
\subsection{Algorithms Implementation Details}
\label{Appendix:algorithm details}

\paragraph{JS Algorithm. }  
JS algorithm updates estimator \eqref{eq: data-pooling estimates} in the main paper using a different weight for each historical patient class $k$, given by 
$$\lambda_{t,k}^{JS}(h,s,a)=\min\left(\max\left(1- \dfrac{(K-2)\sum_{i=1}^{N_k(h,s,a)}(q_{k,i}-\bar{q}_k)^2}{N_k(h,s,a)\sum_{l=0}^{K}(\bar{q}_l-\bar{q})^2},0\right) ,1 \right), \quad 1 \leq k \leq K.
$$
Here, we use $k=0$ to denote the target class. For each $k$, $N_k(h,s,a)$ is the number of samples in class $k$, $q_{k,i} = \hat{r}(h,s,a)+\hat{V}_{h+1}(s^i_{h+1})$ is the observed Q-function corresponding to the $i$-th sample in class $k$, $\bar{q}_k$ is the sample mean of the Q-functions in class $k$, and $\bar{q}$ is the sample mean over all classes. 
The weight for the target  class is $\lambda^{JS}_{t,0}(h,s,a)=1 - \sum_{k=1}^K \lambda^{JS}_{t,k}(h,s,a)$.

\paragraph{Contextual-P Algorithm. } We denote the collected target dataset right before iteration $t$ as $\mathcal{D}^t=\{(h,a,R_h,S_{h+1},x)_i\}$, with $\mathcal{D}^0$ being the historical dataset. Contextual-P algorithm merges the target dataset $\mathcal{D}^t$ with the historical dataset $\mathcal{D}^0$ to construct a merged dataset $\tilde{\mathcal{D}}^t$ and estimates the readmission probability using the merged dataset $\tilde{\mathcal{D}}^t$. 
For each sample $i$,  $x_i\in\mathbb{R}^{19}$ is its feature vector and the features are: 
\small 
\begin{align*}
x =[~ 'prior\_visit\_num', 'num\_of\_transfers', 'Charlson\_score', 'age',
'LOS\_day','gender\_M','class\_B2','class\_C',
\\
'adm\_source\_DS', 'adm\_source\_EL','adm\_source\_ED','adm\_source\_ES','adm\_source\_SD',\\'specialty\_Med','specialty\_Ortho',
'specialty\_Surg',
'Weekend\_True','adm\_to\_icu\_1','had\_operations\_1'].
\end{align*}

\normalsize
Recall that $a\in\{0,1\}$ is the action (taking intervention or not). The readmission probability of a patient with feature $x$ under action $a$ at decision epoch $h$ is fit by a linear model as
\begin{align*}
p^h(a,x) = c^h_1a + c^h_{2}\cdot x + c^h_3,
\end{align*}
where $c^h_1\in\mathbb{R}, c^h_2 \in \mathbb{R}^{19}$ and $c^h_3\in\mathbb{R}$ are the parameters to be estimated. Once the readmission probabilities are estimated, contextual-P then computes the Q-functions using the estimated readmission probabilities and selects the action for iteration $t$ accordingly.

\paragraph{Contextual-Q Algorithm. } Contextual-Q algorithm is similar to contextual-P, except that it directly estimates the Q-function value instead of the readmission probability. In contextual-Q algorithm, the Q-function value $Q_h(a;x)$ is fit by a linear model as 
\begin{align*}
    Q_h(a;x) = \theta^h_1 a + \theta^h_2 \cdot x + \theta^h_3,
\end{align*}
where $\theta^h_1\in\mathbb{R}, \theta^h_2 \in \mathbb{R}^{19}$ and $\theta^h_3\in\mathbb{R}$ are the parameters to be estimated.
For each patient sample $i$,  the observed Q-function $\tilde{Q}^i_h(a;x)$ at decision epoch $h$ is calculated from data following the one-step backward induction
\begin{align*}
\tilde{Q}^i_h(a;x) = c_a \mathbf{1}(a=1) + \hat{V}^i_{h+1}\cdot \mathbf{1}(S_{h+1}=0)+c_R\cdot \mathbf{1}(S_{h+1}=1), 
\end{align*}
where $x_i$ is the feature of patient sample $i$, $\hat{V}^i_{h+1}=\min\limits_a \hat{Q}_{h+1}(a;x_i)$ is the estimated value function of patient $i$ at decision epoch $h + 1$, and $\hat{Q}_{h+1}(a;x_i)$ is the estimated Q-function value of patient sample $i$ at decision epoch $h+1$.

\paragraph{Clustering Algorithm. } We follow the idea in~\cite{miao2019context} to perform an adaptive clustering algorithm. 
At each iteration $t$, we first compute the MLE $\hat{P}^t(\cdot;h,\cdot,a)$ using $n_t(h,\cdot,a)$ target samples. Then, we cluster the target and historical patients according to the closeness in the estimation, defined via 
$$
||\bar{P}^k(\cdot;h,\cdot,a)-\hat{P}^t(\cdot;h,\cdot,a)||_2\le C/\sqrt{n_t(h,\cdot,a)}, 
$$
where $C$ is a hyper-parameter and $\bar{P}^k(\cdot;h,\cdot,a)$ is the MLE estimators from class $k$'s historical data. Lastly, we use the clustered data to estimate the value function for the target patients and move on to the next iteration.

\paragraph{Data-pooling Algorithm. } 
Weights for $K$ different historical datasets are computed by minimizing 
\begin{align*}
\varepsilon_t^{ha}(\lambda) &=\sqrt{\log(2HSAn^2)(1+(H-h)^2)\left(\dfrac{(1-\lambda_1-\lambda_2\ldots-\lambda_K)^2}{n}+\dfrac{\lambda_1^2}{N_1}+\dfrac{\lambda_2^2}{N_2}+\ldots+\dfrac{\lambda_K^2}{N_K}\right) } \\+&(1+H-h)(\lambda_1\Delta_{t1}^{ha}+\lambda_2\Delta_{t2}^{ha}+\ldots+\lambda_K\Delta_{tK}^{ha}), 
\end{align*}
where $\Delta_{tk}^{ha} = ||\bar{P}^k(\cdot;h,\cdot,a)-\hat{P}^t(\cdot;h,\cdot,a)||_1\times\gamma$ with $\gamma$ being a hyper-parameter.

\noindent\textbf{Perturbed framework. } In all of the algorithms, we follow the perturbed LSVI framework. We assume that the extra randomness injected for exploration is a Gaussian noise with mean 0 and variance $\sigma^2/n^t(h,\cdot,a)$ for each $(h,a)$ pair at iteration $t$, with $\sigma$ being a hyper-parameter.

\vspace{-0.1in}
\subsection{Hyper-parameters Tuning}
\label{Appendix: Hyper-parameters}

The hyper-parameter $\sigma$ for exploration in each of the algorithms is chosen from the set $\{0.01,0.05,0.1,0.2,0.4\}$ to achieve the best performance via cross validation. For additional hyper-parameters in each of the algorithms: (i) the candidate of $\gamma$ for data-pooling is selected from $\{0.01,0.3,0.5,0.7,1,1.5,2,10\}$; (ii) the candidate of $C$ for clustering is selected from $\{0.01,0.1,0.3,0.5,0.7,1,2,4\}$.
We tune all these parameters to give enough benefit of doubts for each algorithm. The results under different hyper-parameters are showed in Tables~\ref{Table:Hyper,Data-pooling}, \ref{Table:Hyper,Clustering}, and~\ref{Table:Hyper,others}. The results show that there is a reasonably large range of hyper-parameters that our data-pooling method achieves better performance than other benchmarks. This also indicates the robustness of the performance of our algorithm over the hyper-parameter values; see more in Section~\ref{app:robust-gamma}. 

\vspace{-0.1in}
\begin{table}[htbp]
\centering
\scalebox{0.85}{
\begin{tabular}{|c|c|c|}
\hline
Algorithm\textbackslash{}Hyper-parameters & $\sigma$ & Extra Hyper-parameter \\
\hline
\textbf{Data-pooling}                                         & 0.1      & $\gamma=0.7$               \\\hline
\textbf{Clustering}                                           & 0.2      & $C=0.5$               \\\hline
\textbf{JS}                                                   & 0.1      &                       \\\hline
\textbf{Complete}                                             & 0.05      &                       \\\hline
\textbf{Personalized}                                         & 0.2      &                       \\\hline
\textbf{Contextual-Q}                                  & 0.05      &                       \\\hline
\textbf{Contextual-P}                              & 0.05      &    \\ \hline                  
\end{tabular}
}
\caption{The hyper-parameters used in the algorithms that achieved smallest regret.}\label{Table:Hyper-parameters}
\end{table}

\vspace{-0.3in}
\begin{table}[htbp]
\centering
\scalebox{0.85}{
\begin{tabular}{|c|c|c|c|c|c|c|c|c|}
\hline
$\sigma$\textbackslash$\gamma$ & 0.01     & 0.3      & 0.5      & 0.7      & 1        & 1.5      & 2        & 10       \\ \hline
0.01                              & 3764.702 & 3131.574 & 3153.299 & 3124.427 & 3333.851 & 3672.336 & 3892.774 & 4518.580 \\ \hline
0.05                              & 3887.107 & 3087.282 & 2959.274 & 2885.548 & 3116.548 & 3488.318 & 3730.835 & 4452.750  \\ \hline
0.1                               & 4018.540 & 3172.785 & 2930.379 & 2813.329  & 2992.442 & 3354.466 & 3625.853 & 4349.501  \\ \hline
0.2                               & 4319.077 & 3379.599 & 3020.867 & 2874.198  & 2973.911 & 3273.588 & 3524.715 &  4220.833 \\ \hline
0.4                               & 4697.759 & 3871.202 & 3415.256 & 3187.995 & 3165.118 & 3358.996 & 3579.635 & 4144.007 \\ \hline
\end{tabular}
}
\caption{Total regret in 50 iterations of data-pooling using different hyper-parameters}\label{Table:Hyper,Data-pooling}
\end{table}

\vspace{-0.3in}
\begin{table}[htbp]
\centering
\scalebox{0.85}{
\begin{tabular}{|c|c|c|c|c|c|c|c|c|}
\hline
$\sigma$\textbackslash   C & 0.01     & 0.1      & 0.3      & 0.5      & 0.7      & 1        & 2        & 4        \\ \hline
0.01                       & 4681.712 & 4449.651 & 4266.776 & 4503.931 & 4483.657 & 4179.126 & 3958.573 & 3927.770  \\ \hline
0.05                       & 4655.609 & 4371.936 & 4050.210 & 4075.693 & 4160.202 & 4168.430 & 4028.130 & 3993.867 \\ \hline
0.1                        & 4565.224 & 4295.532 & 3971.203 & 3924.858 & 4016.826 & 4063.454 & 4122.778 & 4100.217 \\ \hline
0.2                        & 4439.702 & 4205.748 & 3880.789 & 3848.244 & 4003.407 & 4169.209 & 4362.318 & 4366.421  \\ \hline
0.4                        & 4347.206 & 4161.077 & 3886.968 & 3926.420 & 4119.112 & 4367.080 & 4702.956 & 4727.316 \\ \hline
\end{tabular}
}
\caption{Total regret in 50 iterations of clustering using different hyper-parameters}\label{Table:Hyper,Clustering}
\end{table}

\begin{table}[H]
\centering
\scalebox{0.85}{
\begin{tabular}{@{}cccccc@{}}
\toprule
Algorithm\textbackslash$\sigma$ & 0.01     & 0.05     & 0.1      & 0.2      & 0.4      \\ \midrule
JS                                 & 3998.232 & 4018.119 & 3994.962 & 4042.544 & 4214.890  \\
Complete                           & 4188.623 & 4095.432 & 4163.520 & 4286.474 & 4419.856 \\
Personalized                       & 4469.011 & 4452.541 & 4417.864 & 4405.007 & 4317.601 \\
Contextual-Q                & 3548.832 & 3529.487 & 3570.606 & 3637.642 & 3713.021 \\
Contextual-P            & 4382.261 & 4342.012 & 4432.124 & 4415.164 & 4397.171 \\ \bottomrule
\end{tabular}
}
\caption{Total regret in 50 iterations using different hyper-parameters $\sigma$}
\label{Table:Hyper,others}
\end{table}

\vspace{-0.3in}
\subsection{Comparison with Clustering}
\label{app:comparison-cluster-context} 

Figure~\ref{fig:comparison_clustering_pooling} plots the trajectories under the clustering algorithm and our data-pooling algorithm on one sample path. Each dot $(1-p_{ha}, p_{ha})$ in the plane represents the combination of transition probabilities of staying at the non-readmit state or moving into the readmit state, respectively. The black dot $(0.96, 0.04)$ in the figure represents the true transition probabilities. Both algorithms start from the initial value $(1,0)$, and the closer it gets to the true estimation (black dot), the better. Note that within 20 iterations, the trajectory (red line) under data-pooling gets close enough to the black dot, and stays close around $(0.97, 0.03)$ from iteration 20 to iteration 50. In contrast, due to the variance in sampling, the trajectory (green line) under clustering gets far away from the black dot -- it stayed around the point $(0.88, 0.12)$ after iteration~30. This drift-away results in picking the suboptimal action ($a=0$) after iteration 30 since action $a=0$ results in a lower cost even with exploration (the exploration noise is not enough to overcome the gap between the two actions). Consequently, this leads to a larger regret and higher cost during the learning under the clustering method, in comparison to our data-pooling algorithm. 
It is worth mentioning that although we show one sample path here for illustration, similar situations happened over many other sample paths. 

\vspace{-0.2in}
\begin{figure}[htbp]
    \centering
    \includegraphics[width=0.5\linewidth]{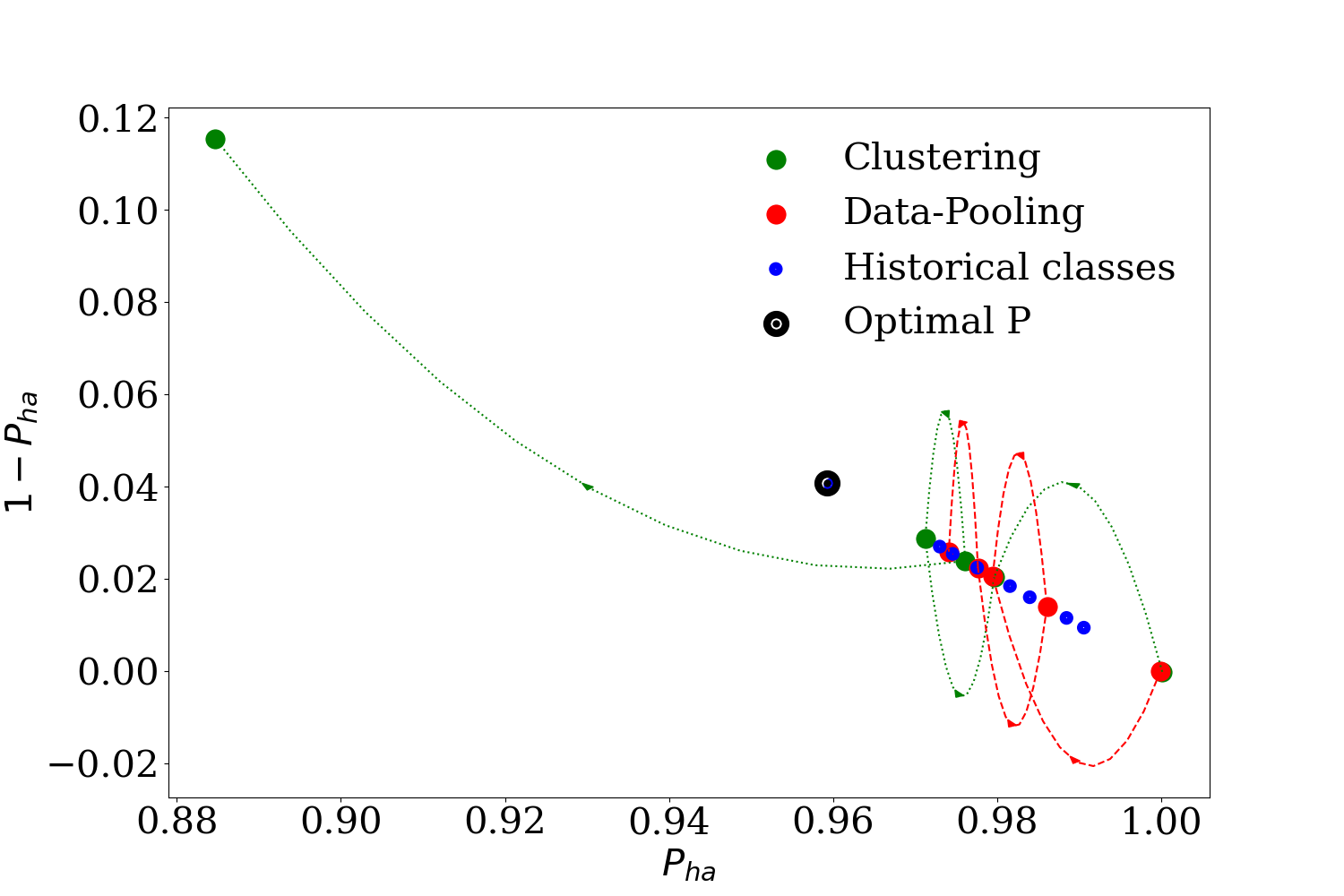}
    \caption{Trajectory of estimating transition probability in one sample path. }
    \label{fig:comparison_clustering_pooling}
\end{figure}

\subsection{Additional Results for the Synthetic Experiments}
\label{app:synth-exp}
Table \ref{table: regret synthetic} summarizes the performance in numeric values for data-pooling, contextual-Q and contextual-P in the synthetic experiments. Figure \ref{fig:syn_lam} illustrates that our data-pooling algorithm successfully pools historical class 1 to target class 0.

\begin{table}[htbp]
\centering 
\scalebox{0.9}{
\begin{tabular}{@{}cllcc@{}}
\toprule
Algorithm             & Total Regret & Total Cost & \begin{tabular}[c]{@{}c@{}}Readmission \\  Rate\end{tabular} & CPU time/class \\ \midrule
 Data-pooling            &       165.3$\pm$18.14  & 6860.49$\pm$50.31  & 5.80\% & 232.36 \\ 
  Contextual-Q     &       299.72$\pm$4.79  & 7007.82$\pm$42.36 & 6.24\% & 1411.15 \\ 
Contextual-P &         386.14$\pm$3.28 & 5415.19$\pm$37.35 & 4.88\% & 1485.36 \\ \hline
\end{tabular}
}
\caption{The Performance of different algorithms for 50 iterations and 100 rounds. 
The numbers following the $\pm$ sign are the half-width of the 95\% confidence interval of the corresponding value. 
}
\label{table: regret synthetic}
\end{table}

\vspace{-0.2in}
\begin{figure}[htbp]
\centering
\subfigure[$a=0$]{\includegraphics[width=0.38\linewidth]{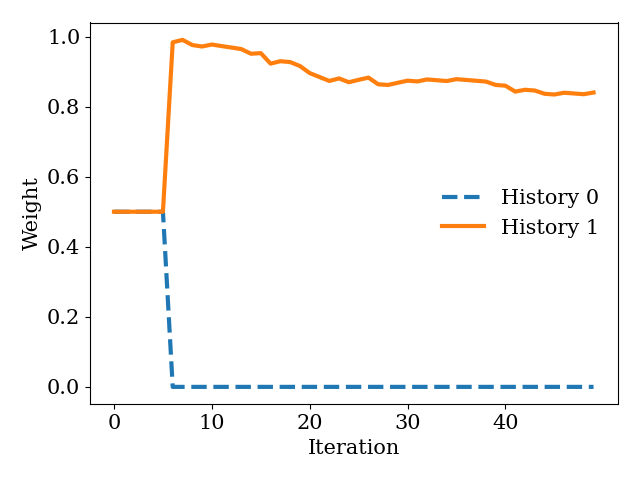}\label{fig:syn_lam_0}} 
\quad 
\subfigure[$a=1$]{\includegraphics[width=0.38\linewidth]{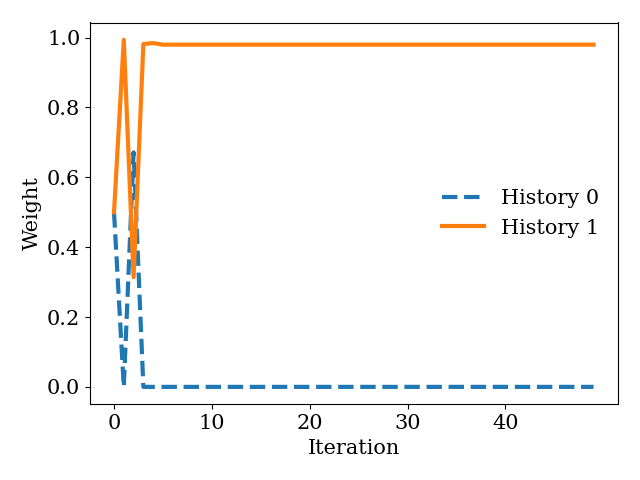}\label{fig:syn_lam_1}} 
\caption{The weight for target 0 at period 0 used in the data-pooling algorithm.}
\label{fig:syn_lam}
\end{figure}

\subsection{Robustness Analysis}
\label{app:robust-gamma}

When implementing algorithm-based decision support in practice, there could be different operational barriers. We discuss how to deal with a few possible ones in this section:
data-privacy environment, robustness to $\Delta$ (the difference between target and historical patients), and different treatment effect.

\vspace{-0.1in}
\subsubsection{Dealing with Data-sharing Restrictions}
\label{app:data-privacy}

We compare our data-pooling algorithm with the clustering method when data-sharing is restricted. 
In this numerical experiment, only the aggregate statistics in terms of empirical averages are available, i.e., $\bar{P}_0(h,a)$, $h=1,2,3,4$. It is worth stressing that this empirical average is sufficient for us to implement our data-pooling algorithm, whereas it is not for the clustering method. To provide the necessary input for the clustering method, we randomly label the historical data into eight groups. Figure~\ref{fig:nolabel, regret} compares the regret from the data-pooling algorithm and the clustering method. The performance gap between the two algorithms increases comparing to that in the baseline, with the clustering method having a worse result. We also note that the regret from the clustering method decreases much slower in this no-label environment than that in the baseline. In contrast, our data-pooling algorithm still performs very well using a single cluster -- that is, we only used the average and sample size of the ungrouped historical data. This robustness is particularly appealing when one does not have access to critical patient information to classify the historical data. 

\vspace{-0.2in}
\begin{figure}[!htb]
    \centering
    \hspace*{\fill}%
    \begin{minipage}[t]{.5\textwidth}
        \centering
        \vspace{0pt}
        \includegraphics[width=0.98\columnwidth]{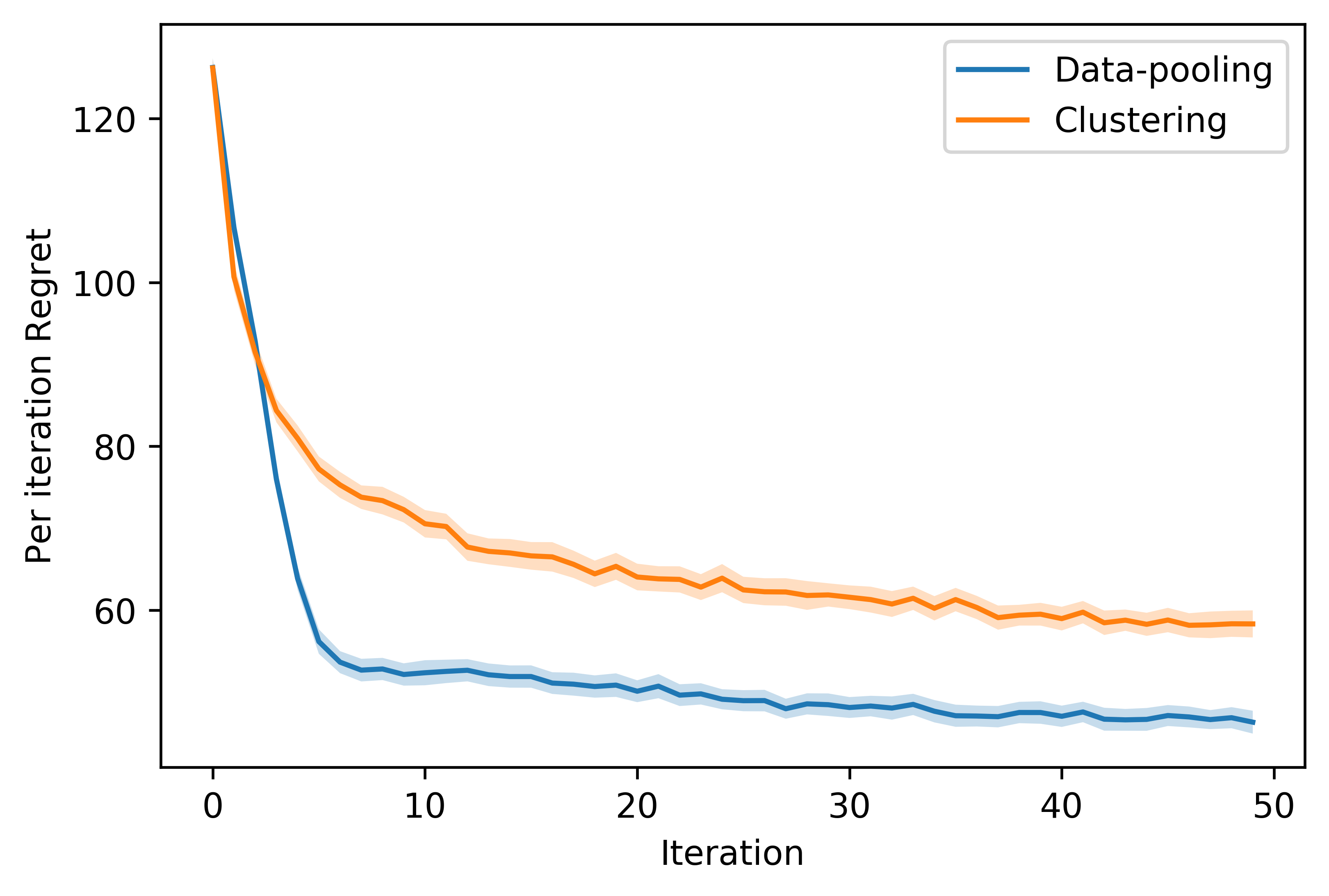}
        \caption{Total regret of 50 patient classes under no-label situation for 50 iterations and 100 rounds. The shaded area represents the 95\% confidence interval.}
        \label{fig:nolabel, regret}
    \end{minipage}%
    \hfill
    \begin{minipage}[t]{0.5\textwidth}
        \centering
        \vspace{0pt}
        \includegraphics[scale=0.4]{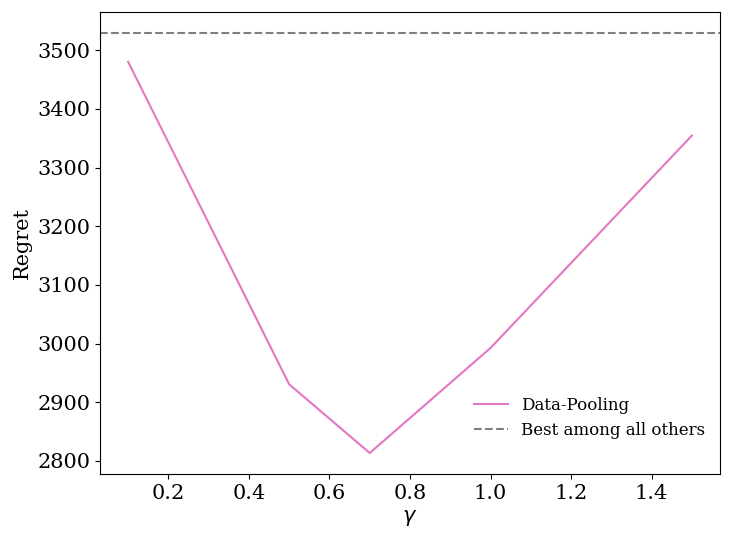}
        \caption{Robustness check for hyper-parameter $\gamma$.}
        \label{fig:robust_check}
    \end{minipage}
    \hspace*{\fill}
\end{figure}





\vspace{-0.1in}
\subsubsection{Robustness of Hyper-parameter $\gamma$}
\label{subsection:robust_check}
To deal with unknown value of difference gap $\Delta$, we introduce $\gamma$ as a tuning parameter in our implementation of the data-pooing method (see Section~\ref{Appendix:algorithm details}). Here we conduct simple robustness check experiments on $\gamma$ with other hyper-parameters unchanged. We compare the performance of data-pooling with different values of $\gamma$ with the best regret obtained by other benchmark algorithms. As shown in Figure \ref{fig:robust_check}, unless the choice of $\gamma$ is far away from the default value (1.0), the performance of data-pooling algorithm is significantly better than the best of other benchmarks.


\subsubsection{Different Treatment Effects}
\label{app:diff-treat}

In this experiment, we follow the synthetic experiment setting in Section~\ref{sec:robust-under-misspec}. However, we modify the coefficient of the action, $c_1$, to $-0.075$ for the target patients (no change for historical patients). Thus, the target class becomes even more different from the historical classes, imposing more difficulty in learning. The optimal policies for each class are summarized in Table~\ref{table:policy-diff-treat}.

\begin{table}[htbp]
\centering
\begin{tabular}{c|ccccccccc}
\multicolumn{1}{c|}{\textbf{class}} & \multicolumn{1}{c}{\textbf{$p_{00}$}} & \multicolumn{1}{c}{\textbf{$p_{01}$}} & \multicolumn{1}{c}{\textbf{$p_{10}$}} & \multicolumn{1}{c}{\textbf{$p_{11}$}} & \multicolumn{1}{c}{\textbf{$p_{20}$}} & \multicolumn{1}{c}{\textbf{$p_{21}$}} & \multicolumn{1}{c}{\textbf{$p_{30}$}} & \multicolumn{1}{c}{\textbf{$p_{31}$}} & \multicolumn{1}{c}{\textbf{opt policy}} \\ \hline
target                              & 2.9\%                            & 1.8\%                            & 2.9\%                            & 1.8\% & 6.7\% & 4.1\% & 6.7\% & 4.1\% & 0011   \\
target                              & 7.3\%                            & 4.7\%                            & 7.3\%                            & 4.7\% & 3.1\% & 2.0\% & 3.1\% & 2.0\% & 1100   \\
history                             & 7.4\%                            & 5.4\%                            & 7.4\%                            & 5.4\% & 3.2\% & 2.3\% & 3.2\% & 2.3\% & 0011   \\
history                             & 2.8\%                            & 2.0\%                            & 2.8\%                            & 2.0\% & 6.6\% & 4.7\% & 6.6\% & 4.7\% & 1100  
\end{tabular}
\caption{Readmission rate over the four-week time window for different treatment effects.}
\label{table:policy-diff-treat}
\end{table}

\noindent\textbf{Performance Comparison. } Figure \ref{fig:regret synthetic diff treat} plots the regret per iteration for the three tested policies. Table \ref{table: regret synthetic diff treat} summarize the performance comparison in numeric values. Our numerical results show that the data-pooling method has the smallest regret comparing to the other two methods. The gap in the regret increases from around 80\% in the setting with same treatment effect to more than 200\% in this setting with different treatment effect. Similar as before, the key to the superior performance of our method is that it properly pools the right historical class and the target class to learn the transition probabilities and policies. In contrast, contextual methods pool the wrong classes due to the model misspecifications. Additionally, the different treatment effect imposes another challenge to contextual methods since they try to learn a single coefficient, $c_1^h$ for the action, resulting in more deviation from the ground truth. As we can see in Figure \ref{fig:regret synthetic diff treat}, although having good initial policies, the regret trajectories of the two contextual methods have no obvious downward trend, which means they need a lot more data to correct the misspecification caused by historical data. 
Though the treatment effect is larger, our data-pooling method still pools correct class pairs together, achieving much better performance than the contextual methods. 

\begin{figure}[htbp]
	\begin{center}
		\centerline{\includegraphics[width=0.45\columnwidth]{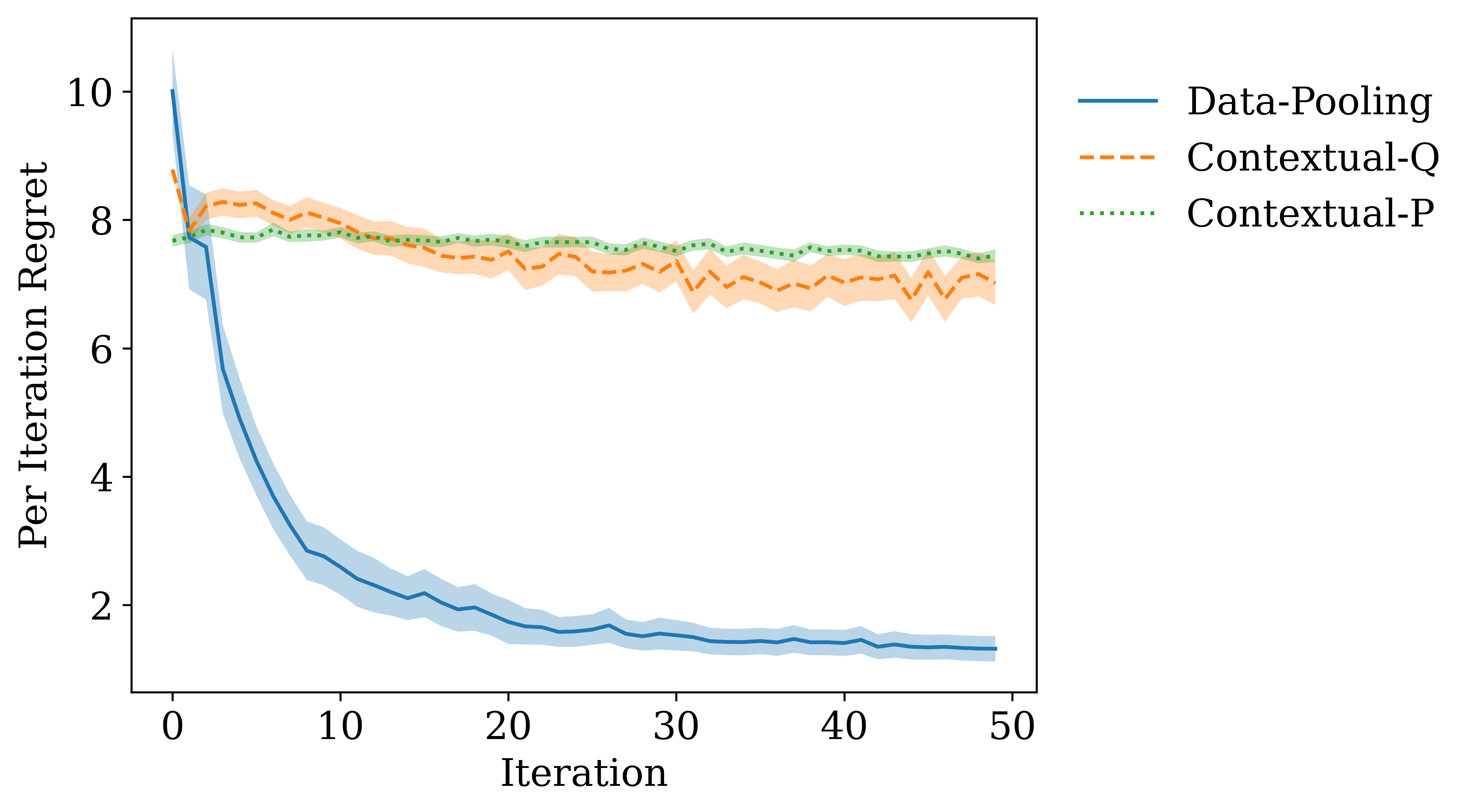}}
		\caption{Regret per iteration of 2 target patient classes. 
		The solid line represents the average from 100 replications, and the shaded area corresponds to the 95\% confidence interval.}
		\label{fig:regret synthetic diff treat}
	\end{center}
	\vskip -0.2in
\end{figure}

\vspace{-0.1in}
\begin{table}[htbp]
\centering 
\scalebox{0.9}{
\begin{tabular}{@{}cllcc@{}}
\toprule
Algorithm             & Total Regret & Total Cost & \begin{tabular}[c]{@{}c@{}}Readmission \\  Rate\end{tabular} & CPU time/class \\ \midrule
 Data-pooling            &       118.52$\pm$10.64  & 6327.71$\pm$40.47 & 5.31\% & 143.99 \\ 
  Contextual-Q     &       371.01$\pm$9.35  & 6605.74$\pm$41.42 & 5.18\% & 1384.08 \\ 
Contextual-P &         381.46$\pm$4.89 & 4901.53$\pm$34.57 & 4.14\% & 1543.71 \\ \hline
\end{tabular}
}
\caption{The Performance of different algorithms for 50 iterations and 100 rounds. 
The numbers following the $\pm$ sign are the half-width of the 95\% confidence interval of the corresponding value.} 
\label{table: regret synthetic diff treat}
\end{table}

\end{document}